\documentclass{article}
  
\usepackage[letterpaper, margin=1.1in]{geometry}

\usepackage{graphicx,xspace,xcolor} 

\usepackage{times}
\usepackage{soul}
\usepackage{url}
\usepackage[hidelinks]{hyperref}
\usepackage[utf8]{inputenc}
\usepackage[small]{caption}
\usepackage{graphicx}
\usepackage{amsmath}
\usepackage{amsthm}
\usepackage{booktabs}
\usepackage{algorithm}
\usepackage[sort&compress, numbers]{natbib}

\usepackage[noend]{algpseudocode}
\usepackage{algorithmicx}

\algnewcommand\algorithmicinput{\textbf{Input:}}
\algnewcommand\INPUT{\item[\algorithmicinput]}

\algnewcommand\algorithmicoutput{\textbf{Output:}}
\algnewcommand\OUTPUT{\item[\algorithmicoutput]}

\newcommand{\NULL}{\textnormal{\texttt{nil}}}

\newif\iflong
\newif\ifshort

\longtrue

\iflong
\else
\shorttrue
\fi

\usepackage{cleveref}

\newtheorem{theorem}{Theorem}
\newtheorem{lemma}[theorem]{Lemma}

\newtheorem{proposition}[theorem]{Proposition}
\newtheorem{corollary}[theorem]{Corollary}

\usepackage{amssymb}
\usepackage{enumerate,verbatim}
\usepackage{xspace}
\usepackage{tikz,tikz-cd}
\usetikzlibrary{arrows,cd,positioning,shapes,patterns}
\usetikzlibrary {decorations.pathmorphing, decorations.pathreplacing, decorations.shapes}
\usepackage{subcaption}
\usetikzlibrary{calc}
\tikzset{cross/.style={cross out, draw=black, minimum size=2*(#1-\pgflinewidth), inner sep=0pt, outer sep=0pt},
cross/.default={1pt}}

\usepackage{todonotes}
\presetkeys%
    {todonotes}%
    {backgroundcolor=yellow}{}

\newcommand{\bigoh}{\mathcal{O}}

\title{Explaining Decisions in ML Models:\\ a Parameterized Complexity Analysis}
\date{}
\author{Sebastian Ordyniak\thanks{School of Computing, University of Leeds, UK,
    \texttt{sordyniak@gmail.com}} \and
  Giacomo Paesani\thanks{Dipartimento di Informatica, Sapienza University of Rome, Italy, \texttt{giacomopaesani@gmail.com}} \and  
  Mateusz Rychlicki\thanks{School of Computing, University of Leeds, UK,
    \texttt{mkrychlicki@gmail.com}} \and
Stefan Szeider\thanks{Algorithms and Complexity Group, TU Wien,
  Vienna, Austria, \texttt{sz@ac.tuwien.ac.at}}}

\usepackage{boxedminipage}
\newcommand{\pbDef}[3]{%
  \noindent
  \begin{center}
  \begin{boxedminipage}{0.98 \columnwidth}
  {\sc #1}\\[2pt]
  \begin{tabular}{@{}l@{~~} p{0.76 \columnwidth}@{}}
  {\sc Instance}: & #2\\
  {\sc Question}: & #3
  \end{tabular}
  \end{boxedminipage}
  \end{center}
}

\newcommand{\mypara}[1]{\smallskip\noindent\textbf{#1}}

\newcommand{\cc}[1]{{\mbox{\textnormal{\textsf{#1}}}}\xspace}  %

\newcommand{\hy}{\hbox{-}\nobreak\hskip0pt}

\renewcommand{\P}{\cc{P}}
\newcommand{\NP}{\cc{NP}}

\newcommand{\FPT}{\cc{FPT}}
\newcommand{\XP}{\cc{XP}}
\newcommand{\Weft}{{\cc{W}}}
\newcommand{\W}[1]{{\Weft}{{\textnormal[\ensuremath{#1}\textnormal]}}}
\newcommand{\coW}[1]{{\mbox{\textnormal{\textsf{co-}}}\Weft}{{\textnormal[\ensuremath{#1}\textnormal]}}}
\newcommand{\paraNP}{\cc{paraNP}}

\newcommand{\NPh}{{\mbox{\textnormal{\textsf{NP}}}}\hy{}\text{hard}}
\newcommand{\coNPh}{{\mbox{\textnormal{\textsf{co-NP}}}}\hy{}\text{hard}}
\newcommand{\Wh}[1]{\W{#1}\hy{}hard}
\newcommand{\coWh}[1]{\coW{#1}\hy{}hard}
\newcommand{\pNPh}{{\mbox{\textnormal{\textsf{pNP}}}}\hy{}\text{hard}}
\newcommand{\pNPtable}{{\mbox{\textnormal{\textsf{pNP}}}}\hy{}\text{h}}
\newcommand{\copNPtable}{{\mbox{\textnormal{\textsf{co-pNP}}}}\hy{}\text{h}}

\newcommand{\NPtable}{{\mbox{\textnormal{\textsf{NP}}}}\hy{}\text{h}}
\newcommand{\coNPtable}{{\mbox{\textnormal{\textsf{co-NP}}}}\hy{}\text{h}}
\newcommand{\Wtable}[1]{\W{#1}\hy{}\text{h}}
\newcommand{\coWtable}[1]{\coW{#1}\hy{}\text{h}}

\newcommand{\SM}{{\;{|}\;}}
\newcommand{\SE}{\,\}}
\newcommand{\SB}{\{\,}

\newcommand{\MM}{\ensuremath{\mathcal{M}}}
\newcommand{\PP}{\ensuremath{\mathcal{P}}}
\newcommand{\FFF}{\ensuremath{\mathcal{F}}}
\newcommand{\TTT}{\ensuremath{\mathcal{T}}}
\newcommand{\MMM}{\ensuremath{\mathcal{E}}}
\newcommand{\PPP}{\mathcal{P}}

\newcommand{\modelname}[1]{\textnormal{#1}}
\newcommand{\MAJ}{\textnormal{MAJ}}

\newcommand{\DT}{\modelname{DT}\xspace}
\newcommand{\DS}{\modelname{DS}\xspace}
\newcommand{\DL}{\modelname{DL}\xspace}
\newcommand{\DSE}{\modelname{DS}\ensuremath{_\MAJ}\xspace}
\newcommand{\DLE}{\modelname{DL}\ensuremath{_\MAJ}\xspace}
\newcommand{\MME}{\MM\ensuremath{_\MAJ}\xspace}
\newcommand{\NN}{\modelname{NN}\xspace}
\newcommand{\BDD}{\modelname{BDD}\xspace}
\newcommand{\OBDD}{\modelname{OBDD}\xspace}

\newcommand{\RF}{\modelname{DT\ensuremath{_\MAJ}}\xspace}
\newcommand{\BC}{\modelname{BC}\xspace}
\newcommand{\OBDDE}{\modelname{OBDD}\ensuremath{_\MAJ}\xspace}

\newcommand{\OBDDO}{\modelname{OBDD}\ensuremath{^<}\xspace}
\newcommand{\OBDDOc}[1]{\modelname{OBDD}\ensuremath{^{#1}}\xspace}
\newcommand{\OBDDOcs}[1]{\modelname{OBDD}\ensuremath{^{#1}}s\xspace}
\newcommand{\OBDDEO}{\modelname{OBDD}\ensuremath{_\MAJ^<}\xspace}
\newcommand{\OBDDEOc}[1]{\modelname{OBDD}\ensuremath{_\MAJ^{#1}}\xspace}

\newcommand{\ds}{S}
\newcommand{\dse}{\mathcal{S}}
\newcommand{\dl}{L}
\newcommand{\dle}{\mathcal{L}}

\newcommand{\obdd}{O}
\newcommand{\obddo}{O}
\newcommand{\obdde}{\mathcal{O}}
\newcommand{\obddeo}{\mathcal{O}}

\renewcommand{\SS}{\subseteq}
\newcommand{\CD}{{\rule{0.5pt}{1.2ex}~\rule{0.5pt}{1.2ex}}}

\newcommand{\N}{$-$\xspace}

\newcommand{\LAEX}{\textsc{lAXp}}
\newcommand{\GAEX}{\textsc{gAXp}}
\newcommand{\LCEX}{\textsc{lCXp}}
\newcommand{\GCEX}{\textsc{gCXp}}

\newcommand{\SMLAEX}{\textsc{lAXp${}_\SS$}}
\newcommand{\SMGAEX}{\textsc{gAXp${}_\SS$}}
\newcommand{\SMLCEX}{\textsc{lCXp${}_\SS$}}
\newcommand{\SMGCEX}{\textsc{gCXp${}_\SS$}}

\newcommand{\MLAEX}{\textsc{lAXp${}_\CD$}}
\newcommand{\MGAEX}{\textsc{gAXp${}_\CD$}}
\newcommand{\MLCEX}{\textsc{lCXp${}_\CD$}}
\newcommand{\MGCEX}{\textsc{gCXp${}_\CD$}}

\newcommand{\SPP}{\ensuremath{\PP_\SS}}
\newcommand{\MPP}{\ensuremath{\PP_\CD}}

\newcommand{\MCC}{\textsc{MCC}}

\newcommand{\MNL}{\textsf{MNL}}
\newcommand{\CIRC}{\mathcal{C}}
\newcommand{\concat}{\circ}
\newcommand{\som}[1]{\|#1\|}

\newcommand{\feat}{F}

\newcommand{\fBDD}{\rho}
\newcommand{\EEE}{\mathcal{E}}

\newcommand{\rw}{\textsf{rw}}

\usepackage{adjustbox}
\usepackage{array}

\newcolumntype{R}[2]{%
    >{\adjustbox{angle=#1,lap=\width-(#2)}\bgroup}%
    l%
    <{\egroup}%
}
\newcommand*\rot{\multicolumn{1}{R{30}{0.4em}}}

\begin{document}
\maketitle
\begin{abstract}
   This paper presents a comprehensive theoretical investigation into
  the parameterized complexity of explanation problems in various
  machine learning (ML) models. Contrary to the prevalent black-box
  perception, our study focuses on models with transparent internal
  mechanisms. We address two principal types of explanation problems:
  abductive and contrastive, both in their local and global
  variants. Our analysis encompasses diverse ML models, including
  Decision Trees, Decision Sets, Decision Lists, Ordered Binary Decision
  Diagrams, Random Forests, and Boolean Circuits, and ensembles
  thereof, each offering unique explanatory challenges. This research fills a
  significant gap in explainable AI (XAI) by providing a foundational
  understanding of the complexities of generating explanations for
  these models.  This work provides insights vital for further
  research in the domain of XAI, contributing to the broader discourse
  on the necessity of transparency and accountability in AI systems.
\end{abstract}

\section{Introduction}

As machine learning (ML) models increasingly permeate essential
domains, understanding their decision-making mechanisms has become
central. This paper delves into the field of explainable AI (XAI) by
examining the parameterized complexity of explanation problems in
various ML models. We focus on models with accessible internal
mechanisms, shifting away from the traditional black-box paradigm. Our
motivation is rooted in establishing a comprehensive theoretical
framework that illuminates the complexity of generating explanations
for these models, a task becoming increasingly relevant in light of
recent regulatory guidelines that  emphasize the importance of
transparent and explainable AI~\cite{EU20,OECD23}.

The need for transparency and accountability in automated
decision-making drives the imperative for explainability in AI
systems, especially in high-risk sectors. ML models, while powerful,
must be demystified to gain trust and comply with ethical and
regulatory standards. Formal explanations serve this purpose,
providing a structured means to interpret model decisions
\cite{Silva22,Guidotti-etal-2018,Carvalho-etal-19}.

Our exploration focuses on two types of explanation problems,
abductive and contrastive, in local and global contexts
\cite{Silva22}.  \emph{Abductive explanations} \cite{IgnatievNM19},
corresponding to 
prime-implicant explanations
\cite{ShihCD18} and sufficient reason
explanations \cite{DarwicheJi22}, clarify specific decision-making
instances, while \emph{contrastive explanations}
\cite{Miller19,IgnatievNA20}, corresponding to necessary reason
explanations \cite{DarwicheJi22}, make explicit the reasons behind the
non-selection of alternatives. The study of contrastive explanations
goes back to the Lipton's work in 1990.
Conversely, \emph{global explanations}
\cite{Ribeiro0G16,IgnatievNM19} aim to unravel models' decision
patterns across various inputs. This bifurcated approach enables a
comprehensive understanding of model behavior, aligning with the
recent emphasis on interpretable ML \cite{LisboaSVFV23}.

In contrast to a recent study by \citet{OrdyniakPS23}, who consider
the parameterized complexity of finding explanations based on samples
classified by a black-box ML model, we focus on the setting where
the model together with its inner workings is available as an input
for computing explanations. This perspective, initiated by \citet{BarceloM0S20}, is particularly
appealing, as it lets us quantify the explainability of various
model types based on the computational complexity of the corresponding
explanation problems.

Challenging the notion of inherent opacity in ML models, our study
includes 
\emph{Decision Trees}~(\DT{}s), \emph{Decision Sets}~(\DS{}s), \emph{Decision Lists}~(\DL{}s),
and \emph{Ordered Binary Decision Diagrams} (\OBDD{}s).
Whereas DTs,
DSs, and DLs are classical ML models, OBDDs can be used to represent
the decision, functions of naive Bayes
classifiers~\cite{ChanDarwiche03}. We also consider \emph{ensembles}
of all the above ML models; where an ensemble classifies an example by
taking the majority classification over its elements.  For instance,
\emph{Random Forests} (RFs) are ensembles of DTs.

Each model presents distinct features affecting explanation
generation. For example, the transparent structure of DTs and RFs
facilitates rule extraction, as opposed to the complex architectures
of \emph{Neural Networks}~(\NN{}s)~\cite{Ribeiro0G16,Lipton18}.

\mypara{Contribution.}
This paper fills a crucial gap in XAI research by analyzing the
complexity of generating explanations across different models. Prior
research has often centered on practical explainability approaches,
but a theoretical understanding still needs to be developed
\cite{HolzingerSMBS20,Molnar23}. Our study is aligned with the
increasing call for theoretical rigor in AI \cite{EU19}. By dissecting
the parameterized complexity of these explanation problems, we lay the
groundwork for future research and algorithm development, ultimately
contributing to more efficient explanation methods in AI.

Since most of the considered explanation problems are NP-hard, we use
the paradigm of fixed-parameter tractability (FPT), which involves
identifying specific parameters of the problem (e.g., explanation
size, number of terms/rules, size/height of a DT, width of a BDD) and
proving that the problem is fixed-parameter tractable concerning these
parameters. By focusing on these parameters, the complexity of the
problem is confined, making it more manageable and often solvable in
uniform polynomial time for fixed values of the parameters. A
significant part of our positive results are based on reducing various
model types to 
\emph{Boolean circuits (\BC{}s)}. 
This reduction is crucial for the
uniform treatment of several model types as it allows the application
of known algorithmic results and techniques from the Boolean circuits
domain to the studied models. It simplifies the problems and brings
them into a well-understood theoretical framework. 
For ensembles, we
consider Boolean circuits with majority gates. In turn, we obtain the
fixed-parameter tractability of problems on Boolean circuits via
results on Monadic Second Order (MSO). We use extended MSO
\cite{BergougnouxDJ23} to handle
majority gates, which allows us to obtain efficient algorithmic
solutions, particularly useful for handling complex structures.

Overall, the approach in the manuscript is characterized by a mix of
theoretical computer science techniques, including parameterization,
reduction to well-known problems, and the development of specialized
algorithms that exploit the structural properties of the models under
consideration. This combination enables the manuscript to effectively
address the challenge of finding tractable solutions to explanation
problems in various machine learning models.

For some of the problems, we develop entirely new 
customized
algorithms. We complement the algorithmic results with hardness
results to get a complete picture of the tractability landscape for
all possible combinations of the considered parameters (an overview of
our results are provided in \Cref{fig:DTresults,,fig:DSDLresults,,fig:BDDresults}).

In summary,
our research marks a significant advancement in the theoretical
understanding of explainability in AI. By offering a detailed
complexity analysis for various ML models, this work enriches academic
discourse and responds to the growing practical and regulatory demand
for transparent, interpretable, and trustworthy AI systems.

\ifshort
\smallskip
\noindent {\emph{Statements whose full proofs are omitted and can be
    found in the supplementary material are marked with $\star$.}}
\fi

\section{Preliminaries}

\iflong
For a positive integer $i$, we denote by $[i]$ the set of integers
$\{1,\dotsc,i\}$.

\mypara{Parameterized Complexity (PC).}  We outline some basic concepts
refer to the textbook by \citet{DowneyFellows13} for an in-depth
treatment. An instance of a parameterized problem $Q$ is a pair
$(x,k)$ where $x$ is the main part and $k$ (usually an non-negative
integer) is the parameter. $Q$ is \emph{fixed-parameter tractable
  (FPT)} if it can be solved in time $f(k)n^c$ where $n$ is the input
size of $x$, $c$ is a constant independent of $k$, and $f$ is a
computable function. If a problem has more then one parameters, then
the parameters can be combined to a single one by addition.  $\FPT$
denotes the class of all fixed-parameter tractable decision
problems. $\XP$ denotes the class of all parameterized decision
problems solvable in time $n^{f(k)}$ where $f$ is again a computable
function. An \emph{fpt-reduction} from one parameterized decision
problem $Q$ to another $Q'$ is an fpt-computable reduction that
reduces of $Q$ to instances of $Q'$ such that yes-instances are mapped
to yes-instances and no-instances are mapped to no-instances. The
parameterized complexity classes $\W{i}$ are defined as the closure of
certain weighted circuit satisfaction problems under
fpt-reductions. Denoting by $\P$ the class of all parameterized
decision problems solvable in polynomial time, and by $\paraNP$ 
the
class of parameterized decision problems that are in $\NP$ and NP-hard
for at least one instantiation of the parameter with a constant, we
have
$ \P \subseteq \FPT \subseteq \W{1} \subseteq \W{2} \subseteq \cdots
\subseteq \XP \cap \paraNP \subseteq \paraNP$, where all inclusions
are believed to be strict. If a parameterized problem is $\Wh{i}$ under fpt-reductions ($\Wtable{i}$, for short) then it is unlikely
to be FPT. $\textsf{co-C}$ denotes the complexity class containing all
problems from $\textsf{C}$ with yes-instances replaced by no-instances and
no-instances replaced by yes-instances.

\mypara{Graphs, Rankwidth, Treewidth and Pathwidth.}
We mostly use standard notation for graphs as can be found, e.g., in~\cite{Diestel00}.
Let $G=(V,E)$ be a directed or undirected graph. For a vertex subset
$V' \subseteq V$, we denote by $G[V']$ the graph induced by the
vertices in $V'$ and by $G \setminus V'$ the graph $G[V \setminus
V']$. If $G$ is directed, we denote by $N_G^-(v)$ ($N_G^+(v)$) the set
of all incoming (outgoing) neighbors of the vertex $v \in V$.

Let $G=(V,E)$ be a directed graph. A
\emph{tree decomposition} of $G$ is a pair $\TTT=(T,\lambda)$ with
$T$ being a tree and $\lambda : V(T) \rightarrow V(G)$ such that: (1)
for every vertex $v \in V$ the set $\SB t \in V(T) \SM
v \in \lambda(t)\SE$ is forms a non-empty subtree of $T$ and (2)
for every arc $e=(u,v) \in E$, there is a node $t \in V(T)$ with $u,v
\in \lambda(t)$. The \emph{width} of $\TTT$ is equal to $\max_{t \in
  V(T)}|\lambda(t)|-1$ and the \emph{treewidth} of $G$ is the minimum
width over all tree decompositions of $G$. $\TTT$ is called a
\emph{path decomposition} if $T$ is a path and the \emph{pathwidth} of
$G$ is the minimum width over all path decompositions of $G$. We will
need the following well-known properties of pathwidth, treewidth, and rankwidth.
\begin{lemma}[{\cite{DBLP:conf/wg/CorneilR01,DBLP:journals/jct/OumS06}}]\label{lem:ranktree}
  Let $G=(V,E)$ be a directed graph and $X \subseteq V$. The
  treewidth of $G$ is at most $|X|$ plus the treewidth of
  $G-X$. Furthermore, if $G$ has rankwidth $r$, pathwidth $p$ and
  treewidth $t$, then
  $r \leq 3\cdot 2^{t-1}\leq 3\cdot 2^{p-1}$. 
\end{lemma}
\fi
\ifshort
\mypara{Parameterized Complexity.}  A
problem with input size $n$ and parameter $k$ is \emph{fixed-parameter
  tractable (fp-tractable)} if it can be solved in time $f(k)n^c$ for
a constant $c$ independent of $k$, and a computable function $f$; the
problem is \emph{xp-tractable} if it can be solved in time $n^{f(k)}$
\cite{DowneyFellows13}. $\FPT$ and $\XP$ are the classes of
fp-tractable and xp-tractable decision problems, respectively. There
is a hierarchy of parameterized complexity classes that represent
various level of \emph{in}tractability:
$ \P \subseteq \FPT \subseteq \W{1} \subseteq \W{2} \subseteq \cdots
\subseteq \XP \cap \paraNP \subseteq \paraNP$. All inclusions are
believed to be proper. If a problem is $\W{i}$\hy hard under
fpt-reductions ($\W{i}$-h, for short) then it is unlikely to be
in FPT. 
The class of 
$\textsf{co-C}$ denotes the complexity class containing all
problems from $\textsf{C}$ with yes-instances and no-instances
swapped.
\fi
 
\mypara{Examples and Models}
Let $F$ be a set of binary features. An \emph{example} $e : F
\rightarrow \{0,1\}$ over $F$ is a $\{0,1\}$-assignment of the features
in $F$. An example is a \emph{partial example (assignment)} over $F$
if it is an example over some subset $F'$ of $F$. We denote by $E(F)$
the set of all possible examples over $F$. A \emph{(binary
  classification) model} $M : E(F) \rightarrow \{0,1\}$ is a specific
representation of a Boolean function over $E(F)$. We denote by
$\feat(M)$ the set of features considered by $M$, i.e., $\feat(M)=F$.
We say that an example $e$ is a $0$-example or negative example ($1$-example
or positive example) w.r.t. the model $M$ if $M(e)=0$ ($M(e)=1$).
For convenience, we restrict our setting to the classification into two
classes. We note however that all our hardness results easily carry over to the
classification into any (in)finite set of classes. The same applies to
our algorithmic results for non-ensemble models since one can easily
reduce to the case with two classes by renaming the class of interest
for the particular explanation problem to $1$ and all other classes to
$0$. We leave it open whether the same holds for our algorithmic
results for ensemble models.

\mypara{Decision Trees.}
A \emph{decision tree} ($\DT$) $\TTT$ is a pair $(T,\lambda)$ such that $T$ is a
rooted binary tree and $\lambda : V(T) \rightarrow F \cup \{0,1\}$ is
a function that assigns a feature in $F$ to every inner node of $T$
and either $0$ or $1$ to every leaf node of $T$. Every inner node of
$T$ has exactly $2$ children, one left child (or $0$-child) and one
right-child (or $1$-child).
The classification function $\TTT : E(F) \rightarrow \{0,1\}$ of a \DT{}
is defined as follows for an example $e \in E(F)$. 
Starting at the root of $T$ one does the following at
every inner node $t$ of $T$. If $e(\lambda(t))=0$ one continues with
the $0$-child of $t$ and if $e(\lambda(t))=1$ one continues with the
$1$-child of $t$ until one eventually ends up at a leaf node $l$ at which
$e$ is classified as $\lambda(l)$.
For every node $t$ of $T$, we denote by $\alpha_{\TTT}^t$ the
partial assignment of $F$ defined by the path from the root of $T$ to
$t$ in $T$, i.e., for a feature $f$, we set $\alpha_{\TTT}^t(f)$ to $0$ ($1$) if
and only if the path from the root of $T$ to $t$ contains an inner
node $t'$ with $\lambda(t')=f$ together with its $0$-child
($1$-child).
We denote by $L(\TTT)$ the set of leaves of $T$ and
we set $L_b(\TTT)=\SB l \in L(\TTT)\SM \lambda(l)=b\SE$
for every $b \in \{0,1\}$. 
Moreover, we denote by $\som{\TTT}$ ($h(\TTT)$) the size (height) of a \DT, which
is equal to the number of leaves of $T$ (the length of a longest root-to-leaf path in $T$).
Finally, we let $\MNL(\TTT)=\min\{ |L_0|,|L_1|\}$.

\mypara{Decision Sets.}
A \emph{term} $t$ over $C$ is a set of
\emph{literals} with each literal being
of the form $(f=z)$ where $f\in F$ and $z\in \{0,1\}$. A \emph{rule}
$r$ is a pair $(t,c)$ where $t$ is a term and $c\in \{0,1\}$. We say
that a rule $(t,c)$ is a \emph{$c$-rule}.
We say that a term $t$ (or rule $(t,c)$) \emph{applies to (or agrees
  with)} an example $e$ if $e(f)=z$ for every element $(f=z)$ of $t$.
Note that the empty rule applies to any example.

A {\it decision set} ($\DS$) $\ds$ is a pair $(T,b)$, where $T$ is a set of terms
and $b \in \{0,1\}$ is the classification of the default rule (or the
default classification).
We denote by $\som{\ds}$ the size of $\ds$ which is
equal to $(\sum_{t \in T}|t|)+1$; the $+1$ is for the default rule.
The classification function $\ds : E(F) \rightarrow \{0,1\}$ of a \DS{}
$\ds=(T,b)$ is defined by setting $\ds(e)=b$ for every example $e \in E(F)$
such that no term in $T$ applies to $e$ and otherwise we set $\ds(e)=1-b$.

\mypara{Decision Lists.}
A {\it decision list} ($\DL$) $L$ is a non-empty sequence of rules
$(r_1=(t_1,c_1),\dotsc,r_\ell=(t_\ell,c_\ell))$, for some $\ell \geq 0$.
The size of a \DL{} $L$, denoted by $\som{L}$, is equal to $\sum_{i=1}^\ell(|t_i|+1)$.
The classification function $L : E(F) \rightarrow \{0,1\}$ of a \DL{}
$L$ is defined by setting $L(e)=b$ if the
first rule in $L$ that applies to $e$ is a $b$-rule.
To ensure that
every example obtains some classification, we assume that the term of the last
rule is empty and therefore applies to all examples.

\mypara{Binary Decision Diagrams.}
A {\it binary decision diagram} ($\BDD$) $B$ is a pair $(D,\fBDD)$ where
$D$ is a directed acyclic graph with three special vertices $s,t_0,t_1$
such that:
\begin{itemize}
\item $s$ is a source vertex that can (but does not have to)
 be equal to $t_0$ or $t_1$,
\item $t_0$ and $t_1$ are the only sink vertices of $D$,
\item every non-sink vertex has exactly two outgoing
 neighbors, which we call the $0$-neighbor and the
 $1$-neighbor, and
\item $\fBDD: V(D)\setminus \{t_0,t_1\} \rightarrow F$ is a function
 that associates with every non-sink node of $D$ a feature in $F$.
\end{itemize}

For an example $e\in E$, we denote by $P_B(e)$ (or $P(e)$ if $B$ is
clear from the context), the unique path from $s$ to either $t_0$ or
$t_1$ followed by $e$ in $B$. That is starting at $s$ and ending at
either $t_0$ or $t_1$, $P(e)$ is iteratively defined as
follows. Initially, we set $P(e)=(s)$, moreover, if $P(e)$ ends in a
vertex $v$ other than $t_0$ or $t_1$, then we extend $P(e)$ by the
$e(\fBDD(v))$-neighbor of $v$ in $D$.
Let $B$ be a \BDD{} and $e \in E(F)$ be an example.
The classification function $B : E(F) \rightarrow \{0,1\}$ of 
$B$ is given by setting $B(e)=b$ if $P_B(e)$ ends in $t_{b}$.
We denote by $\som{B}$ the size of
$B$, which is equal to $|V(D)|$. 
We say that $B$ is an \emph{\OBDD{}} if
every path in $B$ contains features in the same order. Moreover, $B$ is a \emph{complete \OBDD{}} if every maximal path contains the same set of features.
It is known that every \OBDD{} can be transformed in polynomial-time into an equivalent complete \OBDD{}~\cite[Observation 1]{DBLP:journals/corr/abs-2104-02563}. All \OBDD{}s considered in the paper are complete.

\sloppypar\mypara{Ensembles.}
An \emph{$\MM$-ensemble}, also denoted by $\MM_\MAJ$, $\EEE$ is a set of
models of type $\MM$, where
$\MM\in \{\DT,\DS,\DL,\OBDD\}$. We say that $\EEE$ classifies an example $e
\in E(F)$ as $b$ if so do the majority of models in $\EEE$, i.e.,
if there are at least $\lfloor|\EEE|/2\rfloor+1$
models in $\EEE$ that classify $e$ as $b$. We denote by $\som{\EEE}$ the size of
$\EEE$, which is equal to $\sum_{M \in \EEE}\som{M}$. We additionally
consider an \emph{ordered $\OBDD$-ensemble}, denote by \OBDDEO{},
where all \OBDD{}s in the ensemble respect the same ordering of the features.

\section{Considered Problems and Parameters}

\begin{figure}
\centering
\begin{tikzpicture}[scale=1]
\tikzstyle{q} = [draw, rectangle,  fill= gray!10, inner sep=2pt]
\tikzstyle{t} = [draw=black,rectangle, fill= darkgray, text=white, inner sep=2pt]
\draw (13.0,-2.5) node  (dl) {\fbox{$\arraycolsep=1.4pt
\begin{array}{llll}
r_1:&\text{IF}   & (x=1 \wedge y=1)&\text{THEN } 0\\
r_2:&\text{ELSE IF}   &(x=0 \wedge z=0) &\text{THEN } 1\\
r_3:&\text{ELSE IF}   &(y=0 \wedge z=1) &\text{THEN } 0\\
r_4:&\text{ELSE}   &  &\text{THEN } 1\\
\end{array}
$}};
\end{tikzpicture}
\caption{Let $L$ be the \DL given in the figure and let $e$ be the example given by $e(x)=0$, $e(y)=0$ and $e(z)=1$. Note that $L(e)=0$. It is easy to verify that $\{y,z\}$ is the only local abductive explanation for $e$ in $L$ of size at most 2. Moreover, both $\{y\}$ and $\{z\}$ are minimal
local contrastive explanations for $e$ in $L$.
Let $\tau_1=\{x\mapsto 1,y \mapsto 1\}$ and $\tau_2=\{ x \mapsto 0,z \mapsto 0\}$ be a partial assignments.
Note that $\tau_1$ and $\tau_2$ are minimal global abductive and global contrastive
explanations for class $0$ w.r.t. $L$, respectively.
}
\label{fig:expl}
\end{figure}

We consider the following types of explanations
(see~\citeauthor{Silva22}'s  survey \cite{Silva22}).
Let $M$ be a model, $e$ an example over $\feat(M)$, and let $c \in
\{0,1\}$ be a classification (class). We consider the following types
of explanations for which an example is illustrated in \Cref{fig:expl}.
\begin{itemize}
\item A \emph{(local) abductive explanation (\LAEX{}) for $e$
  w.r.t. $M$} is a subset $A \subseteq \feat(M)$ of features
  such that $M(e)=M(e')$ for every example $e'$ that agrees with $e$ on $A$.
\item A \emph{(local) contrastive explanation (\LCEX{}) for $e$
    w.r.t. $M$} is a set $A$ of
  features such that there is an example $e'$ such that $M(e')\neq
  M(e)$ and $e'$ differs from $e$ only on the features in $A$.
\item A \emph{global abductive explanation (\GAEX{}) for $c$
  w.r.t. $M$} is a partial example
  $\tau : F \rightarrow \{0,1\}$, where $F \subseteq \feat(M)$,
  such that $M(e)=c$ for every example $e$ that agrees with~$\tau$.
\item A \emph{global contrastive explanation (\GCEX{}) for $c$
    w.r.t. $M$} is a partial example
  $\tau : F \rightarrow \{0,1\}$, where $F \subseteq \feat(M)$,
  such that $M(e)\neq c$ for every example that agrees with~$\tau$.
\end{itemize}

For each of the above explanation types, each of the considered model
types $\MM{}$, and depending on whether or
not one wants to find a subset minimal or cardinality-wise minimum
explanation, one can now define the corresponding computational
problem. For instance:

\pbDef{\MM{}-\textsc{Subset-Minimal Local Abductive Explanation
    (\SMLAEX{})}}{A model $M \in \MM$ and an example $e$.}{Find a subset minimal local
  abductive explanation for $e$ w.r.t.~$M$.}

\pbDef{\MM{}-\textsc{Cardinality-Minimal Local Abductive Explanation
    (\MLAEX{})}}{A model $M \in \MM$, an example $e$, and an integer
  $k$.}{Is there a local explanation for $e$ w.r.t. $M$ of size at most $k$?}
The problems $\MM$-$X_\SS$ and $\MM$-$X_\CD$ for $X\in \{ \GAEX$,
$\LCEX$, $\GCEX \}$ are defined analogously.

Finally, for these problems we will consider natural parameters listed
in \Cref{tab:parms}; not all parameters apply to all considered
problems. We denote a problem $X$ parameterized by parameters $p,q,r$
by $X(p+q+r)$.

\section{Overview of Results}\label{sec:overview}
\newcommand{\sizeelem}{\textsl{size\_elem}\xspace}
\newcommand{\termselem}{\textsl{terms\_elem}\xspace}
\newcommand{\enssize}{\textsl{ens\_size}\xspace}
\newcommand{\termsize}{\textsl{term\_size}\xspace}
\newcommand{\widthelem}{\textsl{width\_elem}\xspace}
\newcommand{\xpsize}{\textsl{xp\_size}\xspace}
\newcommand{\mnlsize}{\textsl{mnl\_size}\xspace}

As we consider several problems, each with several variants and
parameters, there are hundreds of combinations to consider. We
therefore provide a condensed summary of our results in
\Cref{fig:DTresults,,fig:DSDLresults,,fig:BDDresults}.

The first column in each table indicates whether a result applies to
the cardinality-minimal or subset-minimal variant of the explanation
problem (i.e., to $X_\SS$ or $X_\CD$, respectively).  The next~4
columns in
\Cref{fig:DTresults,fig:DSDLresults,fig:BDDresults} indicate the parameterization,
the parameters are explained in \Cref{tab:parms}. A ``p'' indicates
that this parameter is part of the parameterization, a ``\N''
indicates that it isn't. A ``c'' means the parameter is set to a
constant, ``1'' means the constant is 1.

By default, each row in the tables applies to all four problems
\LAEX, \GAEX, \GCEX, and \LCEX. However, if a result only
applies to \LCEX, it is stated in parenthesis. So, for instance, the
first row of \Cref{fig:DTresults} indicates that \DT-\SMLAEX{},
\DT-\SMGAEX, \DT-\SMGCEX, and \DT-\SMLCEX, where the
ensemble consists of a single \DT, can be solved in polynomial
time. 

The penultimate row of \Cref{fig:DTresults}
indicates that \RF-\MLAEX,
\RF-\MGAEX{} and \RF-\MGCEX{} are \coNPh{} even if
$\mnlsize{}+\sizeelem{}+\xpsize{}$ is constant,
and \RF-\MLCEX{} is \Wh{1} parameterized by \xpsize{} even if 
$\mnlsize{}+\sizeelem{}$ is constant.
Finally, the $\star$ indicates a minor distinction in the complexity
between \DT-\MLAEX{} and the two problems \DT-\MGAEX{} and \DT-\MGCEX{}. That
is, if the cell contains $\NPtable{}^\star$ or $\pNPtable{}^\star$, then \DT-\MLAEX{} is
\NPh{} or \pNPh{}, respectively, and neither \DT-\MGAEX{} nor
\DT-\MGCEX{} are in \P{} unless $\FPT{}=\W{1}$.

We only state in the tables those results that are not implied by
others. Tractability results propagate in the following list from left
to right, and hardness results propagate from right to left.
\begin{eqnarray*}
  \label{eq:1}
  \text{$\CD$-minimality} &\Rightarrow& \text{$\SS$-minimality}\\[-3pt]
  \text{set $A$ of parameters} &\Rightarrow& \text{set $B\supseteq A$ of parameters}\\[-3pt]
  \text{ensemble of models}   &\Rightarrow& \text{single model}\\[-3pt]
  \text{unordered OBDD ensemble}&\Rightarrow& \text{ordered OBDD ensemble}                     
\end{eqnarray*}
For instance, the tractability of $X_\CD$ implies the
tractability of $X_\SS$, and the hardness of $X_\SS$ implies the
hardness of $X_\CD$.\iflong\footnote{Note that even though the inclusion-wise
  minimal versions of our problems are defined in terms of finding
  (instead of decision), the implication still holds because there is
  a polynomial-time reduction from the finding version
  to the decision version of \MLAEX{}, \MLCEX{}, \MGAEX{}, and
  \MGCEX{}.}\fi

\begin{table}[tbh]
  \centering
  \begin{tabular}{@{}l@{~~}l@{}}
    \toprule
    parameter & definition \\
    \midrule
    \enssize & number of elements of the ensemble\\
    \mnlsize & largest number of $\MNL$ over all ensemble elem. \\
    \termselem & largest number of terms per  ensemble elem.          \\
    \termsize & size of a largest term over all ensemble elem.\\
    \widthelem & largest width over all ensemble elements\\
    \sizeelem & size of largest ensemble element         \\
    \xpsize & size of the explanation\\
              \bottomrule
  \end{tabular}
  \caption{Main parameters considered. Note that some parameters (such
    as \widthelem) only
    apply to specific model types.}\label{tab:parms} 
\end{table}

\newcommand{\Tnum}[1]{}
\begin{table}[htb]
    \Crefname{theorem}{Thm}{Thms}
\centering
\begin{tabular}{@{}cc@{}c@{}c@{}cc@{}r@{}}
  \toprule
 \rot{minimality} &\rot{\enssize}&\rot{\mnlsize} &\rot{\sizeelem} &\rot{\xpsize} &\rot{complexity} &\rot{result} \\
  \midrule 
$\SS$ &1 &\N &\N &\N &\P &\Cref{th:poly-DT-SM}\Tnum{4,5} \\
$\CD$ &1 &\N &\N &\N &$\NPtable^{\star}$(\P) &
\ifshort\Cref{th:DT-MLAEX-W2,th:DT-MGAEX-W1,th:poly-DT-SM}\fi\iflong\Cref{th:DT-MLAEX-W2,th:DT-MGAEX-W1,th:DT-MLCEX}\fi\Tnum{4,5} \\
$\CD$ &1 &\N &\N &p &\Wtable{1}(\P) &
\ifshort\Cref{th:DT-MLAEX-W2,th:DT-MGAEX-W1,th:poly-DT-SM}\fi\iflong\Cref{th:DT-MLAEX-W2,th:DT-MGAEX-W1,th:DT-MLCEX}\fi\Tnum{4,5} \\
$\CD$ &1 &\N &\N &p &\XP (\P) &\ifshort\Cref{th:DT-LGA-XP,th:poly-DT-SM}\fi\iflong\Cref{th:DT-LGA-XP,th:DT-MLCEX}\fi\Tnum{4,6} \\
$\SS$ &p &\N &\N &\N &\coWtable{1}(\Wtable{1}) &\Cref{th:RF-W1-ES}\Tnum{8} \\
$\SS$ &p &\N &\N &\N &\XP &\Cref{th:Rf-XP-e}\Tnum{10} \\
$\CD$ &p &\N &\N &\N &$\pNPtable^{\star}$(\XP) &\Cref{th:DT-MLAEX-W2,th:DT-MGAEX-W1,th:Rf-XP-e}\Tnum{10} \\
$\CD$ &p &p &\N &\N &\FPT &\Cref{th:RF-FPT-MNL}\Tnum{10} \\
$\CD$ &p &\N &p &\N &\FPT &\Cref{th:RF-FPT-MNL}\Tnum{7} \\
$\CD$ &p &\N &\N &c(p) &\coWtable{1}(\Wtable{1}) &\Cref{th:RF-W1-ES}\Tnum{8} \\
$\SS$ &\N &c &c &\N &\coNPtable (\NPtable) &\Cref{th:RF-paraNP}\Tnum{9} \\
$\CD$ &\N &c &c &c(p) &\coNPtable (\Wtable{1}) &\Cref{th:RF-paraNP}\Tnum{9} \\
$\CD$ &\N &\N &\N &p &\copNPtable(\XP) &\Cref{th:RF-paraNP,th:MLCEX-XP}\Tnum{1} \\
  \bottomrule
\end{tabular}
\caption{%
  Explanation complexity when the model is a \DT or an ensemble of
  DTs. See  \Cref{sec:overview} for how to read the table.
}\label{fig:DTresults}
\end{table}

\begin{table}[tbh]
    \Crefname{theorem}{Thm}{Thms}
  \Crefname{corollary}{Cor}{Cors}
  \centering
  \begin{tabular}{@{}cc@{}c@{}c@{}ccr@{}}
    \toprule
\rot{minimality} & \rot{\enssize}& \rot{\termselem} & \rot{\termsize}  & \rot{\xpsize} &\rot{complexity} & \rot{result} \\
 \midrule

$\SS$ &1 &\N &c & \N &\coNPtable(\NPtable) &\Cref{th:DS-paraCoNP}\Tnum{15} \\
$\CD$ &1 &\N &\N &p &\copNPtable(\Wtable{1}) & \Cref{th:DS-paraCoNP,th:DS-SMLCEX-W1}\Tnum{31} \\
$\CD$ &c &\N &p &p  &\copNPtable(\FPT) & \Cref{th:DS-paraCoNP,th:DS-MLCEX-FPT-const-ens}\Tnum{33} \\
$\CD$ &p &p &\N &\N &\FPT &\Cref{cor:ds-ensa}\Tnum{12} \\
$\CD$ &p &\N &c &p  &\copNPtable(\Wtable{1}) & \Cref{th:DS-paraCoNP,th:DSE-SMLCEX-W1}\Tnum{33} \\
$\SS$ &\N &c &c &\N &\coNPtable(\NPtable) &\Cref{th:DSE-paraNP}\Tnum{17} \\
$\CD$ &\N &c &c &c(p) &\coNPtable(\Wtable{1}) &\Cref{th:DSE-paraNP}\Tnum{17} \\
$\CD$ &\N &\N &\N &p &\copNPtable(\XP) &\Cref{th:DS-paraCoNP,th:MLCEX-XP}\Tnum{1} \\
    \bottomrule
  \end{tabular}
  \caption{
    Explanation complexity when the model is a \DS, a \DL, or an
    ensemble thereof. See \Cref{sec:overview} for how to read the
    table.}\label{fig:DSDLresults}
\end{table}

\begin{table}[tbh]
  \Crefname{theorem}{Thm}{Thms}
  \Crefname{corollary}{Cor}{Cors}
  \centering
    \begin{tabular}{@{}clc@{}c@{}c@{}cc@{}r@{}@{}}
    \toprule
\rot{minimality} & \rot{ordered/unordered}& \rot{\enssize}  & \rot{\widthelem} & \rot{\sizeelem}  & \rot{\xpsize} &\rot{complexity}& \rot{result} \\ 
    \midrule
$\SS$ &u &1 &\N &\N &\N &\P &\ifshort\Cref{th:OBDD-SLGA-P}\fi\iflong\Cref{th:OBDD-SLGA-P,th:OBDD-MLCEX}\fi\Tnum{20,19} \\
$\CD$ &o &1 &\N &\N &\N &\NPtable(\P) &\ifshort\Cref{th:OBDD-MLAEX-W2,th:OBDD-SLGA-P}\fi\iflong\Cref{th:OBDD-MLAEX-W2,th:OBDD-MLCEX}\fi\Tnum{20,19} \\
$\CD$ &o &1 &\N &\N &p &\Wtable{2}(\P) &\ifshort\Cref{th:OBDD-MLAEX-W2,th:OBDD-SLGA-P}\fi\iflong\Cref{th:OBDD-MLAEX-W2,th:OBDD-MLCEX}\fi\Tnum{24,19} \\
$\SS$ &u &c &c &\N &\N &\coNPtable(\NPtable) &\Cref{th:OBDDE-paraNP}\Tnum{23} \\
$\CD$ &u &c &c &\N &c(p) &\coNPtable(\Wtable{1}) &\Cref{th:OBDDE-paraNP}\Tnum{23} \\
$\SS$ &o &p &\N &\N &\N &\coWtable{1}(\Wtable{1}) &\Cref{th:OBDDEO-W1}\Tnum{29} \\
$\SS$ &o &p &\N &\N &\N &\XP &\Cref{th:OBDDEO-XP-ens}\Tnum{29} \\
$\CD$ &o &p &\N &\N &\N &\pNPtable(\XP) &\Cref{th:OBDD-MLAEX-W2,th:OBDDEO-XP-ens}\Tnum{29} \\
$\CD$ &o &p &p &\N &\N &\FPT &\Cref{th:OBDDE-FPT-es}\Tnum{22} \\
$\CD$ &u &p &\N &p &\N &\FPT &\Cref{th:OBDDE-FPT-eh}\Tnum{28} \\
$\CD$ &o &p &\N &\N &c(p) &\coWtable{1}(\Wtable{1}) &\Cref{th:OBDDEO-W1}\Tnum{29} \\
$\SS$ &o &\N &c &c &\N &\coNPtable(\NPtable) &\Cref{th:OBDDEO-paraNP-hws}\Tnum{26} \\
$\CD$ &o &\N &c &c &c(p) &\coNPtable(\Wtable{1}) &\Cref{th:OBDDEO-paraNP-hws}\Tnum{26} \\
$\CD$ &u &\N &\N &\N &p &\copNPtable(\XP) &\Cref{th:OBDDE-paraNP,th:MLCEX-XP}\Tnum{1} \\
    \bottomrule
  \end{tabular}
  \caption{
    Explanation complexity when the model is an OBDD or an ensemble
    thereof.  For an ensemble, column ``ordered/unordered'' indicates
    whether all the OBDDs in the ensemble have the same
    variable-order.  See \Cref{sec:overview} for how to read the
    table. } 
\label{fig:BDDresults}
\end{table}

\section{Algorithmic Results}

\newcommand{\val}[3]{\textsf{val}(#1,#2,#3)}
\newcommand{\IG}{\textsf{IG}}
\newcommand{\G}{\textsf{G}}
\newcommand{\OUT}{\textsf{O}}

In this section, we will present our algorithmic results.
We start with some general observations that are independent of a
particular model type.
\ifshort \begin{theorem}[$\star$]\fi\iflong\begin{theorem}\fi\label{th:MLCEX-XP}
  Let $\MM{}$ be any model type such that $M(e)$ can be computed in
  polynomial-time for $M \in \MM$. \MM{}-\MLCEX{} parameterized
  by \xpsize{} is in \XP{}.
\end{theorem}
\iflong\begin{proof}
  Let $(M,e,k)$ be the given instance of \MM{}-\MLCEX{} and suppose that
  $A \subseteq \feat(M)$ is a cardinality-wise minimal local contrastive
  explanation for $e$ w.r.t. $M$. Because $A$ is cardinality-wise minimal,
  it holds the example $e_A$ obtained from $e$ by setting
  $e_A(f)=1-e(f)$ for every $f \in A$ and $e_A(f)=e(f)$ otherwise, is
  classified differently from $e$, i.e., $M(e)\neq M(e_A)$. 
  Therefore, a set $A \subseteq \feat(M)$ is a 
  cardinality-wise minimal local contrastive explanation for $e$
  w.r.t. $M$ if and only if $M(e)\neq M(e_A)$ and there is no
  cardinality-wise smaller set $A'$ for which this is the case.
  This
  now allows us to obtain an \XP{} algorithm for $\MM{}$-\MLCEX{} as
  follows. We first enumerate all possible subsets $A \subseteq
  \feat(M)$ of size at most $k$ in time $\bigoh(|\feat(M)|^k)$ and for
  each such subset $A$ we test in polynomial-time 
  if $M(e_A)\neq M(e)$. If so, we output that $(M,e,k)$ is a yes-instance
  and if this is not the case for any of the enumerated subsets, we
  output correctly that $(M,e,k)$ is a no-instance.
\end{proof}\fi

The remainder of the section is organized as follows. First in
\Cref{ssec:bc}, we provide a very general result about Boolean circuits,
which will allow us to show a variety of algorithmic results
for our models. We then provide our algorithms for the considered
models in Subsections~\ref{ssec:algDT} to~\ref{ssec:algOBDD}

\subsection{A Meta-Theorem for Boolean Circuits}\label{ssec:bc}

Here, we present our algorithmic result for Boolean circuits that are
allowed to employ majority circuits. In particular, we will show
that all considered explanation problems are fixed-parameter tractable
parameterized by the so-called 
rankwidth of the Boolean circuit as
long as the Boolean circuit uses only a constant number of majority
gates\ifshort; see, e.g.,~\cite{DBLP:journals/jct/OumS06} for a
  definition of rankwidth\fi. Since our considered models can be naturally translated
into Boolean circuits, which require majority gates in the case of
ensembles, we will obtain a rather large number of algorithmic
consequences from this result by providing suitable reductions of our
models to Boolean circuits in the following subsections.

\iflong
We start by introducing Boolean circuits.
A \emph{Boolean circuit (\BC{})}  is a directed acyclic graph $D$ with a unique
sink vertex $o$ (output gate) such that every vertex $v \in V(D)
\setminus \{o\}$ is either:
\begin{itemize}
\item an \emph{input gate (IN-gate)} with no incoming arcs,
\item an \emph{AND-gate} with at least one incoming arc,
\item an \emph{OR-gate} with at least one incoming arcs,
\item a \emph{majority-gate (MAJ-gate)} with at least one incoming arc and an integer
  threshold $t_v$, or
\item a \emph{NOT-gate} with exactly one incoming arc.
\end{itemize}
We denote by $\IG(D)$ the set of all input gates of $D$ and by
$\MAJ(D)$ the set of all MAJ-gates of $D$.
For an assignment $\alpha : \IG(D) \rightarrow
\{0,1\}$ and a vertex $v \in V(D)$, we denote by $\val{v}{D}{\alpha}$ the
value of the gate $v$ after assigning all input gates according to
$\alpha$. That is, $\val{v}{D}{\alpha}$ is recursively defined as follows: If
$v$ is an input gate, then $\val{v}{D}{\alpha}=\alpha(v)$, if $v$ is an
AND-gate (OR-gate), 
then $\val{v}{D}{\alpha}=\bigwedge_{n \in
  N_D^-(v)}\val{n}{D}{\alpha}$ ($\val{v}{D}{\alpha}=\bigvee_{n \in
  N_D^-(v)}\val{n}{D}{\alpha}$), and if $v$ is a MAJ-gate, then
$\val{v}{D}{\alpha}=|\SB n \SM n \in N^-_D(v) \land
\val{n}{D}{\alpha}=1\SE|\geq t_v$. Here and in the following
$N_D^-(v)$ denotes the set of all incoming neighbors of $v$ in $D$.
We set
$\OUT(D,\alpha)=\val{o}{D}{\alpha}$. We say that $D$ is a $c$-\BC{} if
$c$ is an integer and $D$ contains at most $c$ MAJ-gates.
\fi

\newcommand{\MSO}{\textsf{MSO}$_1$}
\newcommand{\MSOE}{\textsf{MSOE}$_1$}

\iflong
We consider \emph{Monadic Second Order} (\MSO{}) logic on structures
representing \BC{}s as a directed (acyclic) graph with unary relations
to represent the types of gates. That is the structure associated with
a given \BC{} $D$ has $V(D)$ as its universe and contains the
following unary and binary relations over $V(D)$:
\begin{itemize}
\item the unary relations $I$, $A$, $O$, $M$, and $N$ containing all
  input gates, all AND-gates, all OR-gates, all MAJ-gates, and all
  NOT-gates of $D$, respectively,
\item the binary relation $E xy$ containing all pairs $x,y \in V(D)$
  such that $(x,y) \in A(D)$.
\end{itemize}
We assume an infinite supply of \emph{individual variables} and
\emph{set variables}, which we denote by lower case and upper case
letters, respectively. The available \emph{atomic formulas} are
$P g$ (``the value assigned to variable $g$ is contained in the unary
relation or set variable $P$''), $E xy$ (``vertex $x$
is the head of an edge with tail $y$''), $x=y$ (equality), and $x\neq y$
(inequality). \emph{\MSO{} formulas} are built up
from atomic formulas using the usual Boolean connectives
$(\lnot,\land,\lor,\rightarrow,\leftrightarrow)$, quantification over
individual variables ($\forall x$, $\exists x$), and quantification over
set variables ($\forall X$, $\exists X$). 

In order to be able to deal with MAJ-gates, we will need a slightly
extended version of \MSO{} logic, which we denote by \MSOE{} and which in turn is a slightly
restricted version of the so-called distance neighborhood logic that
was introduced by \citet{BergougnouxDJ23}. \MSOE{} extends \MSO{} with
\emph{set terms}, which are built from
set-variables, unary relations, or other set terms by
applying standard set operations such as intersection ($\cap$), union
($\cup$), subtraction ($\setminus$), or complementation (denoted by a
bar on top of the term). Note that the distance neighborhood logic
introduces neighborhood terms, which extend set terms by allowing an
additional neighborhood operator on sets that is not required for our
purposes. Like distance neighborhood logic, \MSOE{} also allows for comparisons between set terms, i.e.,
we can write $t_1=t_2$ or $t_1\subseteq t_2$ to express that the set
represented by the set term $t_1$ is a equal or a subset of the set represented by
the set term $t_2$, respectively.
Most importantly for modeling MAJ-gates is that \MSOE{} allows for \emph{size measurement
  of terms}, i.e., for a set term $t$ and an integer $m$, we can write
$|t|\geq m$ to express that the set represented by $t$ contains at
least $m$ elements.

Let $\Phi$ be an \MSOE{} formula (sentence). For
a \BC{} $D$ and possibly an additional unary relation $U \subseteq
V(D)$, we write $(D,U) \models \Phi$ if $\Phi$ holds
true on the structure representing $D$ with additional unary relation $U$. The following proposition
is crucial for our algorithms based on \MSOE{} and essentially
provides an efficient algorithm for a simple optimization variant of
the model checking problem for \MSOE{}.
\begin{proposition}[{\cite[Theorem 1.2]{BergougnouxDJ23}}]\label{pro:MSOE}
  Let $D$ be a $c$-\BC{}, $U \subseteq V(D)$, and let $\Phi(S_1,\dotsc,S_\ell)$ be an \MSOE{} formula with free (non-quantified) set
  variables $S_1,\dotsc,S_\ell$. The problem to compute sets
  $B_1,\dotsc,B_\ell \subseteq V(D)$ such that $(D,U)\models
  \Phi(B_1,\dotsc,B_\ell)$ and $\sum_{i=1}^\ell|B_i|$ is minimum
  is fixed-parameter tractable parameterized by $\rw(D)+|\Phi(S_1,\dotsc,S_\ell)|$.
\end{proposition}
\fi

\ifshort
  To show our algorithmic result for Boolean circuits given below, we
  make use of an only recently developed meta-theorem~\cite[Theorem
  1.2]{BergougnouxDJ23} involving an extension of Monadic second order
  logic that allows us to easily model majority gates of Boolean circuits.
\fi
\ifshort \begin{theorem}[$\star$]
\fi
\iflong \begin{theorem}
\fi    
  \label{the:solve-circ}
  $c$-\BC{}-\MLAEX{}, $c$-\BC{}-\MGAEX{}, $c$-\BC{}-\MLCEX{},
  $c$-\BC{}-\MGCEX{} are fixed-parameter tractable parameterized by
  the rankwidth of the circuit.
\end{theorem}
\iflong
  \begin{proof}
  Let $D$ be a $c$-\BC{} with output gate $o$. We will define one
  \MSOE{} formula for each of the four considered problems. That is, we will define
  the $\Phi_{\textsf{LA}}(S)$, $\Phi_{\textsf{LC}}(S)$,
  $\Phi_{\textsf{GA}}(S_0,S_1)$, and $\Phi_{\textsf{GC}}(S_0,S_1)$ such that:
  \begin{itemize}
  \item $(D,T) \models \Phi_{\textsf{LA}}(S)$ if and only if $S$
    is a local abductive explanation for $e$ w.r.t. $D$. Here, $e$ is
    the given example and $T$ is a unary relation on $V(D)$ given as
    $T=\SB v \in \IG(D) \SM e(v)=1 \SE$.
  \item $(D,T) \models \Phi_{\textsf{LC}}(S)$ if and only if $S$
    is a local contrastive explanation for $e$ w.r.t. $D$. Here, $e$
    is the given example and $T$ is a unary relation on $V(D)$ given
    as $T=\SB v \in \IG(D) \SM e(v)=1 \SE$.
  \item $(D,C) \models \Phi_{\textsf{GA}}(S_0,S_1)$ if and only if the
    partial assignment $\tau : S_0\cup S_1 \rightarrow \{0,1\}$ with
    $\tau(s)=0$ if $s \in S_0$ and $\tau(s)=1$ if $s \in S_1$
    is a global abductive explanation for $c$ w.r.t. $D$.
    Here, $c \in \{0,1\}$ is the given class and $C$ is a unary
    relation on $V(D)$ that is empty if $c=0$ and otherwise contains
    only the output gate $o$.
  \item $(D,C) \models \Phi_{\textsf{GC}}(S_0,S_1)$ if and only if
    the
    partial assignment $\tau : S_0\cup S_1 \rightarrow \{0,1\}$ with
    $\tau(s)=0$ if $s \in S_0$ and $\tau(s)=1$ if $s \in S_1$
    is a global contrastive explanation for $c$ w.r.t. $D$.
    Here, $c \in \{0,1\}$ is the given class and $C$ is a unary
    relation on $V(D)$ that is empty if $c=0$ and otherwise contains
    only the output gate $o$.
  \end{itemize}
  Because each of the formulas will have constant length, the theorem
  then follows immediately from \Cref{pro:MSOE}.

  We start by defining the auxiliary formula $\textsf{CON}(A)$ such that $D
  \models \textsf{CON}(B)$ if and only if $B=\SB v \SM v \in V(D)
  \land \val{v}{D}{\alpha_B}=1\SE$, where $\alpha_B : \IG(D)
  \rightarrow \{0,1\}$ is the assignment of the input gates of $D$
  defined by setting $\alpha_B(v)=1$ if $v \in B$ and $\alpha_B(v)=0$,
  otherwise. In other words $D
  \models \textsf{CON}(B)$ if and only if $B$ represents a consistent
  assignment of the gates, i.e., exactly those gates are in $B$ that
  are set to $1$ if the circuit is evaluated for the input assignment $\alpha_B$.
  The formula $\textsf{CON}(A)$ is equal to
  $\textsf{CON}_\MAJ(A) \land \textsf{CON}'(A)$, where:

  \[
    \textsf{CON}_\MAJ(A) = \begin{array}{ll}\bigwedge_{g' \in \MAJ(D)} g' \in A
      \leftrightarrow 
      (\exists N\ |N|\geq t_{g'} \land \\(\forall n\ n \in N \leftrightarrow n
      \in A \land \textup{IN}(g,n)))
    \end{array}
  \]
  and
  \[ \textsf{CON}'(A) = \begin{array}{ll}
    \forall g &
    (\textup{AND}(g) \rightarrow \\
    & \quad (g \in A \leftrightarrow
    \forall n\ 
    \textup{IN}(g,n) \rightarrow n \in A)) \land\\
              & (\textup{OR}(g) \rightarrow \\
    & \quad (g \in A \leftrightarrow
      \exists n\ 
      \textup{IN}(g,n) \land n \in A)) \land\\
              & (\textup{NOT}(g) \rightarrow \\
    & \quad (g \in A \leftrightarrow
                \exists n\ 
                \textup{IN}(g,n) \land n \notin A))\\
  \end{array}
  \]

  We also need the formula $\textsf{SAT}(E)$ such that $D
  \models \textsf{SAT}(B)$ if and only if the assignment $\alpha_B :
  \IG(D) \rightarrow \{0,1\}$ defined by setting $\alpha_B(v)=1$ if $v
  \in B$ and $\alpha_B(v)=0$ otherwise satisfies the circuit $D$.

  \[ \textsf{SAT}(E) = \exists E'\ E'\cap \IG(D)=E \land
    \textsf{CON}(E') \land o \in E'
  \]

  Finally, we need the formula $\textsf{AGREE}(E,E_0,E_1)$ such that $D
  \models \textsf{AGREE}(B,B_0,B_1)$ if and only if $B_1 \subseteq B$
  and $B_0\cap B=\emptyset$. In other words, the assignment $\alpha_B :
  \IG(D) \rightarrow \{0,1\}$ defined by setting $\alpha_B(v)=1$ if $v
  \in B$ and $\alpha_B(v)=0$ otherwise is $1$ for every feature in
  $B_1$ and $0$ for every feature in $B_0$.

  \[ \textsf{AGREE}(E,E_0,E_1) = E_1\subseteq E \land E_0\cap E=\emptyset
  \]

  We are now ready to define the formula
  $\Phi_{\textsf{LA}}(S)$. 
  
  \[
    \begin{array}{ll}
      \Phi_{\textsf{LA}}(S) = & \exists E_0 \exists E_1\  S=E_0\cup
                                E_1 \land \\
      & E_1 \subseteq T \land E_0 \subseteq
        (\IG(D)\setminus T) \land \\
                              & \forall A \subseteq \G(D)\  \textsf{CON}(A) \rightarrow\\
                              & (\textsf{AGREE}(A,E_0,E_1) \rightarrow (\textsf{SAT}(T) \leftrightarrow o \in A))
    \end{array}
  \]

  Note that the set $S=E_0\cup E_1$ represents the local abductive explanation and
  in particular $E_b$ contains all input gates that are set to $b$ in
  the explanation. Moreover, the set $A$ represents an example
  together with its evaluation in the circuit $D$ and the subformula $(\textsf{SAT}(T) \leftrightarrow o
  \in A)$ is true if and only if the example represented by $A$ is
  classified in the same manner as the example represented by $T$.

  We are now ready to define the formula
  $\Phi_{\textsf{LC}}(S)$. 

  \[
    \begin{array}{ll}
      \Phi_{\textsf{LC}}(S) = & S \subseteq \IG(D)\land \\
                                & \exists A' \subseteq \G(D)\ T\setminus S \subseteq
                                  A' \land S\setminus T \subseteq A'
                                  \land \\
      & A' \subseteq (T \cup S)\setminus
                                  (T\cap S) \land \\
                                & \textsf{CON}(A') \land
                                  (\textsf{SAT}(A') \leftrightarrow
                                  \lnot \textsf{SAT}(T))\\
    \end{array}
  \]

  Here, $S$ represents the set of features in the local contrastive explanation
  and $A'$ represents the example $e'$ (together with its evaluation
  in the circuit $D$) that differs from the example
  $e$ (represented by $T$) only in the features in $S$ and is
  classified differently from $e$.

  We are now ready to define the formula
  $\Phi_{\textsf{GA}}(S_0,S_1)$. 
  
  \[
    \begin{array}{ll}
      \Phi_{\textsf{GA}}(S_0,S_1) = & S_0 \subseteq \IG(D) \land \\
      & S_1 \subseteq \IG(D) \land S_0\cap S_1=\emptyset \land \\
                                    & \forall A \subseteq \G(D)\  \textsf{CON}(A) \rightarrow\\
      & (\textsf{AGREE}(A, S_0,S_1) \rightarrow (C o \leftrightarrow o \in A))
    \end{array}
  \]

  Note that the set $S_0\cup S_1$ represents the global abductive explanation and
  in particular $S_b$ contains all input gates that are set to $b$ in
  the explanation. Moreover, the set $A$ represents an example
  together with its evaluation in the circuit $D$ and the subformula
  $(C o \leftrightarrow o
  \in A)$ is true if and only if the example represented by $A$ is
  classified as $c$.

  We are now ready to define the formula
  $\Phi_{\textsf{GC}}(S_0,S_1)$. 
  
  \[
    \begin{array}{ll}
      \Phi_{\textsf{GC}}(S_0,S_1) = & S_0 \subseteq \IG(D) \land \\
      & S_1 \subseteq \IG(D)\land S_0\cap S_1=\emptyset \land \\
                                    & \forall A \subseteq \G(D)\
                                      (\textsf{CON}(A) \rightarrow\\
      & (\textsf{AGREE}(A, S_0,S_1)
                                      \rightarrow (C o \leftrightarrow o \notin A)
    \end{array}
  \]
\end{proof}
\fi

\subsection{DTs and their Ensembles}\label{ssec:algDT}

Here, we present our algorithms for \DT{}s and their ensembles.
\iflong
We
start with a simple translation from \DT{}s to \BC{}s that allow us to
employ \Cref{the:solve-circ} for \DT{}s.
\ifshort \begin{lemma}[$\star$]\fi\iflong\begin{lemma}\fi\label{lem:dt-trans-circ}
  There is a polynomial-time algorithm that given a \DT{} $\TTT=(T,\lambda)$ and a
  class $c$ produces a circuit $\CIRC(\TTT,c)$ such that:
  \begin{enumerate}[(1)]
  \item for every example $e$, it holds that $\TTT(e)=c$ if and only if
    (the assignment represented by) $e$ satisfies $\CIRC(\TTT,c)$ and
  \item $\rw(\CIRC(\TTT,c)) \leq 3\cdot 2^{|\MNL(\TTT)|}$
  \end{enumerate}
\end{lemma}
\iflong\begin{proof}
  Let $\TTT=(T,\lambda)$ be the given \DT{} and suppose that
  $\MNL(\TTT)$ is equal to the number of negative leaves;
  the construction of the circuit $\CIRC(\TTT,c)$ is analogous if instead
  $\MNL(\TTT)$ is equal to the number of positive leaves.
  We first construct the circuit $D$ such that $D$ is satisfied by $e$
  if and only if $\TTT(e)=0$. $D$ contains one input gate $g_f$
  and one NOT-gate $\overline{g_f}$, whose only incoming arc is from $g_f$,
  for every feature in $\feat(\TTT)$. Moreover, for every $l \in
  L_0(\TTT)$, $D$ contains an AND-gate $g_l$, whose incoming
  arcs correspond to the partial assignment $\alpha_{\TTT}^l$,
  i.e., for every feature $f$ assigned by $\alpha_{\TTT}^l$,
  $g_l$ has an incoming arc from $g_f$ if
  $\alpha_{\TTT}^l(f)=1$ and an incoming arc from
  $\overline{g_f}$ otherwise. Finally, $D$ contains the OR-gate $o$,
  which also serves as the output gate of $D$, that has one incoming
  arc from $g_l$ for every $l \in L_0$. This completes the
  construction of $D$ and it is straightforward to show that $D$ is
  satisfied by an example $e$ if and only if $\TTT(e)=0$.
  Moreover, using \Cref{lem:ranktree}, we obtain that $D$ has treewidth at most $|\MNL(\TTT)|+1$ because
  the graph obtained from $D$ after removing all gates $g_l$ for every
  $l \in L_0(\TTT)$ is a tree and therefore has treewidth at
  most $1$. Therefore, using \Cref{lem:ranktree}, we obtain that $D$ has
  rankwidth at most $3\cdot 2^{|\MNL(\TTT)|}$. Finally, $\CIRC(\TTT,c)$ can now be obtained from
  $D$ as follows. If $c=0$, then $\CIRC(\TTT,c)=D$. Otherwise,
  $\CIRC(\TTT,c)$ is obtained from $D$ after adding one OR-gate
  that also serves as the new output gate of $\CIRC(\TTT,c)$
  and that has only one incoming arc from $o$.
\end{proof}\fi

We now provide a translation from \RF{}s to $1$-\BC{}s that will 
allow us to obtain tractability results for \RF{}s.

\ifshort \begin{lemma}[$\star$]\fi\iflong\begin{lemma}\fi\label{lem:rf-trans-circ}
  There is a polynomial-time algorithm that given a \RF{} $\mathcal{F}$ and a
  class $c$ produces a circuit $\CIRC(\mathcal{F},c)$ such that:
  \begin{enumerate}[(1)]
  \item for every example $e$, it holds that $\mathcal{F}(e)=c$ if and only if
    (the assignment represented by) $e$ satisfies $\CIRC(\mathcal{F},c)$ and
  \item $\rw(\CIRC(\mathcal{F},c)) \leq 3\cdot 2^{\sum_{\TTT \in \mathcal{F}}|\MNL(\TTT)|}$
  \end{enumerate}
\end{lemma}
\iflong\begin{proof}
  We obtain the circuit $\CIRC(\mathcal{F},c)$ from the (not
  necessarily disjoint) union of the circuits $\CIRC(\TTT,c)$
  for every $\TTT \in \mathcal{F}$, which we introduced in
  \Cref{lem:dt-trans-circ},  after adding a new MAJ-gate
  with threshold $\lfloor|\mathcal{F}|/2\rfloor+1$,
  which also serves as the output gate of $\CIRC(\mathcal{F},c)$, that
  has one incoming arc from the output gate of $\CIRC(\TTT,c)$
  for every $\TTT \in \mathcal{F}$. Clearly,
  $\CIRC(\mathcal{F},c)$ satisfies (1). Moreover, to see that it also satisfies
  (2), recall that every circuit $\CIRC(\TTT,c)$ has only
  $\MNL(\TTT)$ gates apart from the input gates, the NOT-gates
  connected to the input gates, and the output gate. Therefore, after
  removing $\MNL(\TTT)$ gates from every circuit
  $\CIRC(\TTT,c)$ inside $\CIRC(\mathcal{F},c)$, the remaining
  circuit is a tree, which together with \Cref{lem:ranktree}
  implies (2).
\end{proof}\fi
\fi
\ifshort
  With the help of our meta-theorem (\Cref{the:solve-circ}) together with natural translations of \DT{}s and
  \RF{}s into \BC{}s and $1$-\BC{}s, respectively, we obtain the
  following two theorems, showing that all problems are
  fixed-parameter tractable parameterized by \enssize{} plus
  \mnlsize{}.
\fi

\ifshort\begin{theorem}[$\star$]\fi\iflong\begin{theorem}\fi\label{th:RF-FPT-MNL}
  Let $\PP \in \{\LAEX, \LCEX, \GAEX, \GCEX\}$. 
  \RF{}-\MPP{}$(\enssize{}+\mnlsize{})$ and therefore also
  \RF{}-\MPP{}$(\enssize{}+\sizeelem{})$ is
  \FPT{}. 
\end{theorem}
\begin{proof}
  The theorem follows immediately from
  \Cref{the:solve-circ} together with \Cref{lem:rf-trans-circ}.
\end{proof}

\iflong
We now give our polynomial-time algorithms for \DT{}s. We start with
the following known result for constrastive explanations.  
\begin{theorem}[{~\cite[Lemma 14]{BarceloM0S20}}]\label{th:DT-MLCEX}
  There is a polynomial-time algorithm that given a \DT{} $\TTT$ and
  an example $e$ outputs a (cardinality-wise) minimum local
  contrastive explanation for $e$ w.r.t. $\TTT$ or no if such an
  explanation does not exist. Therefore, 
  \DT{}-\MLCEX{} can be solved in polynomial-time. 
\end{theorem}
\fi

\ifshort
The following auxiliary lemma provides polynomial-time algorithms for
testing whether a given subset of features $A$ is a local abductive,
global abductive, or global contrastive explanation.  
\fi
\iflong
The following auxiliary lemma provides polynomial-time algorithms for
testing whether a given subset of features $A$ or partial
example $e'$ is a local abductive, global abductive, or global
contrastive explanation for a given example $e$ or class $c$ w.r.t. a
given \DT{} $\TTT$.
\fi
\ifshort \begin{lemma}[$\star$]\fi\iflong\begin{lemma}\fi\label{lem:dt-sm-testsol}
  Let $\TTT$ be a \DT{}, let $e$ be an example and let $c$ be a class.
  There are polynomial-time algorithms for the following problems:
  \begin{enumerate}[(1)]
  \item Decide whether a given subset $A \subseteq \feat(\TTT)$ of
    features is a local abductive explanation for $e$ w.r.t. $\TTT$.
  \item Decide whether a given partial example $e'$ is a global
    abductive/contrastive explanation for $c$ w.r.t. $\TTT$.
  \end{enumerate}
\end{lemma}
\iflong\begin{proof}\fi
\ifshort\begin{proof}[Proof Sketch]\fi  
  Let $\TTT$ be a \DT{}, let $e$ be an example and let $c$ be a class.
  Note that we assume here that $\TTT$ does not have any contradictory
  path\ifshort.\fi\iflong, i.e., a root-to-leaf path that contains that assigns any
  feature more than once. Because if this was not the case, we could
  easily simplify $\TTT$ in polynomial-time. \fi

  We start by showing (1).
  A subset $A \subseteq \feat(\TTT)$ of
  features is a
  local abductive explanation for $e$ w.r.t. $\TTT$ if and only if the
  \DT{} $\TTT_{|e_{|A}}$ does only contain
  $\TTT(e)$-leaves, which can clearly be decided in polynomial-time.
  Here, $e_{|A}$ is the partial example equal to the
  restriction of $e$ to $A$. Moreover, $\TTT_{|e'}$ for a partial
  example $e'$ is the \DT{} obtained from $\TTT$ after removing every
  $1-e'(f)$-child from every node $t$ of $\TTT$ assigned to a feature
  $f$ for which $e'$ is defined. \ifshort The proof for (2) is similar.\fi
  \iflong
    
  Similarly, for showing (2), observe that partial example (assignment) $\tau : F \rightarrow
  \{0,1\}$ is a global abductive explanation for $c$ w.r.t. $\TTT$ if and only if the \DT{}
  $\TTT_{|\tau}$ does only contain
  $c$-leaves, which can clearly be decided in polynomial-time.

  Finally, note that a partial example $\tau : F \rightarrow
  \{0,1\}$ is a global contrastive explanation for $c$ w.r.t. $\TTT$ if and only if the \DT{}
  $\TTT_{|\tau}$ does not contain any
  $c$-leaf, which can clearly be decided in polynomial-time.\fi
\end{proof}

Using dedicated algorithms for the inclusion-wise minimal variants of
\LAEX{}, \GAEX{}, \iflong and \fi \GCEX{} \iflong together with
  \Cref{th:DT-MLCEX}\fi\ifshort and using the
  polynomial-time algorithm for the
  cardinality-wise minimal version of \LCEX{} given in~\cite[Lemma
  14]{BarceloM0S20}\fi, we obtain the following result.
\ifshort \begin{theorem}[$\star$]\fi\iflong\begin{theorem}\fi\label{th:poly-DT-SM}
    Let $\PP \in \{\LAEX, \LCEX, \GAEX, \GCEX\}$. 
  \DT{}-\SPP{} \ifshort and \DT{}-\MLCEX{} \fi can be solved in
  polynomial-time. 
\end{theorem}
\iflong\begin{proof}\fi
\ifshort\begin{proof}[Proof Sketch]\fi  
  Note that the statement of the theorem for \DT{}-\SMLCEX{} follows
  immediately from\iflong~\Cref{th:DT-MLCEX}\fi\ifshort~\cite[Lemma 14]{BarceloM0S20}\fi. Therefore, it suffices to show
  the statement of the theorem for the remaining 3 problems.
  \iflong
  
  Let $(\TTT,e)$ be an instance of  \DT{}-\SMLAEX{}. We start by setting $A=\feat(\TTT)$. Using \Cref{lem:dt-sm-testsol}, we
  then test for any feature $f$ in $A$, whether $A\setminus \{f\}$
  is still a local abductive explanation for $e$ w.r.t. $\TTT$ in polynomial-time. If so,
  we repeat the process after setting $A$ to $A\setminus\{f\}$ and
  otherwise we do the same test for the next feature $f \in
  A$. Finally, if $A\setminus\{f\}$
  is not a local abductive explanation for every $f \in A$, then $A$ is an inclusion-wise
  minimal local abductive explanation and we can output $A$.
  \fi
  
  The polynomial-time algorithm for an instance $(\TTT,c)$ of
  \DT{}-\SMGAEX{} \iflong now\fi works as
  follows. Let $l$ be a $c$-leaf of $\TTT$; if no such $c$-leaf
  exists, then we can correctly output that there is no global
  abductive explanation for $c$ w.r.t. $\TTT$. Then, $\alpha_\TTT^l$
  is a global abductive explanation for $c$ w.r.t. $\TTT$. To obtain
  an inclusion-wise minimal solution, we do the following. Let $F=\feat(\alpha_\TTT^l)$ be
  the set of features on which $\alpha_\TTT^l$ is defined. We now test for
  every feature $f \in F$ whether the restriction
  $\alpha_\TTT^l[F\setminus \{f\}]$ of $\alpha_\TTT^l$ to $F\setminus
  \{f\}$ is a global abductive explanation for $c$ w.r.t. $\TTT$. This
  can clearly be achieved in polynomial-time with the help of
  \Cref{lem:dt-sm-testsol}. If this is true for any feature $f \in F$,
  then we repeat the process for $\alpha_\TTT^l[F\setminus \{f\}]$,
  otherwise we output $\alpha_\TTT^l$. \ifshort Very similar
    algorithms now also work for \DT{}-\SMGCEX{} and \DT{}-\SMLAEX{}.\fi \iflong A very similar algorithm now
  works for the \DT{}-\SMGCEX{} problem, i.e., instead of starting with a $c$-leaf
  $l$ we start with a $c'$-leaf, where $c \neq c'$.\fi
\end{proof}

The following theorem uses an exhaustive enumeration of all possible
explanations together with \Cref{lem:dt-sm-testsol} to check whether a
set of features or a partial example is an explanation.
\ifshort \begin{theorem}[$\star$]\fi\iflong\begin{theorem}\fi\label{th:DT-LGA-XP}
  Let $\PP \in \{\LAEX, \GAEX, \GCEX\}$. 
  \DT{}-\MPP{}$(\xpsize{})$ is in \XP{}.
\end{theorem}
\iflong\begin{proof}
  We start by showing the statement of theorem for \DT{}-\MLAEX{}.
  Let $(\TTT,e,k)$ be
  an instance of \DT{}-\MLAEX{}. We first enumerate all subsets $A
  \subseteq \feat(\TTT)$ of size at most $k$ in time
  $\bigoh(|\feat(\TTT)|^k)$. For every such subset $A$, we then test
  whether $A$ is a local abductive explanation for $e$ w.r.t. $\TTT$
  in polynomial-time with the help of \Cref{lem:dt-sm-testsol}. If so,
  we output $A$ as the solution. Otherwise, i.e., if no such subset is
  a local abductive explanation for $e$ w.r.t. $\TTT$, we output
  correctly that $(\TTT,e,k)$ has no solution.

  Let $(\TTT,c,k)$ be
  an instance of \DT{}-\MGAEX{}. We first enumerate all subsets $A
  \subseteq \feat(\TTT)$ of size at most $k$ in time
  $\bigoh(|\feat(\TTT)|^k)$. For every such subset $A$, we then
  enumerate all of the at most $2^{|A|}\leq 2^k$ partial examples
  (assignments) $\tau : A \rightarrow\{0,1\}$ in time $\bigoh(2^k)$.
  For every such partial example $\tau$, we then use
  \Cref{lem:dt-sm-testsol} to test whether $\tau$ is a global
  abductive explanation for $c$ w.r.t. $\TTT$ in polynomial-time. If so,
  we output $e$ as the solution. Otherwise, i.e., if no such
  partial example is
  a global abductive explanation for $c$ w.r.t. $\TTT$, we output
  correctly that $(\TTT,c,k)$ has no solution. The total runtime of
  the algorithm is at most $2^k|\feat(\TTT)|^k|\TTT|^{\bigoh(1)}$.

  The algorithm for \DT{}-\MGCEX{} is now very similar to the above
  algorithm for \DT{}-\MGAEX{}.
\end{proof}\fi

The next theorem uses our result that the considered problems are in
polynomial-time for \DT{}s \ifshort(\Cref{th:poly-DT-SM}) \fi
together with an \XP{}-algorithm that
transforms any \RF{} into an equivalent \DT{}.
\ifshort \begin{theorem}[$\star$]\fi\iflong\begin{theorem}\fi\label{th:Rf-XP-e}
  Let $\PP \in \{\LAEX, \LCEX, \GAEX, \GCEX\}$. 
  \RF-\SPP{}$(\enssize{})$ and \RF-\MLCEX{}$(\enssize{})$ are in \XP{}.
\end{theorem}
\iflong\begin{proof}
  Let $\FFF=\{\TTT_1,\dotsc,\TTT_\ell\}$ be the random forest given as
  an input to any of the five problems stated above. We start by constructing a
  \DT{} $\TTT$ that is not too large (of size at most $m^{|\FFF|}$,
  where $m=(\max \SB |L(\TTT_i)|\SM 1 \leq i \leq \ell\SE)$) and that is equivalent to $\mathcal{F}$
  in the sense that $\mathcal{F}(e)=\TTT(e)$ for every
  example in time at most $\bigoh(m^{|\FFF|})$. Since all five
  problems can be solved in polynomial-time on $\TTT$ (because of
  \Cref{th:poly-DT-SM}) this completes the
  proof of the theorem.

  We construct
  $\TTT$ iteratively as follows. First we set
  $\TTT_1'=\TTT_1$. Moreover, for every $i>1$,
  $\TTT_i'$ is obtained from $\TTT_{i-1}'$ and
  $\TTT_i$ by making a copy $\TTT_i^l$ of $\TTT_i$ for every leaf $l$
  of $\TTT_{i-1}'$ and by identifying the root of
  $\TTT_i^l$ with the leaf $l$ of $\TTT_{i-1}'$. Then,
  $\TTT$ is obtained from $\TTT_{\ell}'$ changing the
  label of every leaf $l$ be a leaf of $\TTT_{\ell}'$ as
  follows. Let $P$ be the path
  from the root of $\TTT_{\ell}'$ to $l$ in
  $\TTT_{\ell}'$. By construction, $P$ goes through one copy of
  $\TTT_i$ for every $i$ with $1 \leq i \leq \ell$ and
  therefore also goes through exactly one leaf of every
  $\TTT_i$. We now label $l$ according to the majority of the
  labels of these leaves. This completes the construction of
  $\TTT$ 
  and it is easy to see that
  $\mathcal{F}(e)=\TTT(e)$ for every example $e$. Moreover, the
  size of $\TTT$ is at most
  $\Pi_{i=1}^\ell|L(\TTT_i)|\leq m^\ell$, where $m=(\max \SB |L(\TTT_i)|\SM
  1 \leq i \leq \ell\SE)$ and $\TTT$ can be constructed in time $\bigoh(m^\ell)$.
\end{proof}\fi

\subsection{DSs, DLs and their Ensembles}

This subsection is devoted to our algorithmic results for \DS{},
\DL{}s and their ensembles. Our first algorithmic result is again
based on our meta-theorem (\Cref{the:solve-circ}) and a suitable
translation from \DSE{} and \DLE{} to a Boolean circuit.

\iflong
\ifshort \begin{lemma}[$\star$]\fi\iflong\begin{lemma}\fi\label{lem:dsl-trans-circ}
  There is a polynomial-time algorithm that given a \DS{}/\DL{} $L$ and a
  class $c$ produces a circuit $\CIRC(L,c)$ such that:
  \begin{itemize}
  \item for every example $e$, it holds that $L(e)=c$ if and only if $e$ satisfies $\CIRC(L,c)$
  \item $\rw(\CIRC(L,c)) \leq 3\cdot 2^{3|L|}$
  \end{itemize}
\end{lemma}
\iflong\begin{proof}
  Since every \DS{} can be easily transformed into a \DL{} with the
  same number of terms, it suffices to show the lemma for \DL{}s. Let
  $L$ be a \DL{} with rules
  $(r_1=(t_1,c_1),\dotsc,r_\ell=(t_\ell,c_\ell))$.
  We construct the circuit $D=\CIRC(L,c)$ as follows.
  $D$ contains one input gate $g_f$
  and one NOT-gate $\overline{g_f}$, whose only incoming arc is from $g_f$,
  for every feature in $\feat(L)$. Furthermore, for every rule
  $r_i=(t_i,c_i)$, $D$ contains an AND-gate $g_{r_i}$, whose
  in-neighbors are the literals in $t_i$, i.e., if $t_i$ contains a
  literal $f=0$, then $g_{r_i}$ has $\overline{g_f}$ as an in-neighbor
  and if $t_i$ contains a literal $f=1$, then $g_{r_i}$ has $g_f$ as
  an in-neighbor. We now split the sequence
  $\rho=(r_1=(t_1,c_1),\dotsc,r_\ell=(t_\ell,c_\ell))$ into
  (inclusion-wise) maximal consecutive subsequences $\rho_i$ of rules
  that have the same class. Let $(\rho_1,\dotsc,\rho_r)$ be the
  sequence of subsequences obtained in this manner, i.e.,
  $\rho_1 \concat \rho_2 \concat \dotsb \concat \rho_r=\rho$, every
  subsequence $\rho_i$ is non-empty and contains only rules from one
  class, and the class of any rule in $\rho_i$ is different from the
  class of any rule in $\rho_{i+1}$ for every $i$ with $1 \leq i <
  r$. Now, for every subsequence $\rho_i$, we add an OR-gate
  $g_{\rho_i}$ to $D$, whose in-neighbors are
  the gates $g_{r}$ for every rule $r$ in $\rho_i$. Let $C$ be the set
  of all subsequences $\rho_i$ that only contains rules with class $c$
  and let $\overline{C}$ be the set of all other subsequences, i.e.,
  those that contain only rules whose class is not equal to $c$.
  For every subsequence $\rho$ in $\overline{C}$, we add a NOT-gate
  $\overline{g_\rho}$ to $D$, whose in-neighbor is the gate $g_\rho$.
  Moreover, for every subsequence $\rho_i$ in $C$, we add an AND-gate
  $g_{\rho_i}^A$ to $D$ whose in-neighbors are $g_{\rho_i}$ as well as
  $\overline{g_{\rho_j}}$ for every $\rho_j \in \overline{C}$ with
  $j<i$. Finally, we add an OR-gate $o$ to $D$ that also serves as the
  output gate of $D$ and whose in-neighbors are all the gates
  $g_{\rho}^A$ for every $\rho \in C$. This completes the construction
  of $D$ and it is easy to see that $L(e)=c$ if and only if $e$
  satisfies $D$ for every example $e$, which shows that $D$ satisfies (1).
  Towards showing (2), let $G$ be the set of all gates in $D$ apart
  from the gates $o$ and $g_f$ and $\overline{g_f}$ for every feature
  $f$. Then, $|G|\leq 3|L|$ and $D\setminus G$ is a forest, which
  together with \Cref{lem:ranktree}
  implies (2).
\end{proof}\fi

\ifshort \begin{lemma}[$\star$]\fi\iflong\begin{lemma}\fi\label{lem:dsle-trans-circ}
  Let $\MM \in \{\DSE,\DLE\}$. There is a polynomial-time algorithm that given an $\MM$ $\mathcal{L}$ and a
  class $c$ produces a circuit $\CIRC(\mathcal{L},c)$ such that:
  \begin{enumerate}[(1)]
  \item for every example $e$, it holds that $\mathcal{L}(e)=c$ if and only if $e$ satisfies $\CIRC(\mathcal{L},c)$
  \item $\rw(\CIRC(\mathcal{L},c)) \leq 3 \cdot 2^{3\sum_{L \in \mathcal{L}}|L|}$
  \end{enumerate}
\end{lemma}
\iflong\begin{proof}
  We obtain the circuit $\CIRC(\mathcal{L},c)$ from the (not
  necessarily disjoint) union of the circuits $\CIRC(L,c)$, which are
  provided in \Cref{lem:dsl-trans-circ},
  for every $L \in \mathcal{L}$ after adding a new MAJ-gate
  with threshold $\lfloor|\mathcal{L}|/2\rfloor+1$,
  which also serves as the output gate of $\CIRC(\mathcal{L},c)$, that
  has one incoming arc from the output gate of $\CIRC(L,c)$
  for every $L \in \mathcal{L}$. Clearly,
  $\CIRC(\mathcal{L},c)$ satisfies (1). Moreover, to see that it also satisfies
  (2), recall that every circuit $\CIRC(L,c)$ has only
  $3|L|$ gates apart from the input gates, the NOT-gates
  connected to the input gates, and the output gate. Therefore, after
  removing $3|L|$ gates from every circuit
  $\CIRC(L,c)$ inside $\CIRC(\mathcal{L},c)$, the remaining
  circuit is a tree, which together with \Cref{lem:ranktree} implies (2).
\end{proof}\fi
The following corollary now follows immediately from
\Cref{lem:dsle-trans-circ} and \Cref{the:solve-circ}.
\fi
\begin{corollary}\label{cor:ds-ensa}
  Let $\MM \in \{\DSE,\DLE\}$ and let  $\PP \in \{\LAEX, \LCEX, \GAEX, \GCEX\}$. 
  \MM-\MPP{}$(\enssize+\termselem)$ is \FPT{}.
\end{corollary}
Unlike, \DT{}s, where \DT{}-\MLCEX{} is solvable in polynomial-time,
this is not the case for \DS{}-\MLCEX{}. Nevertheless, we are able to
provide the following result, which shows that \DS{}-\MLCEX{} (and
even \DL{}-\MLCEX{}) is fixed-parameter tractable parameterized by
\termsize and \xpsize. The algorithm is based on a novel
characterization of local contrastive explanations for \DL{}s.

\ifshort \begin{lemma}[$\star$]\fi\iflong\begin{lemma}\fi\label{lem:dl-mlcex-branch}
    Let $\MM \in \{\DS,\DL\}$. \MM{}-\MLCEX{} for 
    $M \in \MM{}$ and integer $k$
    can be solved in time $\bigoh(a^k\som{M}^2)$, where $a$ is equal to \termsize.
\end{lemma}
\iflong\begin{proof}\fi
\ifshort\begin{proof}[Proof Sketch]\fi  
  Since any \DS{} can be easily translated into a \DL{} without
  increasing the size of any term, it suffices to show the lemma for \DL{}s.  
  Let $(L,e,k)$ be an instance of \DL{}-\MLCEX{}, where
  $L=(r_1=(t_1,c_1),\dotsc,r_\ell=(t_\ell,c_\ell))$ is a \DL{},
   and let $r_i$ be the rule that
  classifies $e$, i.e., the first rule that applies to $e$.

  Let $R$ be the set of all rules $r_j$ of $L$ with $c_j\neq c_i$. For
  a rule $r \in R$, let $A \subseteq \feat(L)$ such that the example
  $e_A$, i.e., the example obtained from $e$ after setting
  $e_A(f)=1-e(f)$ for every $f \in A$ and $e_A(f)=e(f)$ otherwise, is
  classified by rule $r$. We claim that:
  \begin{enumerate}[(1)]
  \item For every $r \in R$ and every set $A \subseteq \feat(\TTT)$
    such that $e_A$ is classified by $r$, it holds that $A$ is a local
    contrastive explanation for $e$ w.r.t. $L$.
  \item Every local contrastive explanation $A$ for $e$
    w.r.t. $L$ contains a subset $A' \subseteq A$ for which 
    there is a rule $r \in R$ such that $e_{A'}$ is
    classified by $r$.
  \end{enumerate}
  \ifshort
    Because of (1) and (2), it holds that a set $A \subseteq
    \feat(\TTT)$ is a local contrastive explanation if and only if
    there is a rule $r \in R$ such that $e_A$ is classified by
    $r$. Therefore, it is sufficient to be able to compute
    a minimum set of features $A$ such that $e_A$ is classified by $r$
    for every rule $r \in R$, which can be 
    achieved via a bounded-depth branching algorithm.
  \fi
  \iflong

  Towards showing (1), let $r \in R$ and let $A \subseteq \feat(\TTT)$
  such that $e_A$ is classified by $r$. Then, $e_A$ differs from $e$
  only on the features in $A$ and moreover $M(e)\neq M(e_A)$ because
  $r \in R$. Therefore, $A$ is a local contrastive explanation for $e$
  w.r.t. $\TTT$.

  Towards showing (2), let $A \subseteq \feat(A)$ be a local
  contrastive explanation for $e$ w.r.t. $\TTT$. Then, there is an
  example $e'$ that differs from $e$ only on some set $A'\subseteq A$
  of features such
  that $e'$ is classified by a rule $r \in R$. Then, $e_{A'}$ is
  classified by $r$, showing (2).
    
  \begin{algorithm}[tb]
    \caption{Algorithm used in \Cref{lem:dl-mlcex-branch} to compute a
      local contrastive explanation for example $e$ w.r.t. a \DL{} $L$
      of size at most $k$ if such an explanation exists. The algorithm
      uses the function \textbf{findLCEXForRule}($L$, $e$, $r_j$) as a
      subroutine, which is illustrated in \Cref{alg:findLCEXForRule}}\label{alg:findLCEX}
    \small
    \begin{algorithmic}[1]
      \INPUT \DL{} $L=(r_1=(t_1,c_1),\dotsc,r_\ell=(t_\ell,c_\ell))$,
      example $e$ and integer $k$ (global variable).
      \OUTPUT return a cardinality-wise minimum local contrastive
      explanation $A \subseteq \feat(L)$ with $|A|\leq k$ for $e$
      w.r.t. $L$ or \NULL{} if such an explanation does not exists.
      \Function{\textbf{findLCEX}}{$L$, $e$}
      \State $r_i \gets$ rule of $L$ that classifies $e$
      \State $R \gets \SB r_j \in L \SM c_j\neq c_i\}$
      \State $A_b\gets \NULL$
      \For{$r_j \in R$}
      \State $A \gets$ \Call{findLCEXForRule}{$L$, $e$, $r_j$}
      \If{$A \neq \NULL$ and ($A_b=\NULL$ or $|A_b|>|A|$)}
      \State $A_b\gets A$
      \EndIf
      \EndFor
      \State \Return $A_b$
      \EndFunction
    \end{algorithmic}
  \end{algorithm}

  \begin{algorithm}[tb]
    \caption{Algorithm used as a subroutine in \Cref{alg:findLCEX}
      to compute a smallest set $A$
      of at most $k$ features such that $e_A$ is classified by the
      rule $r_j$ of a \DL{} $L$.}\label{alg:findLCEXForRule}
    \small
    \begin{algorithmic}[1]
      \INPUT \DL{} $L=(r_1=(t_1,c_1),\dotsc,r_\ell=(t_\ell,c_\ell))$,
      example $e$, rule $r_j$, and integer $k$ (global variable).
      \OUTPUT return a smallest set $A \subseteq \feat(L)$ with $|A|\leq k$
      such that $e_A$ is classified by $r_j$ if such a set exists and
      otherwise return \NULL{}.
      \Function{\textbf{findLCEXForRule}}{$L$, $e$, $r_j$}
      \State $A_0 \gets \SB f \SM (f=1-e(f)) \in t_j\SE$
      \State \Return \Call{findLCEXForRuleRec}{$L$, $e$, $r_j$, $A_0$}
      \EndFunction
      \Function{\textbf{findLCEXForRuleRec}}{$L$, $e$, $r_j$, $A'$}
      \If{$|A'|>k$}\label{algLLCEXRecOne}
      \State \Return \NULL
      \EndIf
      \State $r_\ell\gets$ any rule $r_\ell$ with $\ell<j$ that is satisfied by $e_{A'}$
      \If{$r_\ell=\NULL{}$}
      \State \Return $A'$\label{algLLCEXRecTwo}
      \EndIf
      \State $F_j\gets \SB f \SM (f=b) \in t_j \land b \in \{0,1\}\SE$
      \State $B \gets \SB f \SM (f=e_{A'}) \in t_\ell\SE\setminus (A'\cup F_j)$\label{algLLCEXRecThree}
      \If{$B=\emptyset$}
      \State \Return \NULL\label{algLLCEXRecFour}
      \EndIf
      \State $A_b\gets \NULL$
      \For{$f \in B$}
      \State $A\gets $\Call{findLCEXForRuleRec}{$L$, $e$, $r_j$, $A'\cup \{f\}$}\label{algLLCEXRecFive}
      \If{$A\neq \NULL$ and $|A|\leq k$}
      \If{($A_b==\NULL$ or $|A_b|>|A|$)}
      \State $A_b \gets A$
      \EndIf
      \EndIf
      \EndFor
      \Return $A_b$\label{algLLCEXRecSix}
      \EndFunction
    \end{algorithmic}
  \end{algorithm}

  Note that due to (1) and (2), we can compute a smallest local
  contrastive explanation for $e$ w.r.t. $L$ by the algorithm
  illustrated in \Cref{alg:findLCEX}. That is, the algorithm computes
  the smallest set $A$ of at most $k$ features such that $e_A$ is
  classified by $r_j$ for every rule $r_j \in R$. It then, returns the
  smallest such set over all rules $r_j$ in $R$ if such a set existed
  for at least one of the rules in $R$. The main ingredient of the
  algorithm is the function \Call{findLCEXForRule}{$L$, $e$, $r_j$},
  which is illustrated in \Cref{alg:findLCEXForRule}, and
  that for a rule $r_j \in R$ computes the smallest set $A$ of at most
  $k$ features such that $e_A$ is classified by $r_j$. The function
  \Cref{alg:findLCEXForRule} achieves this as follows.
  It first computes the set $A_0=\SB f \SM (f=1-e(f)) \in t_j\SE$ of
  features that need to be part of $A$ in order for $e_A$ to satisfy
  $r_j$, i.e., all features $f$ where $e(f)$ differs from the literal
  in $t_j$. It then, computes the set $A$ recursively via the
  (bounded-depth) branching algorithm given in
  \Call{findLCEXForRuleRec}{$L$, $e$, $r_j$, $A$}, which given a set
  $A'$ of features such that $e_{A'}$ satisfies $r_j$ computes a
  smallest extension (superset) $A$ of $A'$ of size at most $k$ such
  that $e_A$ is classified by $r_j$. To do so the function first
  checks in \Cref{algLLCEXRecOne} whether $|A'|>k$ and if so correctly returns
  \NULL{}. Otherwise, the function checks whether there is any rule $r_\ell$ with
  $\ell<j$ that is satisfied by $e_{A'}$. If that is not the case, it
  correctly returns $A'$ (in \Cref{algLLCEXRecTwo}). Otherwise, the
  function computes the set $B$ of features in \Cref{algLLCEXRecThree} that occur in $t_\ell$ --
  and therefore can be used to falsify $r_\ell$ -- but do not occur in
  $A'$ or in $t_j$. Note that we do not want to include any feature in
  $A'$ or $t_j$ in $B$ since changing those features would either
  contradict previous branching decisions made by our algorithm or it
  would prevent us from satisfying $r_j$. The function then returns
  \NULL{} if the set $B$ is empty in \Cref{algLLCEXRecFour}, since in this case it is no longer
  possible to falsify the rule $r_\ell$. Otherwise, the algorithm
  branches on every feature in $f \in B$ and tries to extend $A'$ with
  $f$ using a recursive call to itself with $A'$ replaced by $A'\cup
  \{f\}$ in \Cref{algLLCEXRecFive}. It then returns the best solution
  found by any of those recursive calls in \Cref{algLLCEXRecSix}.
  This completes the description of the algorithm, which can be easily
  seen to be correct using (1) and (2).

  We are now ready to analyze the runtime of the algorithm, i.e.,
  \Cref{alg:findLCEX}. First note that all operations in
  \Cref{alg:findLCEX} and \Cref{alg:findLCEXForRule} that are not recursive
  calls take time at most $\bigoh(\som{L})$. We start by 
  analyzing the run-time of the function \Call{findLCEXForRule}{$L$,
    $e$, $r_j$}, which is at most $\bigoh(\som{L})$ times the number
  of recursive calls to the function
  \Call{findLCEXForRuleRec}{$L$, $e$, $r_j$, $A$}, which in turn can be easily
  seen to be at most $a^r$ since the function branches into at most
  $a$ branches at every call and the depth of the branching is at most $k$.
  Therefore, we obtain $\bigoh(a^r\som{L})$ as the total runtime of
  the function \Call{findLCEXForRule}{$L$, $e$, $r_j$}, which implies
  a total runtime of the algorithm of $\bigoh(a^r\som{L}^2)$.\fi
\end{proof}

The following lemma is now a natural extension of
\Cref{lem:dl-mlcex-branch} for ensembles of \DL{}s.
\ifshort \begin{lemma}[$\star$]\fi\iflong\begin{lemma}\fi\label{lem:dle-mlcex-branch}
    Let $\MM \in \{\DSE,\DLE\}$. \MM{}-\MLCEX{} for $M \in \MM$
    and integer $k$
    can be solved in time $\bigoh(m^sa^k\som{M}^2)$, where $m$ is
    \termselem, $s$ is \enssize, and $a$ is \termsize.
\end{lemma}
\iflong\begin{proof}
  Since any \DSE{} can be easily translated into a \DLE{} without
  increasing the size of any term and where every ensemble element has
  at most one extra rule using the empty term, it suffices to show the lemma for \DLE{}s.  
  The main ideas behind the algorithm, which is illustrated in \Cref{alg:findELCEX}, are similar to the ideas behind
  the algorithm used in \Cref{lem:dl-mlcex-branch} for a single \DL{}.

  \begin{algorithm}[tb]
    \caption{Algorithm used in \Cref{lem:dle-mlcex-branch} to compute a
      local contrastive explanation for example $e$ w.r.t. a \DLE{}
      $\dle=\{\dl_1,\dotsc,\dl_\ell\}$
      with $\dl_i=(r_1^i=(t_1^i,c_1^i),\dotsc,r_\ell^i=(t_\ell^i,c_\ell^i))$
      of size at most $k$ if such an explanation exists. The algorithm
      uses the function \textbf{findELCEXForRules}($\dle$, $e$, $(r_{j_1}^1,\dotsc,r_{j_\ell}^\ell)$) as a
      subroutine, which is illustrated in \Cref{alg:findELCEXForRules}.}\label{alg:findELCEX}
    \small
    \begin{algorithmic}[1]
      \INPUT \DLE{} $\dle=\{\dl_1,\dotsc,\dl_\ell\}$ with $\dl_i=(r_1^i=(t_1^i,c_1^i),\dotsc,r_\ell^i=(t_\ell^i,c_\ell^i))$,
      example $e$ and integer $k$ (global variable).
      \OUTPUT return a cardinality-wise minimum local contrastive
      explanation $A \subseteq \feat(\dle)$ with $|A|\leq k$ for $e$
      w.r.t. $\dle$ or \NULL{} if such an explanation does not exists.
      \Function{\textbf{findELCEX}}{$\dle$, $e$}
      \For{$o \in [k]$}
      \State $R^o \gets \SB r \in \dl_o \SE$
      \EndFor

      \State $A_b\gets \NULL$
      \For{$(r_{j_1}^1,\dotsc,r_{j_\ell}^\ell) \in R^1\times \dotsb
        \times R^\ell$}\label{algLELCEXOne}
      \State $n_{\neq}\gets |\SB o \in [\ell] \SM c_{j_o}^o\neq
      \dle(e)\SE|$
      \State $n_= \gets |\SB o \in [\ell] \SM c_{j_o}^o= \dle(e)\SE|$
      \If{$n_{\neq} > n_=$}\label{algLELCEXTwo}
      \State $A \gets$ \Call{findELCEXR}{$\dle$, $e$, $(r_{j_1}^1,\dotsc,r_{j_\ell}^\ell)$}
      \If{$A \neq \NULL$ and ($A_b=\NULL$ or $|A_b|>|A|$)}
      \State $A_b\gets A$
      \EndIf
      \EndIf
      \EndFor
      \State \Return $A_b$
      \EndFunction
    \end{algorithmic}
  \end{algorithm}

  \begin{algorithm}[tb]
    \caption{Algorithm used as a subroutine in \Cref{alg:findELCEX}
      to compute a smallest set $A$
      of at most $k$ features such that $e_A$ is classified by the
      rule $r_j^o$ for every \DL{} $\dl_o$.}\label{alg:findELCEXForRules}
    \small
    \begin{algorithmic}[1]
      \INPUT \DL{} \DLE{} $\dle=\{\dl_1,\dotsc,\dl_\ell\}$ with $\dl_i=(r_1^i=(t_1^i,c_1^i),\dotsc,r_\ell^i=(t_\ell^i,c_\ell^i))$,
      example $e$, rules $(r_{j_1}^1,\dotsc,r_{j_\ell}^\ell)$, and integer $k$ (global variable).
      \OUTPUT return a smallest set $A \subseteq \feat(L)$ with $|A|\leq k$
      such that $e_A$ is classified by $r_{j_o}^o$ for every \DL{} $\dl_o$ if such a set exists and
      otherwise return \NULL{}.
      \Function{\textbf{findELCEXR}}{$\dle$, $e$, $(r_{j_1}^1,\dotsc,r_{j_\ell}^\ell)$}
      \State $A_0 \gets \SB f \SM (f=1-e(f)) \in t_{j_o}^o \land o \in
      [\ell]\SE$
      \State \Return \Call{findELCEXRR}{$\dle$, $e$, $(r_{j_1}^1,\dotsc,r_{j_\ell}^\ell)$, $A_0$}
      \EndFunction
      \Function{\textbf{findELCEXRR}}{$\dle$, $e$, $(r_{j_1}^1,\dotsc,r_{j_\ell}^\ell)$, $A'$}
      \If{$|A'|>k$}\label{algLELCEXRecOne}
      \State \Return \NULL
      \EndIf
      \State $r_\ell^o\gets$ any rule with $\ell<j_o$ that
      is satisfied by $e_{A'}$ for $o \in [\ell]$
      \If{$r_\ell^o=\NULL{}$}
      \State \Return $A'$\label{algLELCEXRecTwo}
      \EndIf
      \State $F'\gets \SB f \SM (f=b) \in t_{j_o}^o \land o \in [\ell] \land b \in \{0,1\}\SE$
      \State $B \gets \SB f \SM (f=e_{A'}) \in t_\ell^o\SE\setminus (A'\cup F')$\label{algLELCEXRecThree}
      \If{$B=\emptyset$}
      \State \Return \NULL\label{algLELCEXRecFour}
      \EndIf
      \State $A_b\gets \NULL$
      \For{$f \in B$}
      \State $A\gets $\Call{findELCEXRR}{$\dle$, $e$, $(r_{j_1}^1,\dotsc,r_{j_\ell}^\ell)$, $A'\cup \{f\}$}\label{algLELCEXRecFive}
      \If{$A\neq \NULL$ and $|A|\leq k$}
      \If{($A_b==\NULL$ or $|A_b|>|A|$)}
      \State $A_b \gets A$
      \EndIf
      \EndIf
      \EndFor
      \Return $A_b$\label{algLELCEXRecSix}
      \EndFunction
    \end{algorithmic}
  \end{algorithm}

  Let $(\dle,e,k)$ be an instance of \DLE{}-\MLCEX{} \DLE{} with
  $\dle=\{\dl_1,\dotsc,\dl_\ell\}$ and
  $\dl_i=(r_1^i=(t_1^i,c_1^i),\dotsc,r_\ell^i=(t_\ell^i,c_\ell^i))$
  for every $i \in [\ell]$. First note that if $A$ is a local contrastive explanation
  for $e$ w.r.t. $\dle$, then there is an example $e'$ that differs
  from $e$ only on the features in $A$ such that $\dle(e)\neq
  \dle(e')$. Clearly $e'$ is classified by exactly one rule of every
  \DL{} $\dl_i$. The first idea behind \Cref{alg:findELCEX} is
  therefore to enumerate all possibilities for the rules that classify
  $e'$; this is done in \Cref{algLELCEXOne} of the algorithm. Clearly,
  only those combinations of rules are relevant that lead to a
  different classification of $e'$ compared to $e$ and this is ensured
  in \Cref{algLELCEXTwo} of the algorithm. For every such combination $(r_{j_1}^1,\dotsc,r_{j_\ell}^\ell)$
  of rules the algorithm then calls the subroutine
  \Call{findELCEXR}{$\dle$, $e$,
    $(r_{j_1}^1,\dotsc,r_{j_\ell}^\ell)$}, which is illustrated in
  \Cref{alg:findELCEXForRules}, and computes the smallest set $A$ of
  at most $k$ features such that $e_A$ is classified by the rule
  $r_{j_o}^o$ of $\dl_o$ for every $o \in [\ell]$. This subroutine
  and the proof of its correctness work very similar to the subroutine \textbf{findLCEXForRule}($L$,
  $e$, $r_j$) of \Cref{alg:findLCEXForRule} used in
  \Cref{lem:dl-mlcex-branch} for a single \DL{} and is therefore not
  repeated here. In essence the
  subroutine branches over all possibilities for such a set $A$.

  We are now ready to analyze the runtime of the algorithm, i.e.,
  \Cref{alg:findELCEX}. First note that the function
  \Call{findELCEX}{$\dle$, $e$} makes at most $m^s$ calls to the
  subroutine \Call{findELCEXR}{$\dle$, $e$,
    $(r_{j_1}^1,\dotsc,r_{j_\ell}^\ell)$} and apart from those calls
  all operations take time at most $\bigoh(\som{\dle})$. Moreover, the
  runtime of the subroutine \Call{findELCEXR}{$\dle$, $e$,
    $(r_{j_1}^1,\dotsc,r_{j_\ell}^\ell)$} is the same as the runtime
  of the subroutine \Call{findLCEXForRule}{$L$, $e$, $r_j$} given in
  \Cref{alg:findLCEXForRule}, i.e.,
  $\bigoh(a^k\som{\dle}^2)$. Therefore, we obtain
  $\bigoh(m^sa^k\som{\dle}^2)$ as the total run-time of the algorithm.
\end{proof}\fi

\begin{theorem}\label{th:DS-MLCEX-FPT-const-ens}
  Let $\MM \in \{\DS,\DL\}$. \MM-\MLCEX{}$(\termselem+\xpsize)$ is
  \FPT{}, when \enssize{} is constant.
\end{theorem}
\begin{proof}
  The theorem follows immediately from \Cref{lem:dle-mlcex-branch}.
\end{proof}

\subsection{OBDDs and their Ensembles}\label{ssec:algOBDD}

\newcommand{\width}{\textsf{width}}

In this subsection, we will present our algorithmic results for
\OBDD{}s and their ensembles \OBDDEO{} and \OBDDE{}. Interestingly,
while seemingly more powerful \OBDD{}s and \OBDDEO{}s behave very
similar to \DT{}s and \RF{}s if one replaces \mnlsize with
\widthelem. On the other hand, allowing different orderings for every
ensemble \OBDD{} makes \OBDDE{}s much more powerful and 
harder to explain (see \Cref{ssec:hard-obdd} for an explanation of
this phenomenon).

\iflong
We start by providing reductions of \OBDD{}s and \OBDDEO{}s to Boolean
circuits, which will allow us to employ \Cref{the:solve-circ}.
The following lemma follows from \cite[Lemma
4.1]{DBLP:conf/icdt/JhaS12} since the rank-width is upper bounded by
the path-width.
\begin{lemma}[{\cite[Lemma 4.1]{DBLP:conf/icdt/JhaS12}}]\label{lem:obdd-trans-circ}
  There is a polynomial-time algorithm that given a $\OBDD$
  $\obdd$ and a
  class $c$ produces a circuit $\CIRC(\obdd,c)$ such that:
  \begin{itemize}
  \item for every example $e$, it holds that $\obdd(e)=c$ if and only if $e$ satisfies $\CIRC(\obdd,c)$
  \item $\rw(\CIRC(\obdd,c)) \leq 5\width(\obdd)$
  \end{itemize}
\end{lemma}

\ifshort \begin{lemma}[$\star$]\fi\iflong\begin{lemma}\fi\label{lem:obddeo-trans-circ}
  There is a polynomial-time algorithm that given an $\OBDDEO$
  $\obddeo$ and a
  class $c$ produces a circuit $\CIRC(\obddeo,c)$ such that:
  \begin{enumerate}[(1)]
  \item for every example $e$, it holds that $\obddeo(e)=c$ if and only if $e$ satisfies $\CIRC(\obddeo,c)$
  \item $\rw(\CIRC(\obddeo,c)) \leq 3 \cdot 2^{|\obddeo|5\max_{\obdd \in \obddeo}\width(\obdd)}$
  \end{enumerate}
\end{lemma}
\iflong\begin{proof}
  We obtain the circuit $\CIRC(\obddeo,c)$ from the (not
  necessarily disjoint) union of the circuits $\CIRC(\obdd,c)$, which are
  provided in \Cref{lem:obdd-trans-circ},
  for every $\obdd \in \obddeo$ after adding a new MAJ-gate $r$
  with threshold $\lfloor|\obddeo|/2\rfloor+1$,
  which also serves as the output gate of $\CIRC(\obddeo,c)$, that
  has one incoming arc from the output gate of $\CIRC(\obdd,c)$
  for every $\obdd \in \obddeo$. Clearly,
  $\CIRC(\obddeo,c)$ satisfies (1).
  Moreover, to see that it also satisfies
  (2), we first need to provide the construction for $\CIRC(\obdd,c)$
  given in \cite[Lemma 4.1]{DBLP:conf/icdt/JhaS12}.
  
  That is, $\CIRC(\obdd,c)$ for a given complete \OBDD{}
  $\obdd=(D,\fBDD)$ and a given
  class $c$ is defined as follows. The main idea behind the
  construction is to identify every vertex $v$ of $D$ with a
  propositional formula $F_v$ as follows. First, we set $F_{t_b}$ to
  be true if and only if
  $b=c$ for the two sink vertices $t_0$ and $t_1$ of $D$. Moreover,
  if $v$ is an inner vertex of $D$ with $f=\fBDD(v)$,
  $0$-neighbor $n_0$, and $1$-neighbor $n_1$, the formula $F_v$ is
  defined as $(f=0 \land F_{n_0}) \lor (f=1 \land F_{n_1})$. Then, it
  holds that $M(e)=c$ if and only if (the assignment corresponding to)
  $e$ satisfies the formula $F_s$ for the root $s$ of $D$. The circuit
  $C=\CIRC(\obdd,c)$ that corresponds to $F_s$
  can now be obtained from $D \setminus \{t_0,t_1\}$ by applying the
  following modifications:
  \begin{itemize}
  \item We add one input gate $g_f$ for every feature $f \in
    \feat(\obdde)$.
  \item We add one NOT-gate $\overline{g_f}$ for every feature $f \in
    \feat(\obdd)$, whose only incoming arc is from $g_f$.
  \item We replace every inner vertex $v$ of $D$ with an OR-gate
    $o_v$.
  \item We replace every arc $e=(u,v)$ with $f=\fBDD(u)$ of $D\setminus \{t_0,t_1\}$
    with an AND-gate $a_e$ that has one outgoing arc to $o_u$ and
    that has one incoming arc from $o_v$ and
    one incoming arc from $g_f$ ($\overline{g_f}$)
    if $v$ is a $1$-neighbor ($0$-neighbor) of $u$ in $D$.
  \item We replace every arc $e=(u,t_c)$ with $f=\fBDD(u)$ of $D$
    with an arc from $g_f$ ($\overline{g_f}$) to $o_u$
    if $t_c$ is a $1$-neighbor ($0$-neighbor) of $u$ in $D$.    
  \item We remove every arc $e=(u,t_{c'})$ with $c'\neq c$ from $D$.
  \item We let $o_s$ be the root of $C$ for the root $s$ of $\obdd$.
  \end{itemize}
  We refer to \cite[Lemma 4.1]{DBLP:conf/icdt/JhaS12} for the
  correctness of the construction. Note also that \cite[Lemma
  4.1]{DBLP:conf/icdt/JhaS12} provides a path decomposition
  $\PPP_\obdd=(P_\obdd,\lambda_\obdd)$ of width at most $5\width(\obdd)$ of $C$ as
  follows. For a feature $f \in \feat(\obdd)$, let $L(f)$ be the set of
  all vertices $v$ of $D$ with $\fBDD(v)=f$. Then,
  $P_\obdd$ has one vertex $p_f$ for every but the last feature $f \in
  \feat(\obdd)$ w.r.t. $<_\obdd$. Moreover, let $f$ be a feature in $\feat(\obdd)$,
  whose successor w.r.t. $<_\obdd$ is the feature $f'$, then
  the bag $\lambda_\obdd(p_f)$ is equal to $\{ g_f,\overline{g_f}\} \cup \SB
  o_v \SM v \in L(f)\cup L(f') \SE \cup \SB a_e \SM e=(u,v) \in E(D)
  \land u \in L(f) \land v \in L(f')\SE$.
  This completes the description of $\CIRC(\obdd,c)$ for any $\obdd
  \in\obddeo$. It remains to show that $\CIRC(\obddeo,c)$ satisfies
  (2), in particular since the rankwidth is upper bounded by the
  pathwidth (using \Cref{lem:ranktree})
  it remains to show that the pathwidth of $\CIRC(\obddeo,c)$ is at
  most $|\obddeo|5\max_{\obdd \in \obddeo}\width(\obdd)$. To see this
  consider the following path decomposition $\PPP=(P,\lambda)$ of
  $\CIRC(\obddeo,c)$. $P$ has one vertex $p_f$ for every but the last feature $f \in
  \feat(\obdd)$ w.r.t. $<_\obdd$. Moreover,
  $\lambda(p_f)=\{r\} \cup (\bigcup_{\obdd \in
    \obddeo}\lambda_\obdd(p_f^\obdd))$, where $p_f^\obdd$ is the
  vertex $p_f$ of the path $P_\obdd$ of the path decomposition
  $\PPP_\obdd=(P_\obdd,\lambda_\obdd)$. 
\end{proof}\fi
The following corollary follows immediately from
\Cref{lem:obddeo-trans-circ} and \Cref{the:solve-circ}.
\fi
\ifshort
The following two corollaries follow from our meta-theorem \Cref{the:solve-circ}
using suitable translations of \OBDD{}s, \OBDDE{}s, and \OBDDEO{}s
into Boolean circuits.
While it is sufficient to use \widthelem{} as a
parameter for \OBDDEO{}s; this is no longer the case for \OBDDE{}s,
where one needs to bound the size (instead of the width) of every
element in the ensemble.
\fi
\begin{corollary}\label{th:OBDDE-FPT-es}
  Let  $\PP \in \{\LAEX, \LCEX, \GAEX, \GCEX\}$.
  $\OBDDEO$-\MPP{}$(\enssize{} + \widthelem{})$ is \FPT.
\end{corollary}
\iflong
The next Corollary follows again from \Cref{lem:obddeo-trans-circ} and
\Cref{the:solve-circ} after observing that the pathwidth of the Boolean circuit
constructed in \Cref{lem:obddeo-trans-circ} for an \OBDDE{} is bounded
by $\bigoh(\enssize{}\cdot \sizeelem{})$ since this bound is also an
upperbound on the number of gates in the constructed Boolean
circuit. Note that in contrast to \OBDDEO{}s, where the pathwidth of
the resulting circuit can be bounded purely in terms of $\enssize{}$
and $\widthelem{}$, this is no longer the case for \OBDDE{}s.
\fi
\begin{corollary}\label{th:OBDDE-FPT-eh}
  Let  $\PP \in \{\LAEX, \LCEX, \GAEX, \GCEX\}$.
  $\OBDDE$-\MPP{}$(\enssize{} + \sizeelem{})$ is \FPT.
\end{corollary}

\iflong
\begin{theorem}[{\cite[Lemma 14]{BarceloM0S20}}]\label{th:OBDD-MLCEX}
  There is a polynomial-time algorithm that given a \OBDD{} $\obdd$ and
  an example $e$ outputs a (cardinality-wise) minimum local
  contrastive explanation for $e$ w.r.t. $\obdd$ or no if such an
  explanation does not exist. Therefore, 
  \OBDD{}-\MLCEX{} can be solved in polynomial-time. 
\end{theorem}
\fi
\iflong
The following auxiliary lemma is an analogue of \Cref{lem:dt-sm-testsol} for \OBDD{}s and provides polynomial-time algorithms for
testing whether a given subset of features $A$/partial
example $e'$ is a local abductive/global abductive/global
contrastive explanation for a given example $e$/class $c$ w.r.t. a
given \OBDD{} $\obdd$.
\begin{lemma}\label{lem:obdd-sm-testsol}
  Let $\obdd$ be an \OBDD{}, let $e$ be an example and let $c$ be a class.
  There are polynomial-time algorithms for the following problems:
  \begin{enumerate}[(1)]
  \item Decide whether a given subset $A \subseteq \feat(\obdd)$ of
    features is a local abductive explanation for $e$ w.r.t. $\obdd$.
  \item Decide whether a given partial example $e'$ is a global
    abductive explanation for $c$ w.r.t. $\obdd$.
  \item Decide whether a given partial example $e'$ is a global
    contrastive explanation for $c$ w.r.t. $\obdd$.
  \end{enumerate}
\end{lemma}
\begin{proof}
  Let $\obdd=(D,\fBDD)$ be an \OBDD{}, let $e$ be an example and let
  $c$ be a class.
  For a partial example $e'$, we denote by $D_{e'}$ the directed
  acylic graph obtained from $D$ after removing all arcs from a vertex
  $v$ whose feature $f=\fBDD(v)$ is assigned by $e'$ to its
  $1-e'(f)$-neighbor. For a subset $A \subseteq \feat(\obdd)$ of
  features we denote by $e_{|A}$ the partial example equal to the
  restriction of $e$ to $A$. 
  
  Towards showing (1), let $A \subseteq \feat(\obdd)$ be a subset of
  features. Then, $A$ is a
  local abductive explanation for $e$ w.r.t. $\obdd$ if and only if 
  $t_{\obdd(e)}$ is the only sink vertex of $D_{e_{|A}}$ that is
  reachable from $s$, which can clearly be checked in polynomial-time.

  Similarly, towards showing (2), observe that the partial example (assignment) $\tau : F \rightarrow
  \{0,1\}$ is a global abductive explanation for $c$ w.r.t. $\obdd$ if
  and only if $t_c$ is the only sink reachable from $s$ in $D_{\tau}$,
  which can clearly be checked in polynomial-time.

  Finally, towards showing (3), note that a partial example $\tau : F \rightarrow
  \{0,1\}$ is a global contrastive explanation for $c$ w.r.t. $\obdd$
  if and only if $t_c$ is not reachable from $s$ in $D_{\tau}$, which
  can clearly be verified in polynomial-time.
\end{proof}
\fi

\ifshort
The proof of the following theorem is very similar to the corresponding
result  for \DT{}s (\Cref{th:poly-DT-SM}).
\fi
\iflong
  Using dedicated algorithms for the inclusion-wise minimal variants of
  \LAEX{}, \GAEX{}, and \GCEX{} together with \Cref{th:OBDD-MLCEX}, we
  obtain the following result.
\fi
\ifshort \begin{theorem}[$\star$]\fi\iflong\begin{theorem}\fi\label{th:OBDD-SLGA-P}
  Let $\PP \in \{\LAEX, \LCEX, \GAEX, \GCEX\}$. 
  \OBDD-\SPP{} \ifshort and \OBDD-\MLCEX{} \fi can be solved in polynomial-time.
\end{theorem}
\iflong\begin{proof}
  Note that the statement of the theorem for \OBDD{}-\SMLCEX{} follows
  immediately from \Cref{th:OBDD-MLCEX}. Therefore, it suffices to show
  the statement of the theorem for the remaining 3 problems.

  Let $(\obdd,e)$ be an instance of \OBDD{}-\SMLAEX{}.
  We start by setting $A=\feat(\obdd)$. Using \Cref{lem:obdd-sm-testsol}, we
  then test for any feature $f$ in $A$, whether $A\setminus \{f\}$
  is still a local abductive explanation for $e$ w.r.t. $\obdd$ in polynomial-time. If so,
  we repeat the process after setting $A$ to $A\setminus\{f\}$ and
  otherwise we do the same test for the next feature $f \in
  A$. Finally, if $A\setminus\{f\}$
  is not a local abductive explanation for every $f \in A$, then $A$ is an inclusion-wise
  minimal local abductive explanation and we can output $A$.

  The polynomial-time algorithm for a given instance $(\obdd,c)$ with
  $\obdd=(D,\fBDD)$ of
  \OBDD{}-\SMGAEX{} now works as follows. Let $P$ be any shortest path
  from $s$ to $t_c$ in $D$; if no such path exists we can correctly
  output that there is no global abductive explanation for $c$
  w.r.t. $\obdd$. Then, the assignment $\alpha$ defined by $P$
  is a global abductive explanation for $c$ w.r.t. $\obdd$. To obtain
  an inclusion-wise minimal solution, we do the following. Let $F=\feat(\alpha)$ be
  the set of features on which $\alpha$ is defined. We now test for
  every feature $f \in F$ whether the restriction
  $\alpha[F\setminus \{f\}]$ of $\alpha$ to $F\setminus
  \{f\}$ is a global abductive explanation for $c$ w.r.t. $\TTT$. This
  can clearly be achieved in polynomial-time with the help of
  \Cref{lem:obdd-sm-testsol}. If this is true for any feature $f \in F$,
  then we repeat the process for $\alpha[F\setminus \{f\}]$,
  otherwise we output $\alpha$. A very similar algorithm now
  works for the \DT{}-\SMGCEX{} problem, i.e., instead of starting
  with a shortest path from $s$ to $t_c$, we start with a shortest
  path from $s$ to $t_{c'}$ for some $c'\neq c$.
\end{proof}\fi

\iflong
\begin{lemma}\label{lem:OBDDEO-to-OBDDE}
  Let $\obddeo$ be an \OBDDEO{} with $\ell=|\obddeo|$ and $m$ is the
  maximum size of any \OBDD{} in $\obddeo$. There is an algorithm
  that in time $\bigoh(m^\ell)$ computes an \OBDDO{} $\obdd$ of size
  at most $m^\ell$ such that $\obdd(e)=\obddeo(e)$ for every example $e$.
\end{lemma}
\begin{proof}
  We construct the \OBDD{} $\obdd=(D,\fBDD)$ as follows. Let $D'$ be
  the directed acyclic
  graph obtained as follows. $D'$ has one vertex $n_{p}$ for every tuple $p=(v_1,\dotsc,v_\ell) \in V(D_1) \times \dotsb
  \times V(D_\ell)$. For a tuple $p=(v_1,\dotsc,v_\ell) \in V(D_1)
  \times \dotsb \times V(D_\ell)$, we denote by $\lambda(p)$ the
  smallest feature in $\SB \fBDD_i(v_i) \SM 1 \leq i \leq \ell\SE$
  w.r.t. the ordering $<_\obdde$ if $\SB \fBDD_i(v_i)
  \SM 1 \leq i \leq \ell\SE\neq \emptyset$. Moreover, otherwise, i.e., if $\SB \fBDD_i(v_i)
  \SM 1 \leq i \leq \ell\SE=\emptyset$, then every vertex $v_i$ is a
  sink vertex in $D_i$, i.e., it either corresponds to $t_0$ or to
  $t_1$, in which case we let $\lambda(p)=1$ if the majority of the
  vertices $v_1,\dotsc,v_\ell$ correspond to $t_1$ in $D_i$ and
  otherwise we set $\lambda(p)=0$.
  Let $p$ be such that
  $\lambda(p)\neq \emptyset$ and let $b \in \{0,1\}$. The $b$-successor of $p$,
  denoted by $S_b(p)$, is the tuple $(v_1',\dotsc, v_k')$, where
  $v_i'=v_i$ if $\fBDD_i(v_i)\neq \lambda(p)$ and $v_i'$ is the
  $b$-neighbor of $v_i$ in $D_i$, otherwise. Now, every vertex $n_p$
  in $D'$ with $\lambda(p)\neq \emptyset$ has $n_{S_0(p)}$ as its
  $0$-neighbor and $n_{S_1(p)}$ as its $1$-neighbor and this completes the
  description of $D'$. Note that $D'$ is acylic because it only
  contains arcs from a vertex $n_p$ with $\lambda(p)\notin \{0,1\}$
  to a vertex $n_{(S_b(p)}$ and either $\lambda(n_p)<_\obdde
  \lambda(S_b(p))$ or $\lambda(S_b(p)) \in \{0,1\}$.

  Then, $D$ is the
  directed acyclic graph with source vertex $n_{(s_0,\dotsc,s_\ell)}$
  and sink vertices $t_0$ and $t_1$ that is obtained from $D'$ after applying the
  following operations:
  \begin{itemize}
  \item Remove all vertices from $D'$ that are not reachable from the
    vertex $n_{(s_0,\dotsc,s_\ell)}$, where $s_i$ is the root of
    $D_i$.
  \item Identify the vertices in $\SB n_{p} \SM \lambda(p)=b \SE$ with
    the new vertex $t_b$ for every $b \in \{0,1\}$.
  \end{itemize}
  Finally, we set $\fBDD(n_p)=\lambda(p)$ for every $p$ with
  $\lambda(p)\notin \{0,1\}$.
  This completes the construction of \obdd, which can clearly be
  achieved in time $\bigoh(m^\ell)$ and has size at most $m^\ell$.
  It
  is now straightforward to verify that $\obddeo(e)=\obdd(e)$ for
  every example, as required.
\end{proof}
\fi

The next theorem uses our result that the considered problems are in
polynomial-time for \OBDD{}s \ifshort(\Cref{th:OBDD-SLGA-P}) \fi
together with an \XP{}-algorithm that
transforms any \OBDDEO{} into an equivalent \OBDD{}.

\ifshort \begin{theorem}[$\star$]\fi\iflong\begin{theorem}\fi\label{th:OBDDEO-XP-ens}
    Let  $\PP \in \{\LAEX, \LCEX, \GAEX, \GCEX\}$. 
    \OBDDEO-\SPP$(\enssize{})$ and
    \OBDDEO-\MLCEX$(\enssize{})$ are in 
    \XP{}.
\end{theorem}
\iflong\begin{proof}
    Let $\obdde=\{\obdd^1,\dotsc,\obdd^\ell\}$ with variable ordering
    $<_\obdde$ and $\obdd_i=(D_i,\fBDD_i)$ be the \OBDDEO{} given as
    an input to any of the five problems stated above.

    We use \Cref{lem:OBDDEO-to-OBDDE} to construct an \OBDD{} $\obdd$
    that is not too large (of size at most
    $m^{\ell}$, where $m=(\max_{i=1}^\ell|V(D_i)|)$) and that is equivalent to
    $\obdde$ in the sense that $\obdde(e)=\obdd(e)$ for every
    example in time at most $\bigoh(m^{\ell})$. Since all five
    problems can be solved in polynomial-time on $\obdd$ (because of
    \Cref{th:OBDD-MLCEX} and \Cref{th:OBDD-SLGA-P}) this completes the
    proof of the theorem.
\end{proof}\fi

\section{Hardness Results}

\newcommand{\HOM}{\textsc{HOM}}
\newcommand{\kHOM}{\textsc{p-HOM}}

In this section, we provide our algorithmic lower
bounds. We start by showing a close connection between the complexity
of all of our explanation problems to the following two problems. As
we will see the hardness of finding explanations comes from the
hardness of deciding whether or not a given model classifies all
examples in the same manner. More specifically, from the \HOM{}
problem defined below, which asks whether a given model has an example that is classified
differently from the all-zero example, i.e., the example being $0$ on
every feature. We also need the \kHOM{} problem, which is a parameterized
version of \HOM{} that we use to show parameterized hardness results
for deciding the existence of local contrastive explanations.

In the following, let $\MM$ be a model type.
\pbDef{$\MM{}$-\textsc{Homogeneous (\HOM{}})}{A model $M \in \MM{}$.}{
Is there an example $e$ such that $M(e) \neq M(e_0)$, where $e_0$ is the
all-zero example?}

\pbDef{$\MM$-p-\textsc{Homogeneous (\kHOM{}})}{A model $M \in \MM{}$ and
  integer $k$.}{
Is there an example $e$ that sets at most $k$ features to $1$ such that $M(e) \neq M(e_0)$, where $e_0$ is the all-zero example?}

The following lemma now shows the connection between \HOM{} and the
considered explanation problems.

\begin{lemma}\label{lem:HOM-reduction}
  Let $M \in \MM$ be a model, $e_0$ be the all-zero example, and let $c=M(e_0)$.
  The following problems are equivalent:  
  \begin{enumerate}[(1)]
  \item $M$ is a no-instance of \MM{}-\HOM{}.
  \item The empty set is a solution for the instance $(M,e_0)$ of
    \MM{}-\SMLAEX{}.
  \item $(M,e_0)$ is a no-instance of \MM{}-\SMLCEX{}.    
  \item The empty set is a solution for the instance $(M,c)$ of \MM{}-\SMGAEX{}.
  \item The empty set is a solution for the instance $(M,1-c)$ of
    \MM{}-\SMGCEX{}.
  \item $(M,e_0,0)$ is a yes-instance of \MM{}-\MLAEX{}.
  \item $(M,e_0)$ is a no-instance of \MM{}-\MLCEX{}.    
  \item $(M,c,0)$ is a yes-instance of \MM{}-\MGAEX{}.
  \item $(M,1-c,0)$ is a yes-instance of \MM{}-\MGCEX{}.
  \end{enumerate}
\end{lemma}
\begin{proof}
  It is easy to verify that all of the statements (1)--(9) are
  equivalent to the following statement (and therefore equivalent to
  each other): $M(e)=M(e_0)=c$ for every example $e$. 
\end{proof}
While \Cref{lem:HOM-reduction} is sufficient for most of our hardness
results, we also need the following lemma to show certain
parameterized hardness results for deciding the existence of local
contrastive explanations.
\ifshort \begin{lemma}[$\star$]\fi
\iflong\begin{lemma}\fi\label{lem:kHOM-reduction}
    Let $M \in \MM$ be a model and let $e_0$ be the all-zero example.
    The following problems are equivalent:  
    \begin{enumerate}[(1)]
    \item $(M,k)$ is a yes-instance of \MM{}-\kHOM{}.
    \item $(M,e_0,k)$ is a yes-instance of \MM{}-\MLCEX{}.
  \end{enumerate}
\end{lemma}
\iflong\begin{proof}
    The lemma follows because both statements are equivalent to
    the following statement:
    There is an example $e$ that sets at most $k$ features to $1$ such
    that $M(e)\neq M(e_0)$. 
\end{proof}
\fi

We will often reduce from the following problem, which is well-known
to be \NPh{} and also \Wh{1} parameterized by $k$.
\pbDef{{\sc Multicolored Clique} (\MCC)}
{A graph $G$ with a proper $k$-coloring of $V(G)$.}
{Is there a clique of size $k$ in $G$?}

The following lemma provides a unified way to show hardness
results for ensembles for practically all of our model types in the
case that we allow arbitrary many (constant-size) ensemble elements,
i.e., we use it to show \Cref{th:RF-paraNP,th:DSE-paraNP,th:OBDDEO-paraNP-hws}.

\ifshort \begin{lemma}[$\star$]\fi
\iflong\begin{lemma}\fi\label{lem:MME-kHOM-W1}
  Let \MM{} be a class of models such that there are 
  models $M^0\in \MM$, 
  $M^1_{f} \in \MM$ and $M^2_{f_1,f_2}\in \MM$ for features
  $f$, $f_1$, and $f_2$ of size at most $d$ such that:
  \begin{itemize}
  \item $M^0$ classifies every example negatively.
  \item $M^1_{f}$ classifies an example $e$ positively iff $e(f)=1$.
  \item $M^2_{f_1,f_2}$  classifies an example $e$ positively iff $e(f_1)=0$ or $e(f_2)=0$.
  \end{itemize}
  \MME-\kHOM{} is \Wh{1} parameterized by $k$ even if the size of
  each ensemble element is at most $d$ and
  \MME-\HOM{} is \NPh{} even if the size of
  each ensemble element is at most $d$.
\end{lemma}
\iflong\begin{proof}\fi
\ifshort\begin{proof}[Proof Sketch]\fi  
  We provide a parameterized reduction from the \textsc{Multicolored
    Clique} (\MCC) problem, which is also a polynomial-time reduction. Given an instance $(G,k)$
  of the \MCC{} problem with $k$-partition $(V_1,\dotsc,V_k)$ of
  $V(G)$, we will construct an equivalent instance
  $(\MMM,k)$ of \MME-\kHOM{} in polynomial-time as follows.
  $\MMM$ uses one binary feature $f_v$ for every $v \in V(G)$. Let $<_V$
  be an arbitrary ordering of $V(G)$. We denote by $n$ and $m$ the
  number of vertices and edges of the graph $G$, respectively.

  $\MMM$ contains the following ensemble elements:
  \begin{itemize}
  \item For every non-edge $uv \notin E(G)$ with $u<_V v$, we add
    the model $M^2_{f_u,f_v}$ to $\MMM$.
  \item For every vertex $v \in V(G)$, we add the model $M^1_{f_v}$ to
    $\MMM$.
  \item We add $(\binom{n}{2}-m)-n+2k-1$ models $M^0$ to $\MMM$.
  \end{itemize}
  Clearly, the reduction works in polynomial-time and preserves the
  parameter and it only remains to show that $G$ has a $k$-clique if
  and only if there is an example $e$ such that $\MMM(e)\neq
  \MMM(e_0)$ that sets at most $k$ features to $1$. 
  \iflong
    Note first that
  $\MMM(e_0)=0$ because $M(e_0)=1$ holds only for
  $(\binom{n}{2}-m)$ out of the $2(\binom{n}{2}-m)+2k-1$
  models in $\MMM$.
    
  Towards showing the forward direction, let $C=\{v_1,\dotsc,v_k\}$ be
  a $k$-clique of $G$ with $v_i \in V_i$ for every $i \in [k]$. We
  claim that the example $e$ that sets all features in $\SB f_v \SM v
  \in C \SE$ to $1$ and all other features to $0$ satisfies
  $\MMM(e)\neq \MMM(e_0)$. Because $C$ is a clique in $G$, we obtain
  that $M(e)=1$ for all of the $(\binom{n}{2}-m)$ copies $M$ of
  $M^2_{f_u,f_v}$ in $\MMM$. Moreover, because $e$ sets exactly $k$
  features to $1$, it holds that $M(e)=1$ for exactly $k$ copies $M$ of
  $M^1_{f_v}$ in $\MMM$. Therefore, $e$ is classified positively by
  exactly $(\binom{n}{2}-m)+k$ models in $\MMM$ and $e$ is
  classified negatively by exactly
  $n-k+(\binom{n}{2}-m)-n+2k-1=(\binom{n}{2}-m)+k-1$ models
  in $\MMM$, which shows that $\MMM(e)=1\neq \MMM(e_0)$.
  
  Towards showing the reverse direction, let $e$ be an example that
  sets at most $k$ features to $1$ such that $\MMM(e)=1\neq
  \MMM(e_0)$. First note that because $M(e)=0$ for every of the
  $(\binom{n}{2}-m)-n+2k-1$ copies $M$ of $M^0$ in $\MMM$, it has
  to hold that $M(e)=1$ for at least $(\binom{n}{2}-m)+k$ of the
  copies $M$ of either $M^1_{f_v}$ or $M^2_{f_u,f_v}$ in $\MMM$. Since
  $e$ sets only at most $k$ features to $1$, we obtain that $M(e)=1$
  for at most $k$ of the copies $M$ of $M^1_{f_v}$ in
  $\MMM$. Therefore, $M(e)=1$ for all copies of $M$ of $M^2_{f_u,f_v}$
  in $\MMM$ and moreover $M(e)=1$ for exactly $k$ copies $M$ of
  $M^1_{f_v}$ in $\MMM$. But then the set $\SB v \SM e(f_v)=1\SE$ must
  be a $k$-clique of $G$.\fi
\end{proof}

\subsection{DTs and their Ensembles}

Here, we provide our algorithmic lower bounds for \DT{}s and their
ensembles.
We say that a \DT{} $\TTT$ is \emph{ordered} if there is an ordering
$<$ of the features
in $\feat(\TTT)$ such that the ordering of the features on every root-to-leaf path of
$\TTT$ agrees with $<$. We need the following auxiliary lemma to
simplify the descriptions of our reductions.
\begin{lemma}\label{lem:DT-construction}
  Let $E \subseteq E(F)$ be a set of examples defined on features in $F$.
  An ordered \DT{} $\TTT_E$ of size at
  most $2|E||F| + 1$ such that $\TTT_E(e)=1$ if and only if $e \in E$ can
  be constructed in time $\bigoh(|E||F|)$. 
\end{lemma}
\begin{proof}
  Let $\mathnormal{<}=(f_1,\dotsc,f_n)$ be an arbitrary order of the features in $F$.
  First, we construct a simple ordered \DT{} $\TTT_e=(T_e,\lambda_e)$ that
  classifies only example $e$ as $1$ and all other examples as $0$.
  $\TTT_e$ has one inner node $t_i^e$ for every $i \in [n]$ with
  $\lambda_e(t_i^e)=f_i$. Moreover, for $i<n$, $t_i^e$ has $t_{i+1}^e$
  as its $e(f_i)$-child and a new $0$-leaf as its other
  child. Finally, $t_n^e$ has a new $1$-leaf as its $e(f_n)$-child and
  a $0$-leaf as its other child. Clearly, $\TTT_e$ can be constructed
  in time $\bigoh(|F|)$.
  
  We now construct $\TTT_E$ iteratively starting from $\TTT_\emptyset$
  and adding one example from $E$ at a time (in an arbitrary order).
  We set $\TTT_\emptyset$ to be the \DT{} that only consists of a
  $0$-leaf. Now to obtain $\TTT_{E'\cup \{e\}}$ from $\TTT_{E'}$
  for some $E'\subseteq E$ and $e \in E \setminus E'$, we do the following. Let
  $l$ be the $0$-leaf of $\TTT_{E'}$ that classifies
  $e$ and let $f_i$ be the feature assigned to the parent of
  $l$. Moreover, let $\TTT_e'$ be the sub-\DT{} of $\TTT_e$ rooted at
  $t^e_{i+1}$ or if $i=n$ let $\TTT_e'$ be the \DT{} consisting only of a $1$-leaf.
  Then, $\TTT_{E'\cup \{e\}}$ is obtained from the
  disjoint union of $\TTT_{E'}$ and $\TTT_e'$ after
  identifying the root of $\TTT_e'$ with $l$. Clearly, $\TTT_E$ is an
  ordered \DT{} that can be constructed in time $\bigoh(|E||F|)$ has
  size at most $2|E||F|+1$ and satisfies $\TTT_E(e)=1$ if
  and only if $e \in E$.
\end{proof}

We note that the following theorem also follows from a result
in~\cite[Proposition 5]{BarceloM0S20} for FBDDs, i.e., \BDD{}s without
contradicting paths. However, we require a
different version of the proof that generalizes easily to \OBDD{}s,
i.e., we need to show hardness for ordered \DT{}s.

\ifshort \begin{theorem}[$\star$]\fi\iflong\begin{theorem}\fi\label{th:DT-MLAEX-W2}
    \DT{}-\MLAEX{} is \NPh{} and \DT{}-\MLAEX{}$(\xpsize{})$ is \Wh{2}
    even if for ordered \DT{}s.
\end{theorem}
\iflong\begin{proof}
  We provide a parameterized reduction from the \textsc{Hitting Set}
  problem, which is well-known to be \NPh{} and \Wh{2}
  parameterized by the size of the solution. That is, given a family
  of sets $\FFF$ over a universe $U$ and an integer $k$, the
  \textsc{Hitting Set} problem is to decide whether $\FFF$ has a
  \emph{hitting set} of size at most $k$, i.e., a subset $S \subseteq
  U$ with $|U|\leq k$ such that $S\cap F\neq \emptyset$ for every $F
  \in \FFF$. Given an instance $(U,\FFF,k)$ of the \textsc{Hitting
    Set} problem, we will construct an equivalent instance of $(\TTT,e,k)$
  \MLAEX{} in polynomial-time as follows.
  
  Let $B$ be the set of features containing one binary feature $f_u$ for every $u \in
  U$ and let $E$ be the set of examples equal to $\SB e_F \SM F \in
  \FFF\SE$, where $e_F$ is the example that is $1$ at every feature
  $f_u$ such that $u \in F$ and otherwise $0$. Using \Cref{lem:DT-construction} we can
  construct an ordered \DT{} $\TTT_E$ of size at most $2|E||B|$ such
  that $\TTT_E(e)=1$ if and only if $e \in E$. Let $\TTT=\TTT_E$ and
  let $e$ be the all-zero example. Clearly, $(\TTT,e,k)$ can be
  constructed from $(U,\FFF,k)$ in polynomial-time and it only remains
  to show the equivalence of the two instances.

  Towards showing the forward direction, let $S$ be a hitting set for
  $\FFF$ of size at most $k$. We claim that $A=\SB f_u \SM u \in S\SE$
  is a local abductive explanation for $e$ w.r.t. $\TTT$, which
  concludes the proof of the forward direction. To see that this is
  indeed the case consider any example $e'$ that agrees with $e$ on
  $A$, i.e., $e'$ is $0$ at every feature in $A$. Then, because $S$ is
  a hitting set for $\FFF$, we obtain that $e'$ is not in $E$, which implies that
  $\TTT(e')=0=\TTT(e)$, as required.

  Towards showing the reverse direction, let $A$ be a local abductive
  explanation of size at most $k$ for $e$ w.r.t. $\TTT$. We claim that
  $S=\SB u \SM f_u \in A\SE$ is a hitting set for $\FFF$. Suppose that
  this is not the case and there is a set $F \in \FFF$ with $F\cap
  S=\emptyset$. Then, the example $e_F$ agrees with $e$ on every
  feature in $A$, however, it holds that
  $\TTT(e_F)=1\neq \TTT(e)$, contradicting our assumption that $A$ is
  a local abductive explanation for $e$ w.r.t. $\TTT$.
\end{proof}\fi

The following theorem is an analogue of \Cref{th:DT-MLAEX-W2} for
global abductive and global contrastive explanations. It is
interesting to note that while it was not
necessary to distinguish between local abductive explanations on one
side and global abductive and global contrastive explanations on the
other side in the setting of algorithms, this is no longer the case
when it comes to algorithmic lower bounds. Moreover, while the following
result establishes \W{1}-hardness for \DT{}-\MGAEX{}$(\xpsize{})$ and
\DT{}-\MGCEX{}$(\xpsize{})$, this is achieved via fpt-reductions that
are not polynomial-time reductions, which is a behavior that is very
rarely seen from natural parameterized problems. While it is therefore
not clear whether the problems are \NPh{}, the result still shows that
the problems are not solvable in polynomial-time unless
$\FPT=\W{1}$, which is considered unlikely \cite{DowneyFellows13}.
\ifshort \begin{theorem}[$\star$]\fi\iflong\begin{theorem}\fi\label{th:DT-MGAEX-W1}
  \DT{}-\MGAEX{}$(\xpsize{})$ and \DT{}-\MGCEX{}$(\xpsize{})$ are \Wh{1}. 
  Moreover, there is no polynomial time algorithm for solving \DT{}-\MGAEX{} and \DT{}-\MGCEX{},
  unless $\FPT = \W{1}$.
\end{theorem}
\iflong\begin{proof}\fi
\ifshort\begin{proof}[Proof Sketch]\fi 
  We provide a parameterized reduction from the \textsc{Multicolored
    Clique} (\MCC)
  problem, which is well-known to be \Wh{1} parameterized by the size
  of the solution. Given an instance $(G,k)$
  of the \MCC{} problem with $k$-partition $(V_1,\dotsc,V_k)$ of $V(G)$, we will construct an equivalent instance
  $(\TTT,c,k)$ of \MGAEX{} in fpt-time. Note that since
  a partial example $e'$ is a global abductive explanation for $c$
  w.r.t. $\TTT$ if and only if $e'$ is a global contrastive
  explanation for $1-c$ w.r.t. $\TTT$, this then also implies the
  statement of the theorem for \MGCEX{}.
  $\TTT$ uses one binary feature $f_v$ for every $v \in V(G)$.

  We start by constructing the \DT{} $\TTT_{i,j}$ for every $i,j \in
  [k]$ with $i\neq j$ satisfying the
  following: (*) $\TTT_{i,j}(e) = 1$ for an example $e$ if and only if
  either $e(f_v) = 0$ for every $v \in V_i$ or there exists 
  $v\in V_i$ such that $e(f_{v}) = 1$ and $e(f_{v'}) = 0$
  for every $v' \in (V_i\setminus \{v\}) \cup (N_G(v) \cap V_j)$.
  Let $\TTT_{i}$ be the \DT{} obtained using
  \Cref{lem:DT-construction} for the set of examples $\{e_0\} \cup \SB
  e_v \SM v\in V_i \SE$ defined on the features in $F_i=\SB f_v\SM
  v\in V_i \SE$. Here, $e_0$ is the all-zero example and for every $v
  \in V_i$, $e_v$ is the example that is $1$ only at the feature $f_v$
  and $0$ otherwise.
  Moreover, for every $v\in V_i$, let $\TTT^{v}_{j}$ be the \DT{} obtained using
  \Cref{lem:DT-construction} for the set of examples containing only
  the all-zero example defined on the features in
  $\SB f_{v'}\SM v'\in N_G(v)\cap V_j\SE$.
  Then, $\TTT_{i,j}$ is obtained from $\TTT_i$ after replacing
  the $1$-leaf that classifies $e_v$  with $\TTT^v_j$ for every $v \in
  V_i$. Clearly, $\TTT_{i,j}$ satisfies (*) and since 
  $\TTT_{i}$ has at most $|V_i|^2$ inner nodes and $\TTT^v_{j}$ has at most
  $|V_j|$ inner nodes, we obtain that $\TTT_{i,j}$ has at most
  $\bigoh(|V(G)|^2)$ nodes. 

  For an integer $\ell$, we denote by $\DT(\ell)$ the
  complete \DT{} of height $\ell$,
  where every inner node is
  assigned to a fresh auxiliary feature and every of the exactly
  $2^\ell$ leaves is a $0$-leaf. 
  Let $\TTT_\Delta$ be the \DT{} obtained from the disjoint union of
  $\TTT_U=\DL(k)$ and $2^{k}$ copies $\TTT_D^1,\dotsc,
  \TTT_D^{2^{k}}$ of \DT($\lceil \log (k(k-1))\rceil$) by
  identifying the $i$-th leaf of $\TTT_U$ with the root of $\TTT_D^i$
  for every $i$ with $1 \leq i \leq 2^{k}$; each
  copy is equipped with its own set of fresh features.

  Then,
  $\TTT$ is obtained from $\TTT_\Delta$ after doing the following with
  $\TTT_D^\ell$ for every $\ell \in [2^k]$.
  For every $i,j \in [k]$ with $i\neq j$,
  we replace a private leaf of $\TTT_D^\ell$ with the \DT{}
  $\TTT_{i,j}$; note that this is possible because
  $\TTT_D^\ell$ has at least $k(k-1)$ leaves. Also note that
  $\TTT$ has size at most $\bigoh(|\TTT_\Delta||V(G)|^2)$.
  This completes the construction of $\TTT$ and we set $c=0$.
  Clearly, $\TTT$ can be
  constructed from $G$ in fpt-time w.r.t. $k$. It remains to show that
  $G$ has a $k$-clique if and only if there is a global abductive
  explanation of size at most $k$ for $c$ w.r.t. $\TTT$.
  \iflong %

  Towards showing the forward direction, let $C=\{v_1,\dotsc,v_k\}$ be
  a $k$-clique of $G$ with $v_i \in V_i$ for every $i \in [k]$. We
  claim that $\alpha : \SB f_v \SM v \in V(C) \SE \rightarrow \{1\}$ is a
  global abductive explanation for $c$ w.r.t. $\TTT$, which
  concludes the proof of the forward direction.   
  To see that this is
  indeed the case consider any example $e$ that agrees with $\alpha$,
  i.e., $e$ is $1$ at any feature $f_{v}$ with $v \in V(C)$.
  For this is suffices to show that $\TTT_{i,j}(e)=0$ for every
  $i,j\in [k]$. Because $\TTT_{i,j}$ satisfies (*) and
  because $e(f_{v_i})=1$, it has to
  hold that $e$ is $1$ for at least one feature in
  $\SB f_v \SM v \in N_G(v_i)\cap V_j\SE$. But this clearly holds
  because $C$ is a $k$-clique in $G$.

  Towards showing the reverse direction, let $\alpha : C' \rightarrow
  \{0,1\}$ be a global abductive explanation of size at most $k$ for
  $c$ w.r.t. $\TTT$. 
  We claim that $C=\SB v \SM f_v \in C'\SE$ is a
  $k$-clique of $G$.
  Let $\TTT_\alpha$ be the partial \DT{} obtained from
  $\TTT$ after removing all nodes that can never be reached by an
  example that is compatible with $\alpha$, i.e., we obtain
  $\TTT_\alpha$  from $\TTT$ by removing the subtree rooted at the $1-\alpha(t)$-child for every node $t$ of $\TTT$
  with $\lambda(t) \in C'$. Then, $\TTT_\alpha$ contains
  only $0$-leaves. We first show that there is an $\ell \in
  [2^{k}]$
  such that $\TTT_\alpha$ contains $\TTT_D^\ell$ completely, i.e., this in particular means that $C'$ contains no feature of $\TTT_D^\ell$. To see this, let $x$ be the number features in
  $C'$ that are assigned to a node of $\TTT_U$. Then, $\TTT_\alpha$
  contains the root of at least $2^{k-x}$ \DT{}s
  $\TTT_D^i$. Moreover, since $2^{k-x}\geq k-x$ there is at least
  one $\TTT_D^i$ say $\TTT_D^\ell$, whose associated features are not in
  $C'$.  Therefore, for every $i,j \in [k]$ with $i \neq j$, $\TTT_\alpha$ contains at least the root of $\TTT_{i,j}$. Since $\TTT_{i,j}$ satisfies (*), it follows that for every $\ell \in [k]$ there is $v_\ell \in V_\ell$ such that $\alpha(f_{v_{\ell}})=1$. Because $|C'|\leq k$, we obtain that $C'$ contains exactly one feature $f_{v_\ell} \in F_\ell$ for every $\ell$. Finally, using again that $\TTT_{i,j}$ satisfies (*), we obtain that $v_{i}$ and $v_{j}$ are adjacent in $G$, showing that $\{v_1,\dotsc,v_k\}$ is a $k$-clique of $G$.
 \fi
\end{proof}

\ifshort \begin{lemma}[$\star$]\fi\iflong\begin{lemma}\fi\label{lem:rf-hard-example}
    \RF{}-\HOM{} is \NPh{} and
    both \RF{}-\HOM{}$(\enssize{})$
    and \RF{}-\kHOM{}$(\enssize{})$ are
    \Wh{1}\iflong even if 
    the partial orders of the features on all paths from the root to a leaf
    respect the same total order\fi.
\end{lemma}
\iflong\begin{proof}\fi
\ifshort\begin{proof}[Proof Sketch]\fi  
  We give a parameterized reduction from \MCC{} that is also a
  polynomial-time reduction. That is, given an
  instance $(G,k)$ of \MCC{} with $k$-partition $V_1,\dotsc,V_k$, we
  will construct a \RF{} $\EEE$ with $|\EEE|=2(k+\binom{k}{2})-1$ such
  that $G$ has a $k$-clique if and only if $\EEE$ classifies at least
  one example positively. This will already suffice to show the stated
  results for \RF{}-\HOM{}. Moreover, to show the results for
  \RF{}-\kHOM{} we additional show that $G$ has a $k$-clique if and only if
  $\EEE$ classifies an example positively that sets at most $k$
  features to $1$.

  $\EEE$ will use the set of features $\bigcup_{i \in [k]} F_i$, where $F_i = \SB f_{v} \SM v \in V_i\SE$.
  For each $v \in V_i$ and $u \in V_j$,
  let $e_{v,u}$ be an example 
  defined on set of features $F_i \cup F_j$ 
  that is $1$ only at the features $f_v$ and $f_v$, and otherwise $0$.
  For every $i \in [k]$, $\EEE$ will have a \DT{} $\TTT_i$ 
  obtained using \Cref{lem:DT-construction} 
  for the set of examples $\SB e_{v,v}\SM v \in V_i \SE$
  defined on the features in $F_i$.
  Also, for every $i$ and $j$ with $1 \leq i < j \leq k$, 
  $\EEE$ contains a \DT{} $\TTT_{i,j}$ 
  obtained using \Cref{lem:DT-construction} 
  for the set of examples 
  $\SB e_{v,u}\SM v \in V_i \land u \in V_j \land vu \in E(G) \SE$
  defined on the features in $F_i\cup F_j$.
  Finally, $\EEE$ contains $k+\binom{k}{2}-1$ \DT{}s that classify
  every example negatively, i.e., those \DT{}s consists only of one
  $0$-leaf.
  \ifshort  
  The correctness of the reduction is provided in the long version of the paper.
  \fi
  \iflong
  Clearly, the reduction can be achieved in polynomial-time and
  preserves the parameter, i.e., $|\EEE|=2(k+\binom{k}{2})-1$.
  It remains to show that $G$ has a $k$-clique if and only if $\EEE$
  classifies at least one example positively, which in turn holds if
  and only if $\EEE$ classifies an example positively that sets at
  most $k$ features to $1$.

  Towards showing the
  forward direction, let $C=\SB v_1,\dotsc,v_k\SE$ be a $k$-clique of
  $G$, where $v_i \in V_i$ for every $i$ with $1 \leq i \leq k$. We
  claim that the example $e$ that is $1$ exactly at the features
  $f_{v_1},\dotsc,f_{v_k}$ (and otherwise $0$) is classified
  positively by $\EEE$. By construction, $e$ is classified positively by every \DT{}
  $\TTT_i$ for every $i \in [k]$ and since $C$ is a $k$-clique also every
  \DT{} $\TTT_{i,j}$ for every $1 \leq i < j \leq k$. Therefore, $e$
  is classified positively by $k+\binom{k}{2}$ \DT{}, which represents
  the majority of the \DT{}s in $\EEE$.

  Towards showing the reverse direction, suppose that there is an
  example $e$ that is classified positively by $\EEE$. Because $e$ has
  to be classified positively by the majority of \DT{} in $\EEE$ and
  there are $k+\binom{k}{2}-1$ \DT{}s in $\EEE$ that classify $e$
  negatively, we obtain that $e$ has to be classified positively by
  every \DT{} $\TTT_i$ for every $i \in [k]$ and by every \DT{}
  $\TTT_{i,j}$ for every $1 \leq i < j \leq k$. Since $e$ is
  classified positively by $\TTT_i$, we obtain by construction that
  there is exactly one feature $f_{v_i} \in \SB f_v \SM v \in V_i\SE$
  such that $e(f_{v_i})=1$ and $e$ is $0$ at all other features in
  $f_{v_i} \in \SB f_v \SM v \in V_i\SE$. Since $e$ sets exactly one
  feature $f_{v_i} \in \SB f_v \SM v \in V_i\SE$ to $1$ and $e$ sets
  exactly one feature $f_{v_j} \in \SB f_v \SM v \in V_i\SE$ to $1$
  and because $e$ is classified positively by $\TTT_{i,j}$, we obtain
  from the construction that $v_i$ and $v_j$ are adjacent in $G$. Therefore,
  $C=\{v_1,\dotsc,v_k\}$ is a $k$-clique of $G$. \fi
\end{proof}

The final two theorems of this section provide all the remaining
hardness results for \RF{}s and follow from
\Cref{lem:rf-hard-example} and \Cref{lem:MME-kHOM-W1}, respectively,
together with \Cref{lem:HOM-reduction,lem:kHOM-reduction}.

\begin{theorem}\label{th:RF-W1-ES}
  Let $\PP \in \{\LAEX{}, \GAEX{}, \GCEX{}\}$. 
  \RF{}-\SPP{}$(\enssize{})$ is \coWh{1};
  \RF{}-\SMLCEX{}$(\enssize{})$ is \Wh{1};
  \RF{}-\MPP{}$(\enssize{})$ is \coWh{1} even if \xpsize{} is constant;
  \RF{}-\MLCEX{}$(\enssize{}+\xpsize{})$ is \Wh{1}.
\end{theorem}
\ifshort \begin{theorem}[$\star$]\fi\iflong\begin{theorem}\fi\label{th:RF-paraNP}
  Let $\PP \in \{\LAEX{}, \GAEX{}, \GCEX{}\}$. 
  \RF{}-\SPP{} is \coNPh{} even if  $\mnlsize{} + \sizeelem{}$ is constant;
  \RF{}-\SMLCEX{} is \NPh{} even if  $\mnlsize{} + \sizeelem{}$ is constant;
  \RF{}-\MPP{} is \coNPh{} even if  $\mnlsize{}+ \sizeelem{} + \xpsize{}$ is constant;
  \RF{}-\MLCEX{}$(\xpsize{})$ is \Wh{1} even if $\mnlsize{} + \sizeelem{}$ is constant.
\end{theorem}
\iflong\begin{proof}

    We will use \Cref{lem:MME-kHOM-W1} to show the theorem by
    giving \DT{}s $M^0$, $M^1_{f}$ and $M^2_{\{f_1,f_2\}}$ satisfying the
    conditions of \Cref{lem:MME-kHOM-W1} as follows.
    We let $M^0$ be the \DT{} consisting merely of one $0$-leaf.
    We let $M^1_{f}$ be the \DT{} $(T^{1}_f, \lambda^1_f)$, such that 
    $t$ is the only inner node of $T^{1}_f$ with $\lambda^1_f(t)=f$. 
    Let $0$-leaf and $1$-leaf be the $0$-child and
    $1$-child of $t$, respectively.

    We let $M^2_{\{f_1,f_2\}}$ be the \DT{} $(T^2_{\{f_1,f_2\}} , \lambda^2_{\{f_1,f_2\}})$, such that
    $t^{f_1}$ and $t^{f_2}$ are the only inner nodes of $T^2_{\{f_1,f_2\}}$ with 
    $\lambda^2_{\{f_1,f_2\}}(t^{f_1})=f_1$ and $\lambda^2_{\{f_1,f_2\}}(t^{f_2})=f_2$.
    Let $1$-leaf and $t^{f_2}$ be the $0$-child and $1$-child of $t^{f_1}$, respectively.
    Let $1$-leaf and $0$-leaf be the $0$-child and $1$-child of $g^{f_2}$, respectively.

    The constructions of the \DT{}s are self explainable.
    Note that $M^0$, $M^1_{f}$ and $M^2_{\{f_1,f_2\}}$ have \sizeelem at most $5$.
    All statements in the theorem now follow from 
    \Cref{lem:MME-kHOM-W1,lem:kHOM-reduction,lem:HOM-reduction}.
\end{proof}\fi

\subsection{DSs, DLs  and their Ensembles}

Here, we establish our hardness results for \DS{}, \DL{}s, and their
ensembles. It is interesting to note that there is no real distinction
between \DS{} and \DL{}s when it comes to explainability and that both are considerably
harder to explain then \DT{}s and \OBDD{}s.

\newcommand{\TAU}{\textsc{Taut}}
\iflong
\pbDef{{\sc $3$-DNF Tautology} (\TAU)}
{A $3$-DNF formula $\psi$.}
{Is there an assignment that falsifies $\psi$.}
\fi
\ifshort \begin{theorem}[$\star$]\fi\iflong\begin{theorem}\fi\label{th:DS-paraCoNP}
  Let $\MM \in \{\DS,\DL\}$ and let $\PP \in \{\LAEX, \GAEX, \GCEX\}$. 
  \MM-\SPP{} is \coNPh{} even if \termsize{} is constant;
  \MM-\SMLCEX{} is \NPh{} even if \termsize{} is constant;
  \MM-\MPP{} is \coNPh{} even if $\termsize{} + \xpsize{}$ is constant.
\end{theorem}
\iflong\begin{proof}
    Because every \DS{} can be easily transformed into an equivalent
    \DL{} that shares all of the same parameters, it is sufficient to
    show the theorem for the case that $\MM=\DS$.
    
    First we give a polynomial-time reduction from \TAU{} to
    \DS-\HOM{}. 
    Let $\psi$ be the $3$-DNF formula given as an instance of \TAU.
    W.l.o.g., we can assume that the all-zero assignment satisfies
    $\psi$, since otherwise $\psi$ is a trivial no-instance.
    Then, the \DS{} $\ds = (T_{\psi}, 0)$, where $T_{\psi}$ is the set
    of terms of $\psi$ is the constructed instance of
    \DS-\HOM{}. Since the reduction is clearly polynomial, it only
    remains to show that $\phi$ is a tautology if and only if $\ds$
    classifies all examples in the same manner as the all-zero
    example, i.e., positively. But this is easily seen to hold, which
    concludes the reduction from \TAU{} to \DS-\HOM{}.
    All statements in the theorem now follow from 
    \Cref{lem:HOM-reduction}.
\end{proof}\fi

\ifshort \begin{theorem}[$\star$]\fi\iflong\begin{theorem}\fi\label{th:DS-SMLCEX-W1}
  Let $\MM \in \{\DS,\DL\}$. 
  \MM-\MLCEX{}$(\xpsize{})$ is \Wh{1}.
\end{theorem}
\iflong\begin{proof}
    Again it suffices to show the theorem for the case that $\MM=\DS$.
  We provide a parameterized reduction from \MCC{} to \DS-\kHOM{},
  which is also a polynomial-time reduction. Given an
  instance $(G,k)$ of \MCC{} with a $k$-partition $V_1,\dotsc,V_k$, we
  will construct a \DS{} $\ds = (T_{(1)}\cup T_{(2)}, 1)$
  such
  that $G$ has a $k$-clique if and only if $\ds$ classifies at least
  one example positively, that sets at most $k$ features to $1$. 
  $\ds$ will use one binary feature $f_v$ for
  every vertex $v \in V(G)$ and will classify an example $e$ positively
  if and only if all of the following items are true:
  \begin{enumerate}[(1)]
      \item for every non-edge $uv \notin E(G)$, either the feature
        $f_u$ or the feature $f_v$ are $0$ in $e$, i.e., $e(f_u)=0$ or $e(f_v)=0$ 
      \item for every $i \in [k]$, at least one of the features
        in $\SB f_v \SM v \in V_i \SE$ is $1$ in $e$,
  \end{enumerate}
  For every property (i) from the list above, we create a set of terms
  $T_{(i)}$ such that an example $e$ does not satisfy any term in
  $T_{(i)}$ if and only if it satisfies (i).
  Let $T_{(1)}$ be the set of terms equal to $\SB \{(f_u=1), (f_v=1)\}\SM  uv \notin E(G)\SE$.
  For every $1 \leq i \leq k$, let $t_i$  be the term $\SB (f_v=0) \SM  v \in V_i\SE$.
  Let $T_{(2)}$ be the set of terms equal to $\SB t_i \SM  1 \leq i \leq k \SE$.

  Clearly, the reduction can be achieved in polynomial-time.
  It remains to show that $G$ has a $k$-clique if and only if $\ds$
  classifies at least one example positively,
  that sets at most $k$ features to $1$.
  Towards showing the
  forward direction, let $C=\SB v_{1},\dotsc,v_{k}\SE$ be a $k$-clique of
  $G$. We
  claim that the example $e$ that is $1$ exactly at the features
  $\SB f_{v_1}, \dotsc f_{v_k}\SE$
  (and otherwise $0$) is classified
  positively by $\ds$. 
  By construction, $e$ does not satisfy any of the terms from
  $T_{(2)}$. Moreover, since $C$ is a $k$-clique also 
  no term from $T_{(1)}$ is satisfied by $e$.
  Therefore, $e$ is classified positively by $\ds$.

  Towards showing the reverse direction, suppose that there is an
  example $e$,   that sets at most $k$ features to $1$,
  which is classified positively by $\ds$, i.e., no term from
  $T_{(1)}$ and $T_{(2)}$ is satisfied by $e$.
  Since $e$ does not satisfy $T_{(2)}$ we get that $e$ sets
  at least $k$ to $1$, i.e., one feature $f_v$ with $v \in
  V_i$ for every $i \in [k]$.
  Moreover, since $e$ does not satisfy $T_{(1)}$, all the vertices $v$
  with $e(f_v)=1$ form a clique in $G$. Therefore, $C=\{v | e(f_v)=1
  \land v \in V(G)\}$ is a $k$-clique of $G$.

  Note that $\ds$ classifies the example $e_0$ negatively.
  This implies that \DS{}-\kHOM{} is \Wh{1} and
  because of  \Cref{lem:kHOM-reduction} we obtain that
  \DS-\MLCEX{}$(\xpsize)$ is \Wh{1}.
\end{proof}\fi

\ifshort \begin{theorem}[$\star$]\fi\iflong\begin{theorem}\fi\label{th:DSE-SMLCEX-W1}
  Let $\MM \in \{\DSE,\DLE\}$. 
  \MM-\MLCEX{}$(\enssize{}+\xpsize{})$ is \Wh{1} even if \termsize{} is constant.
\end{theorem}
\iflong\begin{proof}
    Again it suffices to show the theorem for the case that $\MM=\DSE$.
  We provide a parameterized reduction from \MCC{} to \DSE-\kHOM{},
  which is also a polynomial-time reduction. Given an
  instance $(G,k)$ of \MCC{} with a $k$-partition $V_1,\dotsc,V_k$, we
  will construct a \DSE{} $\dse$ 
  such that $|\dse| = 2k+1$, \termsize{} is equal to $2$ and
  $G$ has a $k$-clique if and only if $\dse$ classifies at least
  one example positively, 
  that sets at most $k$ features to $1$. 
  We will construct \DSE $\dse = \SB \ds_{(1)}, \ds_{(2)}^1, \dotsc,
  \ds_{(2)}^k, \ds_{\perp}^1,\dotsc, \ds_{\perp}^k \SE$ which use one
  binary features $f_v$ for
  every vertex $v \in V(G)$ and will classify an example $e$ positively
  if and only if all of the following items are true:
  \begin{enumerate}[(1)]
      \item for every non-edge $uv \notin E(G)$, either the feature
        $f_u$ or the feature $f_v$ are $0$ in $e$, i.e., $e(f_u)=0$ or $e(f_v)=0$ 
      \item for every $i \in [k]$, at least one of the features
        in $\SB f_v \SM v \in V_i \SE$ is $1$ in $e$,
  \end{enumerate}
  Let $T_{(1)}$ be the set of terms equal to $\SB \{(f_u=1), (f_v=1)\}\SM  uv \notin E(G)\SE$
  and let $\ds_{(1)} = (T_{(1)}, 1)$.
  For every $1 \leq i \leq k$, let $\ds_{(2)}^i = (\SB \{(f_v = 1)\}\SM v \in V_i \SE,0)$ and 
  let $\ds_{\perp}^i = (\emptyset, 0)$.

  Clearly, the reduction can be achieved in polynomial-time.
  It remains to show that $G$ has a $k$-clique if and only if $\dse$
  classifies at least one example positively,
  that sets at most $k$ features to $1$.
  Towards showing the
  forward direction, let $C=\SB v_{1},\dotsc,v_{k}\SE$ be a $k$-clique of
  $G$. We
  claim that the example $e$ that is $1$ exactly at the features
  $\SB f_{v_1}, \dotsc f_{v_k}\SE$
  (and otherwise $0$) is classified
  positively by $\dse$. 
  By construction, $e$ is positively classified by every \DS{} $\ds_{(2)}^i$.
  Moreover, since $C$ is a $k$-clique then 
  no term from $T_{(1)}$ is satisfied by $e$, 
  which means that $e$ is classified positively by $\ds_{(1)}$.
  Therefore, $e$ is classified positively by $\dse$.

  Towards showing the reverse direction, suppose that there is an
  example $e$, that sets at most $k$ features to $1$,
  which is classified positively by $\dse$.
  Note that all $\ds_{\perp}^i$ classify all example negatively.
  Therefore \DS{}s $\ds_{(1)}, \ds_{(2)}^1, \dotsc, \ds_{(2)}^k$ classify $e$ positively.
  Since $e$ is classified positively by $\ds_{(2)}^i$, 
  we get $e$ sets exactly one feature from $\SB f_v  \SM v \in
  V_i \SE $ to $1$.
  Moreover, since $e$ is classified positively by $\ds_{(1)}$,
  all the vertices $v\in V(G)$ with $e(f_v)=1$ form a clique in $G$.
  Therefore, $C=\{v | e(f_v)=1 \land v \in V(G)\}$ is a $k$-clique of $G$.

  Note that $\dse$ classifies an example $e_0$ negatively.
  This implies that \DSE{}-\kHOM{}$(\enssize{})$ is \Wh{1} even if \termsize{} is constant.
  Moreover, by combining with \Cref{lem:kHOM-reduction} we get that
  \DSE-\MLCEX{}$(\xpsize{}+\enssize{})$ is \Wh{1}.
\end{proof}\fi

\ifshort \begin{theorem}[$\star$]\fi\iflong\begin{theorem}\fi\label{th:DSE-paraNP} 
  Let $\MM \in \{\DSE,\DLE\}$ and let $\PP \in \{\LAEX, \GAEX, \GCEX\}$. 
  \MM-\SPP{} is \coNPh{} even if $\termselem{} + \termsize{}$ is constant;
  \MM-\SMLCEX{} is \NPh{} even if $\termselem{} + \termsize{}$ is constant;
  \MM-\MPP{} is \coNPh{} even if $\termselem{}+ \termsize{} + \xpsize{}$ is constant;
  \MM-\MLCEX{}$(\xpsize{})$ is \Wh{1} even if $\termselem{} + \termsize{}$ is constant.
\end{theorem}
\iflong\begin{proof}
    Again it suffices to show the theorem for the case that $\MM=\DSE$.
  We will use \Cref{lem:MME-kHOM-W1} to show the theorem by
  giving \DS{}s $M^0$, $M^1_{f}$ and $M^2_{\{f_1,f_2\}}$ satisfying the
  conditions of \Cref{lem:MME-kHOM-W1} as follows.
  We let $M^0$ be the \DS{} $(\emptyset,0)$.
  We let $M^1_{f}$ be the \DS{} given by $(\{\{(f=1)\}\}, 0)$ and
  we let $M^2_{\{f_1,f_2\}}$ be the \DS{} given by $(\{\{(f_1 =
  0),(f_2 = 0 )\}\} , 1)$.
  Note that $M^0$, $M^1_{f}$ and $M^2_{\{f_1,f_2\}}$ have \termsize{} at most $2$ and \termselem{} at most $1$.
  All statements in the theorem now follow from 
  \Cref{lem:MME-kHOM-W1,lem:kHOM-reduction,lem:HOM-reduction}.
\end{proof}\fi

\subsection{OBDDs and their Ensembles}\label{ssec:hard-obdd}

\newcommand{\atl}[1]{\stackrel{#1}{+}}
\newcommand{\Atl}[1]{\stackrel{#1}{\Sigma}}
\newcommand{\Lor}{\bigvee}
\newcommand{\Land}{\bigwedge}
\newcommand{\bla}{ Bla bla bla...}

\newcommand{\nx}[1]{{next(#1)}}
\newcommand{\blank}{f_{\#}}

We are now ready to provide our hardness results for \OBDD{}s and
their ensembles \OBDDEO{}s and \OBDDE{}s. While the proofs for
\OBDD{}s and \OBDDEO{}s follow along very similar lines as the
corresponding proofs for \DT{}s, the main novelty and challenge of this
subsection are the much stronger hardness results for \OBDDE{}s.
Informally, we show that the satisfiability of any CNF formula $\phi$ can be modelled
in terms of an ensemble of two \OBDD{}s $\obdd_1$ and $\obdd_2$ each using a different
ordering of the variables. In particular, it holds that both \OBDD{}s
classify an example positively if and only if the corresponding
assignment satisfies $\phi$. The main idea behind the construction of
$\obdd_1$ and $\obdd_2$ is
to make copies for every occurrence of a variable in $\phi$ and then
use $\obdd_1$ to verify that the assignment satisfies $\phi$ and
$\obdd_2$ to verify that all copies of every variable are assigned to the
same value. 

\ifshort \begin{theorem}[$\star$]\fi\iflong\begin{theorem}\fi\label{th:OBDDE-paraNP}
  Let $\PP \in \{\LAEX, \GAEX, \GCEX\}$. 
  \OBDDE-\SPP{} is \coNPh{} even if  $\enssize{} + \widthelem{}$ is constant;
  \OBDDE-\SMLCEX{} is \NPh{} even if  $\enssize{} + \widthelem{}$ is constant;
   \OBDDE-\MPP{} is \coNPh{} even if  $\enssize{} + \widthelem{} +  \xpsize{}$ is constant;
  \OBDDE-\MLCEX{}$(\enssize{})$ is \Wh{1} even if  $\enssize{} + \widthelem{}$ is constant.
\end{theorem}
\iflong\begin{proof}
  We provide a parameterized reduction from \MCC{} to \OBDDE-\kHOM{},
  which is also a polynomial-time reduction. Given an
  instance $(G,k)$ of \MCC{} with a $k$-partition $V_1,\dotsc,V_k$, we
  will construct an \OBDDE{} $\obdde$ such that
  $|\obdde|=3$, the maximal
  width of any \OBDD{} in $\obdde$ at most $4$
  and
  $G$ has a $k$-clique if and only if $\obdde$ classifies at least
  one example positively. Let $\{v_1,\dotsc, v_n\} = V(G) = \bigcup_{i} V_i$.
  $\obddeo$ will use $k+2$ binary vertex features $f_a, f_a', f_a^1,\dotsc f_a^k$ for
  every vertex $v_a \in V(G)$ and three binary edge features $f_{a,b}'$,
  $f_{a,b}$ and $f_{b,a}$ for each edge $v_av_b \in E(G) \land a<b$.
  
  Before constructing $\obdde$, we start by providing some helpful
  auxiliary \OBDD{}s that will correspond to operations on
  some subset $F$ of features. For convenience, we define the fresh
  auxiliary feature $\blank$ that will not be part of the constructed
  \OBDD{} but instead only helps with its construction.
  We also fix an arbitrary ordering $<_F$ of the features in $F$.
  Let $\nx{}$ be the successor function of the features in $F$
  w.r.t. $<_F$, i.e., for every $f \in F$, $\nx{f} = f'$ if $f$ has a
  successor $f'$ and $\nx{f}=\blank$ otherwise.
  Let $D_F$ be the directed acyclic graph with
  vertices $V(D_F)=\SB g^f_0, g^f_1, g^f_2 \SM f\in F\cup \{\blank\} \SE$
  and source vertex $s = g_0^{\min_{<_F}F}$. Let $\fBDD_{F}$ be the
  function giving by setting $\fBDD_{F}(g^f) =f$ for every $g^f\in
  V(D_F)$. We now define different \OBDD{}s based on $(D_F,
  \fBDD_{F})$, which only differ in the edges and the assignment of $0$-neighbors
  and $1$-neighbors.
  
  We first define the \OBDD{} $\obdd^{!}_F$ that
  classifies an example $e$ positively if and only if exactly one of
  the features $f \in F$ appears positively in $e$. $\obdd^{!}_F =
  (D^{!}_F, \fBDD_{F})$ is obtained from $(D_F, \fBDD_{F})$ as follows.
  For each $f \in F$ and $i \in \{0,1,2\}$, let $g^\nx{f}_i$ and
  $g^\nx{f}_{\max \{i+1, 2\}}$ be the $0$-neighbor and $1$-neighbor of
  $g^f_i$, respectively. We set $t_1$ to be $g^{\blank}_1$  and we
  set $t_0$ to be equal to the vertex obtained after identifying
  $g^{\blank}_0$ with $g^{\blank}_2$.
  Note that the \OBDD{} $\obdd^{!}_F$ simulates a simple counter, i.e.,
  if an example $e$ reaches a node $g^f_i$ it means that for $i$ equal
  to $0$, $1$ and $2$ there are no, exactly one and more then one,
  respectively, features in $\{f' | f' \in F \land f' <_F f\}$, for
  which $e$ is $1$.  
      
  We now define the \OBDD{} $\obdd^{\exists}_F = (D^{\exists}_F, \fBDD_{F})$ 
  that classifies an examples $e$
  positively if and only if there is a feature $f \in F$, where $e$ is
  set to $1$, i.e. $\exists_{f\in F} e(f)=1$.
  The construction of $D^{\exists}_F$ is similar to $D^{!}_F$, but 
  we set $t_0$ to be $g^{\blank}_0$  and we
  set $t_1$ to be equal to the vertex obtained after identifying
  $g^{\blank}_1$ with $g^{\blank}_2$. 
  This modification ensures that only examples
  with one or more features in $F$ set to $1$ reach $t_1$.

  Next, we define the \OBDD{} $\obdd^{\Leftrightarrow \exists}_{\{f\} \cup F} =
  (D^{\Leftrightarrow \exists}_{\{f\} \cup F}, \fBDD^{\Leftrightarrow
    \exists}_{\{f\}\cup F})$ 
  that classifies an examples $e$
  positively if and only if 
  the special feature $f$ is set to $1$ by $e$ if and only if
  there is a feature $f' \in F$ that appears
  positively in $e$,
  i.e., $e(f)=1 \Leftrightarrow \exists_{f'\in F} e(f')=1$.
  To obtain $D^{\Leftrightarrow \exists}_{\{f\} \cup F}$ we modify 
  $D^{\exists}_F$ by replacing $t_0$ and $t_1$ with new vertices
  $g^f_0$ and $g^f_1$, respectively. Moreover for each $i \in \{0,1\}$, 
  let $t_1$ be $i$-neighbor of $g^f_i$ and let $t_0$ be the other neighbor.
  The modification compares the result from $D^{\exists}_F$ with $f$.
  
  We now define the \OBDD{} $\obdd^{=}_F = (D^{=}_F, \fBDD_{D^{=}_F})$ 
  that classifies an example $e$ 
  positively if and only if 
  all or none of the features $f \in F$ appears positively in $e$,
  i.e., $\forall_{f\in F} e(f)=0$ or $\forall_{f\in F} e(f)=1$.
  For each $f \in F$, let $g^\nx{f}_2$ be the $0$-neighbor and $1$-neighbor of $g^f_2$.
  Let $f_{first}$ be the first feature in $F$.
  Let $g^\nx{f_{first}}_0$ and $g^\nx{f_{first}}_1$ be the $0$-neighbor
  and $1$-neighbor of the source vertex $s = g^{f_{first}}_0$, respectively.
  For each $f\in F \setminus \{f_{first} \}$ and $i \in \{0, 1\}$,
  let $g^\nx{f}_i$ be $i$-neighbor of $g^f_i$ and let $g^\nx{f}_2$ be the other neighbor.
  We set $t_0$ to be $g^{\blank}_2$  and we
  set $t_1$ to be equal to the vertex obtained after identifying
  $g^{\blank}_0$ with $g^{\blank}_1$.
  Note that vertices in $D^{=}_F$ have different roles compared to
  other reductions, which are to store information about
  $e(f_{first})$ and the negative result.
  If example $e$ reaches node $g^f_i$ and $i \in \{0,1\}$, it implies
  that for every $f'\in F \land f' < f$, $e(f_{first}) = e(f') = i$.
  If example $e$ reaches node $g^f_2$, it implies that there exists
  $f' \in F \land f' <_F f$, such that $e(f_{first}) \neq e(f')$.
  
  Note that all of the above \OBDD{}s have width equal to $3$.
  Before showing our main reduction we need to introduce a tool to help
  us merge the result of several \OBDD{}s into one.
  
  Let $\obdd_1, \cdots ,\obdd_m$ be \OBDD{}s such that no two of them share a
  feature. We will construct the \OBDD{} $\obdd = (D, \fBDD)$ 
  that classifies an examples $e$
  positively if and only if for every $1 \leq i \leq m$, $\obdd_i$
  classifies $e$ positively.
  For each $1 \leq i \leq m$, let $\obdd_i = (D_i, \fBDD_i)$, $F_i$ be
  set of features of $\obdd_i$ with an order $<_{F_i}$ and let $s^i$,
  $t^i_0$, $t^i_1$ be the special nodes $s$, $t_0$, $t_1$ of $D_i$,
  respectively.
  Let $F = \bigcup F_i$ and let $<_F$ be the ordering of the features
  in $F$ such that for every $f_1 \in F_i$ and $f_2 \in F_j$, if $f_1 <_F f_2$
  then $i<j$ or $i=j$ and $f_1 <_{F_i} f_2$.

  Let $D$ be the graph obtained from the union of the graphs $D_i$
  after adding the new vertices $\{t_0, t_1\} \cup \{g_{\perp}^f | f
  \in F\cup \{\blank\} \}$ as follows.
  For every $f\in F$, let $g_{\perp}^\nx{f}$ be the $0$-neighbor and
  $1$-neighbor of $g_{\perp}^f$.
  For every $ 1 \leq i <m$, we identify $s^{i+1}$ with $t_1^i$ and
  $g^{f_{first}^{i+1}}_{\perp}$ with $t_0^i$, where $f_{first}^{i+1}$
  is the $<_{F_{i+1}}$-smallest feature in $F_{i+1}$.
  Moreover, we set $t_1$ to be $t_1^m$ and we
  set $t_0$ to be equal to the vertex obtained after identifying
  $g^{\blank}_{\perp}$ with $t_0^m$.
  Note that if example $e$ reaches a node $g^f_{\perp}$ and $f \in
  F_i$, it means there exists $1\leq j <i$ such that $\obdd_j$
  classifies $e$ negatively.
  Moreover, the width of $\obdd$ is equal maximal width of $\obdd_i$ plus one.  

  Finally we are ready to provide our reduction from \MCC{}. We will
  construct $\OBDDEO$ $\obddeo = \{\obdd_1, \obdd_2\, \obdd_{3}\}$, which
  will classify an example $e$ positively if and only if all of the
  following items are true:
  \begin{enumerate}[(1)]
  \item for every $v_a \in V(G)$, all or none of the vertex features $f_a, f_a', f_a^1,\dotsc f_a^k$ are set to $1$ in $e$,
      \item for every $v_av_b \in E(G)$ and $a<b$, all or none of the edge features $f_{a,b}'$, $f_{a,b}$, $f_{b,a}$ are set to $1$ in $e$,
      \item for every $1 \leq i \leq k$, exactly one feature from $F^i
        = \{f_a' | v_a \in V_i\}$ is set to $1$ in $e$.
      \item for every $1 \leq i<j \leq k$, exactly one feature from $F^{i,j} =\{ f_{\min \{a,b\}, \max\{a,b\}}' | v_a \in V_i \land v_b \in V_j \land v_av_b\in E(G)\}$ is set to $1$ in $e$,
      \item for every $1 \leq i,j \leq k \land i\neq j$ and $v_a \in V_i$, $e(f_a^j)=1$ if and only if at least one feature from $F_a^j = \{f_{a,b} | v_b \in V_j \land v_av_b\in E(G)\}$ is set to $1$ in $e$ 
  \end{enumerate}
For every item $(i)$ from the list above we create $\obdd_{(i)}$ which
classify an example $e$ positively if and only if it satisfies property $(i)$.

\begin{align*}    
&\obdd_{(1)} = \Land_{v_a \in V(G)} \obdd^{=}_{\{f_a, f_a', f_a^1,\dotsc f_a^k\}}\\
&\obdd_{(2)} = \Land_{v_av_b \in E(G) \land a<b} \obdd^{=}_{\{f_{a,b}',f_{a,b}, f_{b,a}\}}\\
&\obdd_{(3)} = \Land_{1\leq i \leq k} \obdd^{!}_{F^i} \\
&\obdd_{(4)} = \Land_{1\leq i < j \leq k} \obdd^{!}_{F^{i,j}}\\
&\obdd_{(5)} = \Land_{ 1 \leq i,j \leq k \land i\neq j \land v_a \in V_i} \obdd^{\Leftrightarrow \exists}_{\{f_a^j \cup F_a^j\}} \\
\end{align*}

Let $\obdd_1 = \obdd_{(1)} \land \obdd_{(2)}$, $\obdd_2 = \obdd_{(3)} \land \obdd_{(4)} \land \obdd_{(5)}$ and let $\obdd_{3}$ be a trivial \OBDD that classifies every
  example negatively.
  
Note that $\obddeo$ classifies an example $e$ positively if and only if
$\obdd_1$ and $\obdd_2$ classifies $e$ positively, i.e., if and only if
for every $ 1 \leq i  \leq 5$, $\obdd_{(i)}$ classifies $e$ positively. 
Moreover, the width of $\obdd_1$ and $\obdd_2$ is $4$, because they
are made from the logical conjunction of \OBDD{}s
$\obdd^{!}_F$, $\obdd^{\exists}_F$, $\obdd^{\Leftrightarrow
  \exists}_{\{f\} \cup F}$ and $\obdd^{=}_F$ of width $3$.
This completes the description of our reduction, which can clearly be
achieved in polynomial-time.
It remains to show that $G$ has a $k$-clique if and only if $\obdde$
classifies at least one example positively.

Towards showing the
forward direction, let $C=\SB v_{a_1},\dotsc,v_{a_k}\SE$ be a $k$-clique of
$G$, where $v_{a_i} \in V_i$ for every $i$ with $1 \leq i \leq k$. We
claim that the example $e$ that is $1$ exactly at the features
$\SB f_{a_i}, f_{a_i}', f_{a_i}^j \SM 1 \leq i,j \leq k\SE \cup
\SB f_{a_i, a_j}', f_{a_i, a_j}, f_{a_j, a_i} \SM 1 \leq i,j \leq k \land i<j \SE$
(and otherwise $0$) is classified
positively by $\obdde$. By construction, $e$ is positively classified by 
$\obdd_{(1)}$ and $\obdd_{(2)}$, since $C$ is a $k$-clique also 
$\obdd_{(3)}$, $\obdd_{(4)}$ and $\obdd_{(5)}$.
Therefore, $e$
is classified positively by $\obdd_1$ and $\obdd_2$, which represents
the majority of the \DT{}s in $\obdde$.

Towards showing the reverse direction, suppose that there is an
example $e$ that is classified positively by $\obdde$. Because $e$ has
to be classified positively by the majority of \OBDDE{} in $\obdde$, we obtain that $e$ has to be classified positively by
$\obdd_1$ and $\obdd_2$. Since $e$ is
classified positively by $\obdd_1$, we obtain all of the features 
that corresponds to a specific vertex or a specific edge are assigned the same way.
Moreover, since $\obdd_{(5)}$ and $\obdd_1$ classify $e$ positively,
this means that if edge feature $f_{a,b}$ is set to $1$, then both
vertex features $f_a$ and $f_b$ are set to $1$.
Since $e$ is classified positively by $\obdd_{1}$, $\obdd_{(3)}$ and $\obdd_{(4)}$,
it means that the edge features of exactly $\binom{k}{2}$ edges are set to $1$
and that the vertex features of exactly $k$ vertices are set to $1$.
Therefore, $C=\{v_{a} | e(f_a)=1 \land v_a \in V(G)\}$ is a $k$-clique
of $G$.

Note that the all-zero example is classified negatively by $\obdde$ 
and if there is an example $e$ such that $\obdde$ classify $e$ positively 
then $e$ contains exactly $3\binom{k}{2} + k(k+2)$ positively assigned features.
It implies that \OBDDE{}-\kHOM{} is \Wh{1} even if \enssize{} equals
$3$ and \widthelem{} equals $4$,
which combined with \Cref{lem:kHOM-reduction} shows that
\OBDDE-\MLCEX{} is \Wh{1} parameterized by \xpsize{} even if $\enssize{}+\widthelem{}$ is constant.
Moreover, \OBDDE{}-\HOM{} is \NPh{} even if \enssize{} equals $3$ and \widthelem{} equals $4$,
which combined with \Cref{lem:HOM-reduction} shows every other statement in the theorem.
\end{proof}\fi

\iflong
\begin{lemma}\label{lem:DT-to-OBDD}
  Let $\TTT$ be an ordered $\DT{}$ for the ordering $<_{\TTT}$. There is an $\OBDDOc{<_{\TTT}}$
  $\obdd^{\TTT}$ that can be constructed in polynomial-time such that
  $\TTT(e) = \obdd^{\TTT}(e)$ for every example $e$.
\end{lemma}
\begin{proof}
  Let $\TTT=(T,\lambda)$. We define $\obdd^{\TTT}=(D^\TTT,\fBDD^\TTT)$ as follows. Let
  $D^\TTT$ be the directed acyclic graph obtained from $T$ after
  directing all edges away from the root and identifying every
  $b$-leaf with a new vertex $t_b$ for every $b \in
  \{0,1\}$. Moreover, let $\fBDD^\TTT(v)=\lambda(v)$ for every inner
  vertex of $D^\TTT$. This completes the construction of our
  \OBDD{}, which because of \cite[Observation 1]{DBLP:journals/corr/abs-2104-02563} can also be turned
  into an equivalent complete \OBDD{}.
\end{proof}
\fi

A special case of the following theorem, the NP-hardness of \MLAEX{}
for FBDDs, i.e., ``free'' BDDs, was shown by
\citet{BarceloM0S20}.

\ifshort \begin{theorem}[$\star$]\fi\iflong\begin{theorem}\fi\label{th:OBDD-MLAEX-W2}
  Let $\PP \in \{\LAEX, \GAEX, \GCEX\}$. 
  \OBDD-\MPP{} is \NPh{} and \OBDD-\MPP{}$(\xpsize{})$ is \Wh{2}.
\end{theorem}
\iflong\begin{proof}
    Combining  \Cref{th:DT-MLAEX-W2} with \Cref{lem:DT-to-OBDD}  shows 
    that \OBDD-\MLAEX{} is \Wh{2} parameterized by \xpsize.

  Next, we will provide a polynomial-time reduction from
  \OBDD-\MLAEX{} to \OBDD-\MGAEX{} and one from \OBDD-\MGAEX{} to \OBDD-\MGCEX{}.
  Let $(\obdd, e, k)$ be an instance for \OBDD-\MLAEX{}, let
  $F=\feat(\obdd)$ and let $<_F$ be the order of the features as they
  occur on paths of $\obdd$.
  We now construct the \OBDDOc{<_F} $\obdd_e^k = (D_e^k, \fBDD_e^k)$
  (which uses the order  $<_F$) that classifies an example $e'$
  as $\obdd(e)$ if and only if there are at
  least $k$ features that $e$ and $e'$ agree on, i.e., $k \leq |\{f :
  e(f) = e'(f) \land f\in F\}|$.
  For convenience, we will again use the auxiliary feature $\blank$,
  which will not occur in the final \OBDD{}.
  Let $\nx{}$ be the successor function of the features in $F$
  w.r.t. $<_F$, i.e., for every $f \in F$, $\nx{f} = f'$ if $f$ has a
  successor $f'$ and $\nx{f}=\blank$ otherwise.
  Let $D_e^k$ be the directed acyclic graph with
  vertices $V(D_e^k)=\SB g^f_i \SM 0\leq i \leq k \land f\in F\cup \{\blank\} \SE$.
  Let $\fBDD_{e}^k$ be the
  function defined by setting $\fBDD_{e}^k(g^f) =f$ for every $g^f\in
  V(D_e^k)$.
  For each $f \in F$, let $g_k^\nx{f}$ be the $0$-neighbor and $1$-neighbor of $g_k^f$.
  For each $f \in F$ and $0\leq i < k$, 
  let $g_{i+1}^\nx{f}$ be the $e(f)$-neighbor of $g_{i}^f$ and let $g_{i}^\nx{f}$ be the other neighbor.
  We set $t_1$ to be $g^{\blank}_k$  and we
  set $t_0$ to be equal to the vertex obtained after identifying all
  vertices in
  $\{g^{\blank}_i | 0 \leq i < k\}$.
  Note that the graph $D_e^k$ simulates a simple counter, i.e.,
  if an example $e'$ reaches node $g^f_i$ and $i<k$ it means that 
  there are exactly $i$ features from $\{f' | f' \in F \land f' <_F
  f\}$ that $e$ and $e'$ agree on.
  
  Let $\obddeo^{<_F}$ be the \OBDDEOc{<_F} given by
  $\obddeo^{<_F}=\{\obdd, \obdd_e^k,\obdd_{\perp}\}$, where
  $\obdd_{\perp}$ is a trivial \OBDDOc{<_F} that classifies every
  example $1-\obdd(e)$.
  We apply \Cref{lem:OBDDEO-to-OBDDE} to $\obddeo^{<_F}$ to create
  an \OBDDOc{<_F} $\obddo'$ that is equivalent with $\obddeo^{<_F}$.
  Note this reduction can be achieved in polynomial-time.
  
  We are now ready to prove that the following problems are equivalent:
  \begin{enumerate}[(1)]
      \item $(\obdd, e, k)$ is a yes-instance of \OBDD-\MLAEX{}.
      \item $(\obdd', \obdd(e), k)$ is a yes-instance of \OBDD-\MGAEX{}.
      \item $(\obdd', 1-\obdd(e), k)$ is a yes-instance of \OBDD-\MGCEX{}.
  \end{enumerate}
  It is obvious that $(2)$ and $(3)$ are equivalent.

  Towards showing that $(1) \implies (2)$, let $A \subseteq F$ be a
  local abductive explanation for $e$ w.r.t. $\obdd$ of size at most $k$, i.e.,
  every example $e'$ that agrees with $e$ on the features in $A$ is
  classified as $\obdd(e)$ by $\obdd$. Let $A'$ be a superset of $A$ that
  contains exactly $k$ features and let $\tau : A' \rightarrow
  \{0,1\}$ be the assignment defined by setting $\tau(f)=e(f)$ for
  every $f \in A'$. We claim that $\tau$ is a global abductive
  explanation for $\obdd(e)$ w.r.t. $\obdd'$ of size at most $k$,
  i.e., every example $e'$ that agrees with $\tau$ is classified as
  $\obdd(e)$ by $\obdd'$. But this clearly holds because both $\obdd$
  and $\obdd_e^k$ classify $e$ as $\obdd(e)$.
  
  Towards showing that $(2) \implies (1)$, let $\tau : A \rightarrow
  \{0,1\}$ for some $A \subseteq F$ with $|A|\leq k$ be a global
  abductive explanation for $\obdd(e)$ w.r.t. $\obdd'$, i.e., every
  example $e'$ that agrees with $\tau$ is classified as $\obdd(e)$ by
  $\obdd'$ and therefore also by $\obdd$ and $\obdd_e^k$. But then,
  $\tau$ has to agree with $e$ on at least $k$ features, because
  otherwise there would be an example
  that agrees with $\tau$ but is not classified as $\obdd(e)$ by
  $\obdd_e^k$.  But then, $A$ is a local abductive explanation for $e$
  w.r.t. $\obdd$.
\end{proof}\fi

\ifshort \begin{theorem}[$\star$]\fi\iflong\begin{theorem}\fi\label{th:OBDDEO-paraNP-hws}
  Let $\PP \in \{\LAEX, \GAEX, \GCEX\}$. 
  \OBDDEO-\SPP{} is \coNPh{} even if $\widthelem{} + \sizeelem{}$ is constant;
  \OBDDEO-\SMLCEX{} is \NPh{}  even if $\widthelem{} + \sizeelem{}$ is constant;
  \OBDDEO-\MPP{} is \coNPh{} even if $\widthelem{} + \sizeelem{} +\xpsize{}$ is constant;
  \OBDDEO-\MLCEX{}$(\xpsize{})$ is \Wh{1} even if $\widthelem{} + \sizeelem{}$ is constant;
\end{theorem}
\iflong\begin{proof}
    Let $F$ be a set of features and let $<_F$ be an arbitrary
    ordering of the features in $F$.

    We will use \Cref{lem:MME-kHOM-W1} to show the theorem by
    giving \OBDDOcs{<_F} $M^0$, $M^1_{f}$ and $M^2_{\{f_1,f_2\}}$ satisfying the
    conditions of \Cref{lem:MME-kHOM-W1} as follows.
    We let $M^0$ be the trivial \OBDDOc{<_F} that classifies every
    example negatively.
    We let $M^1_{f}$ be the \OBDDOc{<_F} $(D^{1}_f, \fBDD^1_f)$, such that
    $V(D^{1}_f) = \{g, t_0, t_1\}$, $ \fBDD^{1}_f(g)=f$, $g$ is
    the source vertex, and $t_0$ and $t_1$ is the $0$-neighbor and
    $1$-neighbor of $g$, respectively.
    We let $M^2_{f_1,f_2}$ be the \OBDDOc{<_F} $(D^2_{f_1,f_2}
    , \fBDD^2_{f_1,f_2})$ such that
    $V(D^2_{f_1,f_2}) = \{g^{f_1}, g^{f_2}, t_0, t_1\}$,
    $\fBDD^2_{f_1,f_2}(g^{f_1})=f_1$ and
    $\fBDD^2_{f_1,f_2}(g^{f_2})=f_2$.
    Here, we assume that w.l.o.g. $f_1 <_F f_2$. We let $g^{f_1}$ be
    the source vertex. We let $t_1$ and $g^{f_2}$ be the
    $0$-neighbor and $1$-neighbor of $g^{f_1}$, respectively.
    Moreover, we let $t_1$ and $t_0$ be the $0$-neighbor and
    $1$-neighbor of $g^{f_2}$, respectively.

    Note that $M^0$, $M^1_{f}$ and $M^2_{\{f_1,f_2\}}$ have \sizeelem at most
    $4$ and \widthelem{} equal $2$.
    All statements in the theorem now follow from 
    \Cref{lem:MME-kHOM-W1,lem:kHOM-reduction,lem:HOM-reduction}.
\end{proof}\fi

\begin{theorem}\label{th:OBDDEO-W1}
  \sloppypar
  Let $\PP \in \{\LAEX, \GAEX, \GCEX\}$. 
  \OBDDEO-\SPP$(\enssize)$ is \coWh{1};
  \OBDDEO-\SMLCEX$(\enssize)$ is \Wh{1};
  \OBDDEO-\MPP$(\enssize)$ is \coWh{1} even if \xpsize is  constant;
  \OBDDEO-\MLCEX$( \enssize{}  + \xpsize)$ is \Wh{1}.
\end{theorem}
\iflong\begin{proof}
  All statements in the theorem follow from 
  \Cref{lem:rf-hard-example,lem:DT-to-OBDD,lem:HOM-reduction,lem:kHOM-reduction}.
\end{proof}\fi

\section{Conclusion}

We have developed an in-depth exploration of the parameterized complexity of explanation problems in various machine learning (ML) models, focusing on models with transparent internal mechanisms. By analyzing different models and their ensembles, we have provided a comprehensive overview of the complexity of finding explanations in these systems. These insights are crucial for understanding the inherent complexity of different ML models and their implications for explainability.

Among our findings, some results stand out as particularly
unexpected. For instance, while \RF{} and \OBDDEO{}s are seemingly
different model types, our results show that they behave similarly
w.r.t.\ tractability for explanation problems. On the other hand, it
seems surprising that many of the tractability results that hold for
\DT{}s and \OBDD{}s do not carry over to seemingly simpler models such
as \DS{}s and \DL{}s. For instance, while all variants of \LCEX{} are
polynomial-time for \DT{}s and \OBDD{}s, this is not the case for
\DS{}s or \DL{}s. Nevertheless, we obtain interesting FPT-algorithms
for \DL{}-\MLCEX{} (\Cref{th:DS-MLCEX-FPT-const-ens}).  \OBDDE{}
stands out as the hardest model for computing explanations by far,
which holds even
for models with only two ensemble elements.
From a complexity point of view, \DT-\MGAEX{} provides the rare
scenario where a problem is known as \Wh{1} but not confirmed
to be $\NP$\hy hard (\Cref{th:DT-MGAEX-W1}).

Looking ahead, there are several promising directions for future
research. First, we aim to extend our complexity classification to
Sequential Decision Diagrams \cite{Darwiche11} or even FBDDs, which offer a more
succinct representation than OBDDs \cite{Bova16}. This extension could
provide further insights into the complexity of explanations in more
compact ML models. Secondly, we propose to explore other problem
variations, such as counting different types of explanations or
finding explanations that meet specific constraints beyond just the
minimum ones \cite{BarceloM0S20}.
Lastly, the concept of weighted ensembles presents an intriguing avenue for research. While the hardness results we established likely still apply, the tractability in the context of weighted ensembles needs to be clarified and warrants further investigation. It would be  interesting to see how our results hold up when considering polynomial-sized weights.

\iflong
In summary, our work marks a significant stride in the theoretical
understanding of explainability in AI. This research responds to the
practical and regulatory demand for transparent, interpretable, and
trustworthy AI systems by offering a detailed complexity analysis
across various ML models. As the field of XAI evolves, our study lays
a foundational groundwork for future research and the development of
more efficient explanation methods in AI.
\fi

\section*{Acknowledgements}
Stefan Szeider acknowledges support by the Austrian Science Fund (FWF)
within the projects 10.55776/P36688, 10.55776/P36420, and
10.55776/COE12.
Sebastian Ordyniak was supported by the Engineering
and Physical Sciences Research Council (EPSRC) (Project EP/V00252X/1).


\begin{thebibliography}{28}
\providecommand{\natexlab}[1]{#1}
\providecommand{\url}[1]{\texttt{#1}}
\expandafter\ifx\csname urlstyle\endcsname\relax
  \providecommand{\doi}[1]{doi: #1}\else
  \providecommand{\doi}{doi: \begingroup \urlstyle{rm}\Url}\fi

\bibitem[Barcel{\'{o}} et~al.(2020)Barcel{\'{o}}, Monet, P{\'{e}}rez, and
  Subercaseaux]{BarceloM0S20}
Pablo Barcel{\'{o}}, Mika{\"{e}}l Monet, Jorge P{\'{e}}rez, and Bernardo
  Subercaseaux.
\newblock Model interpretability through the lens of computational complexity.
\newblock \emph{Proc. {NeurIPS} 2020}, 33:\penalty0 15487--15498, 2020.

\bibitem[Bergougnoux et~al.(2023)Bergougnoux, Dreier, and
  Jaffke]{BergougnouxDJ23}
Benjamin Bergougnoux, Jan Dreier, and Lars Jaffke.
\newblock A logic-based algorithmic meta-theorem for mim-width.
\newblock \emph{Proc. {SODA} 2023}, pages 3282--3304, 2023.

\bibitem[Bova(2016)]{Bova16}
Simone Bova.
\newblock {SDD}s are exponentially more succinct than {OBDD}s.
\newblock \emph{Proc. {AAAI} 2016}, pages 929--935, 2016.
\newblock \doi{10.1609/AAAI.V30I1.10107}.

\bibitem[Carvalho et~al.(2019)Carvalho, Pereira, and Cardoso]{Carvalho-etal-19}
Diogo~V. Carvalho, Eduardo~M. Pereira, and Jaime~S. Cardoso.
\newblock Machine learning interpretability: A survey on methods and metrics.
\newblock \emph{Electronics}, 8\penalty0 (8), 2019.
\newblock \doi{10.3390/electronics8080832}.

\bibitem[Chan and Darwiche(2003)]{ChanDarwiche03}
Hei Chan and Adnan Darwiche.
\newblock Reasoning about bayesian network classifiers.
\newblock \emph{Proc. {UAI} 2003}, pages 107--115, 2003.

\bibitem[Commission(2019)]{EU19}
European Commission.
\newblock \emph{Ethics guidelines for trustworthy AI}.
\newblock Publications Office, European Commission and Directorate-General for
  Communications Networks, Content and Technology, 2019.
\newblock \doi{doi/10.2759/346720}.

\bibitem[Commission(2020)]{EU20}
European Commission.
\newblock \emph{White Paper on {Artificial Intelligence}: a {European} approach
  to excellence and trust}.
\newblock Publications Office, European Commission and Directorate-General for
  Communications Networks, Content and Technology, 2020.

\bibitem[Corneil and Rotics(2001)]{DBLP:conf/wg/CorneilR01}
Derek~G. Corneil and Udi Rotics.
\newblock On the relationship between clique-width and treewidth.
\newblock In Andreas Brandst{\"{a}}dt and Van~Bang Le, editors,
  \emph{Graph-Theoretic Concepts in Computer Science, 27th International
  Workshop, {WG} 2001, Boltenhagen, Germany, June 14-16, 2001, Proceedings},
  volume 2204 of \emph{Lecture Notes in Computer Science}, pages 78--90.
  Springer, 2001.

\bibitem[Darwiche(2011)]{Darwiche11}
Adnan Darwiche.
\newblock {SDD:} {A} new canonical representation of propositional knowledge
  bases.
\newblock \emph{Proc. {IJCAI} 2011}, pages 819--826, 2011.
\newblock \doi{10.5591/978-1-57735-516-8/IJCAI11-143}.

\bibitem[Darwiche and Ji(2022)]{DarwicheJi22}
Adnan Darwiche and Chunxi Ji.
\newblock On the computation of necessary and sufficient explanations.
\newblock \emph{Proc. {AAAI} 2022}, pages 5582--5591, 2022.
\newblock \doi{10.1609/AAAI.V36I5.20498}.

\bibitem[Diestel(2000)]{Diestel00}
Reinhard Diestel.
\newblock \emph{Graph Theory}, volume 173 of \emph{Graduate Texts in
  Mathematics}.
\newblock Springer Verlag, New York, 2nd edition, 2000.

\bibitem[Downey and Fellows(2013)]{DowneyFellows13}
Rodney~G. Downey and Michael~R. Fellows.
\newblock \emph{Fundamentals of parameterized complexity}.
\newblock Texts in Computer Science. Springer Verlag, 2013.

\bibitem[Guidotti et~al.(2019)Guidotti, Monreale, Ruggieri, Turini, Giannotti,
  and Pedreschi]{Guidotti-etal-2018}
Riccardo Guidotti, Anna Monreale, Salvatore Ruggieri, Franco Turini, Fosca
  Giannotti, and Dino Pedreschi.
\newblock A survey of methods for explaining black box models.
\newblock \emph{{ACM} Computing Surveys}, 51\penalty0 (5):\penalty0
  93:1--93:42, 2019.
\newblock \doi{10.1145/3236009}.

\bibitem[Holzinger et~al.(2020)Holzinger, Saranti, Molnar, Biecek, and
  Samek]{HolzingerSMBS20}
Andreas Holzinger, Anna Saranti, Christoph Molnar, Przemyslaw Biecek, and
  Wojciech Samek.
\newblock Explainable {AI} methods - {A} brief overview.
\newblock \emph{Proc. {xxAI@ICML} 2020}, 13200:\penalty0 13--38, 2020.
\newblock \doi{10.1007/978-3-031-04083-2\_2}.

\bibitem[Ignatiev et~al.(2019)Ignatiev, Narodytska, and
  Marques{-}Silva]{IgnatievNM19}
Alexey Ignatiev, Nina Narodytska, and Jo{\~{a}}o Marques{-}Silva.
\newblock Abduction-based explanations for machine learning models.
\newblock \emph{Proc. {AAAI} 2019}, pages 1511--1519, 2019.
\newblock \doi{10.1609/AAAI.V33I01.33011511}.

\bibitem[Ignatiev et~al.(2020)Ignatiev, Narodytska, NicholasAsher, and
  Marques{-}Silva]{IgnatievNA20}
Alexey Ignatiev, Nina Narodytska, NicholasAsher, and Jo{\~{a}}o
  Marques{-}Silva.
\newblock From contrastive to abductive explanations and back again.
\newblock \emph{Proc. {AIxIA} 2020}, 12414:\penalty0 335--355, 2020.
\newblock \doi{10.1007/978-3-030-77091-4\_21}.

\bibitem[Jha and Suciu(2012)]{DBLP:conf/icdt/JhaS12}
Abhay~Kumar Jha and Dan Suciu.
\newblock On the tractability of query compilation and bounded treewidth.
\newblock \emph{Proc. {ICDT} 2012}, pages 249--261, 2012.

\bibitem[Lipton(2018)]{Lipton18}
Zachary~C. Lipton.
\newblock The mythos of model interpretability.
\newblock \emph{Communications of the {ACM}}, 61\penalty0 (10):\penalty0
  36--43, 2018.
\newblock \doi{10.1145/3233231}.

\bibitem[Lisboa et~al.(2023)Lisboa, Saralajew, Vellido,
  Fern{\'{a}}ndez{-}Domenech, and Villmann]{LisboaSVFV23}
Paulo J.~G. Lisboa, Sascha Saralajew, Alfredo Vellido, Ricardo
  Fern{\'{a}}ndez{-}Domenech, and Thomas Villmann.
\newblock The coming of age of interpretable and explainable machine learning
  models.
\newblock \emph{Neurocomputing}, 535:\penalty0 25--39, 2023.
\newblock \doi{10.1016/J.NEUCOM.2023.02.040}.

\bibitem[Marques-Silva(2023)]{Silva22}
Joao Marques-Silva.
\newblock Logic-based explainability in machine learning.
\newblock \emph{Reasoning Web. Causality, Explanations and Declarative
  Knowledge}, pages 24--104, 2023.
\newblock \doi{10.1007/978-3-031-31414-8_2}.

\bibitem[Mengel and Slivovsky(2021)]{DBLP:journals/corr/abs-2104-02563}
Stefan Mengel and Friedrich Slivovsky.
\newblock Proof complexity of symbolic {QBF} reasoning.
\newblock \emph{CoRR}, abs/2104.02563, 2021.

\bibitem[Miller(2019)]{Miller19}
Tim Miller.
\newblock Explanation in artificial intelligence: Insights from the social
  sciences.
\newblock \emph{Artificial Intelligence}, 267:\penalty0 1--38, 2019.
\newblock \doi{10.1016/J.ARTINT.2018.07.007}.

\bibitem[Molnar(2023)]{Molnar23}
Christoph Molnar.
\newblock \emph{Interpretable Machine Learning}.
\newblock Lulu.com, 2023.

\bibitem[OECD(2023)]{OECD23}
OECD.
\newblock The state of implementation of the {OECD} {AI} principles four years
  on.
\newblock Technical Report~3, The Organisation for Economic Co-operation and
  Development, 2023.

\bibitem[Ordyniak et~al.(2023)Ordyniak, Paesani, and Szeider]{OrdyniakPS23}
Sebastian Ordyniak, Giacomo Paesani, and Stefan Szeider.
\newblock The parameterized complexity of finding concise local explanations.
\newblock \emph{Proc. {IJCAI} 2023}, pages 3312--3320, 2023.
\newblock \doi{10.24963/IJCAI.2023/369}.

\bibitem[Oum and Seymour(2006)]{DBLP:journals/jct/OumS06}
Sang{-}il Oum and Paul~D. Seymour.
\newblock Approximating clique-width and branch-width.
\newblock \emph{J. Comb. Theory, Ser. {B}}, 96\penalty0 (4):\penalty0 514--528,
  2006.

\bibitem[Ribeiro et~al.(2016)Ribeiro, Singh, and Guestrin]{Ribeiro0G16}
Marco~T{\'{u}}lio Ribeiro, Sameer Singh, and Carlos Guestrin.
\newblock ``why should {I} trust you?'': Explaining the predictions of any
  classifier.
\newblock \emph{Proc. {KDD} 2016}, pages 1135--1144, 2016.
\newblock \doi{10.1145/2939672.2939778}.

\bibitem[Shih et~al.(2018)Shih, Choi, and Darwiche]{ShihCD18}
Andy Shih, Arthur Choi, and Adnan Darwiche.
\newblock A symbolic approach to explaining {Bayesian} network classifiers.
\newblock \emph{Proc. {IJCAI} 2018}, pages 5103--5111, 2018.
\newblock \doi{10.24963/IJCAI.2018/708}.

\end{thebibliography}
\end{document}